\documentclass[10pt]{amsart}
\usepackage[utf8x]{inputenc}
\usepackage{kantlipsum}
\usepackage[pagebackref=true, colorlinks=true,linkcolor=blue, citecolor={red}]{hyperref}
\usepackage{lmodern}\usepackage[T1]{fontenc}
\usepackage{amsmath,amssymb,amsfonts,euscript,graphicx,hyperref,indentfirst}
\usepackage{enumerate,linegoal}
\usepackage[all]{xy}
\usepackage{tikz-cd}
\usepackage{stmaryrd}
\usepackage{relsize,exscale}
\usepackage{bigints}
\usepackage{cancel}
\usepackage{accents}

\calclayout

\newcommand{\bigzero}{\mbox{\normalfont\Large\bfseries 0}}
\newcommand{\dbtilde}[1]{\accentset{\approx}{#1}}

\newtheorem*{corollary*}{Corollary}

\newtheorem{theorem}{Theorem}[section]
\newtheorem{lemma}[theorem]{Lemma}
\newtheorem{proposition}[theorem]{Proposition}
\newtheorem{corollary}[theorem]{Corollary}
\newtheorem{conjecture}[theorem]{Conjecture}
\newtheorem{definition-theorem}[theorem]{Definition-Theorem}

\theoremstyle{definition}
\newtheorem{definition}[theorem]{Definition}
\newtheorem{example}[theorem]{Example}
\newtheorem{terminology}[theorem]{Terminology}

\theoremstyle{remark}
\newtheorem{remark}[theorem]{Remark}
\makeatletter
\def\l@subsection{\@tocline{2}{0pt}{2.5pc}{5pc}{}} 
\makeatother

\numberwithin{equation}{section}

\title{PDE constraints on smooth hierarchical functions computed by neural networks}
\author{Khashayar Filom, Konrad Paul Kording, Roozbeh Farhoodi}

\renewcommand{\shorttitle}{PDE constraints on functions computed by neural networks}

\address{Khashayar Filom, Department of Mathematics, University of Michigan}
\email{filom@umich.edu}

\address{Konrad Paul Kording, Department of Bioengineering, University of Pennsylvania}
\email{kording@upenn.edu}

\address{Roozbeh Farhoodi, Department of Bioengineering, University of Pennsylvania}
\email{roozbeh@seas.upenn.edu}

\date{}
\begin{document}

\vspace*{-2cm}
\maketitle

\vspace{-1cm}
\begin{abstract}
Neural networks are versatile tools for computation, having the ability to approximate a broad range of functions. An important problem in the theory of deep neural networks is expressivity; that is, we want to understand the functions that are computable by a given network.  We study real infinitely differentiable (smooth) hierarchical functions implemented by feedforward neural networks via composing simpler functions in two cases:
\begin{enumerate}
\item each constituent function of the composition has fewer inputs than the resulting function;
\item constituent functions are in the more specific yet prevalent form of a non-linear univariate function (e.g. tanh) applied to a linear multivariate function. 
\end{enumerate}
We establish that in each of these regimes there exist non-trivial algebraic partial differential equations (PDEs), which are satisfied by the computed functions. These PDEs are purely in terms of the partial derivatives and are dependent only on the topology of the network.  For compositions of polynomial functions, the algebraic PDEs yield non-trivial equations (of degrees dependent only on the architecture) in the ambient polynomial space  that are satisfied on the associated functional varieties. Conversely, we conjecture that such PDE constraints, once accompanied by appropriate non-singularity conditions and perhaps certain inequalities involving partial derivatives, guarantee that the smooth function under consideration can be represented by the network. The conjecture is verified in numerous examples including the case of tree architectures which are of neuroscientific interest.  Our approach is a step toward formulating an algebraic description of functional spaces associated with specific neural networks, and may provide new, useful tools for constructing neural networks. 
\end{abstract}

\renewcommand{\baselinestretch}{0.9}\normalsize
\footnotesize
\tableofcontents
\renewcommand{\baselinestretch}{1.0}\normalsize
\normalsize

\section{Introduction}\label{Introduction}

\subsection{Motivation}
A central problem in the theory of deep neural networks is to understand the functions that can be computed by a particular architecture \cite{raghu2017expressive,poggio2019theoretical}. Such functions are typically \textit{superpositions} of simpler functions; that is, compositions of functions of fewer variables. This article aims to study  superpositions of real \textit{smooth} (i.e. infinitely differentiable or $C^\infty$) functions which are constructed hierarchically; see Figure \ref{fig:superposition}. Our core thesis is that such functions (also referred to as \textit{hierarchical} or \textit{compositional} interchangeably) are constrained in the sense that they satisfy 
certain \textit{\textbf{p}artial \textbf{d}ifferential \textbf{e}quations} (\textbf{PDE}s). These PDEs are dependent only on the topology of the network, and could furthermore be employed to characterize smooth functions computable by a given network.

\begin{example}\label{basic-1}
One of the simplest examples of a superposition is when a trivariate function is obtained from composing two bivariate functions; for instance, let us consider the composition 
\begin{equation}\label{3var_form}
F(x,y,z) = g\left(f(x,y),z\right)
\end{equation}
of functions $f=f(x,y)$ and $g=g(u,z)$ that can be computed by the network in Figure \ref{fig:basic-1}.
Assuming that all functions appearing here are twice continuously differentiable (or $C^2$), the chain rule yields
$$
F_x=g_uf_x,\quad F_y=g_uf_y.
$$
If either $F_x$ or $F_y$ -- say the former -- is non-zero, the equations above imply that the ratio between $F_x$ and $F_y$ is independent of $z$:
\small
\begin{equation}\label{auxiliary-basic}
\frac{F_y}{F_x}=\frac{f_y}{f_x}. 
\end{equation}
\normalsize
Therefore, its derivative with respect to $z$ must be identically zero:
\small
\begin{equation}\label{auxiliary-basic-constraint}
\left(\frac{F_y}{F_x}\right)_z = \frac{F_{yz}F_x-F_{xz}F_y}{(F_x)^2}=0.
\end{equation}
\normalsize
This amounts to 
\begin{equation}\label{auxiliary-basic-constraint'}
F_{yz}F_x=F_{xz}F_y,
\end{equation}
\normalsize
an equation that always holds for functions of the form \eqref{3var_form}. Notice that one may readily exhibit functions that do not satisfy the necessary PDE constraint $F_{xz}F_y=F_{yz}F_x$ and so cannot be brought into form \eqref{3var_form}, e.g.  
\begin{equation}\label{non-example}
xyz+x+y+z.    
\end{equation}
\normalsize
Conversely, if the constraint $F_{yz}F_x=F_{xz}F_y$ is satisfied and $F_x$ (or $F_y$) is non-zero, we can reverse this processes to obtain a local expression of the form \eqref{3var_form} for $F(x,y,z)$: By interpreting the constraint as the independence of $\frac{F_x}{F_y}$ of $z$, one can devise a function $f=f(x,y)$ whose ratio of partial derivatives coincides with $\frac{F_x}{F_y}$ (this is a calculus fact; see Theorem \ref{integrability}). Now that \eqref{auxiliary-basic}  is satisfied, the gradient of $F$ may be written as
\small
$$
\nabla F
=\begin{bmatrix}
F_x\\
F_y\\
F_z
\end{bmatrix}
=\frac{F_x}{f_x}
\begin{bmatrix}
f_x\\
f_y\\
0
\end{bmatrix}
+F_z
\begin{bmatrix}
0\\
0\\
1
\end{bmatrix};
$$
\normalsize
i.e. as a linear combination of gradients of $f(x,y)$ and $z$. The Implicit Function Theorem then guarantees that $F(x,y,z)$ is (at least locally) a function of the latter two: There exists a bivariate function $g$ defined on a suitable domain with $F(x,y,z)=g(f(x,y),z)$. Later in the paper, we shall generalize this toy example to a characterization of superpositions computed by \textit{tree architectures}; cf. Theorem \ref{main-tree} \\
\indent
Functions appearing in the context of neural networks are more specific than a general superposition such as \eqref{3var_form}; they are predominantly constructed by composing univariate non-linear \textit{activation functions} and multivariate linear functions defined by \textit{weights} and \textit{biases}. In the case of a trivariate functions $F(x,y,z)$, we should replace the representation $g(f(x,y),z)$ studied so far with 
\begin{equation}\label{3var_form-activation}
F(x,y,z)=g(w_3f(w_1x+w_2y+b_1)+w_4z+b_2).
\end{equation}
Assuming that activation functions $f$ and $g$ are differentiable, now new constraints of the form \eqref{auxiliary-basic-constraint} are imposed: The ratio $\frac{F_y}{F_x}$ is equal to $\frac{w_2}{w_1}$, hence it is not only independent of $z$ as \eqref{auxiliary-basic-constraint} suggests, but indeed a constant function. So we arrive at
\small
$$
\left(\frac{F_y}{F_x}\right)_x = \left(\frac{F_y}{F_x}\right)_y = \left(\frac{F_y}{F_x}\right)_z = 0;
$$
\normalsize
or equivalently
$$
F_{xy}F_x=F_{xx}F_y,\quad F_{yy}F_x=F_{xy}F_y, \quad F_{yz}F_x=F_{xz}F_y.
$$
\normalsize
Again, these equations characterize differentiable functions of the form \eqref{3var_form-activation}; this is a special case of Theorem \ref{main-tree-activation} below.

\begin{figure}
    \centering
    \includegraphics[height=3cm]{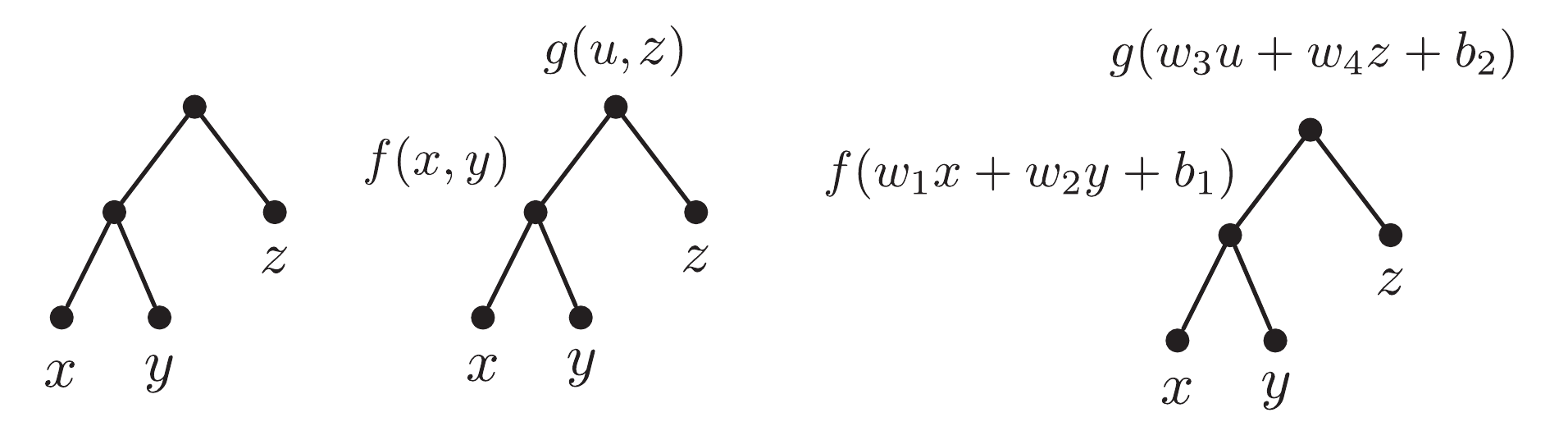}
    \caption{The architecture on the left (studied in Example \ref{basic-1}) can compute functions of the form $g(f(x,y),z)$ as in the middle. They involve the smaller class of functions of the form $g(w_3f(w_1x+w_2y+b_1)+w_4z+b_2)$ on the right. }
    \label{fig:basic-1}
\end{figure}

\end{example}

\begin{example}\label{basic-1'}
The preceding example dealt with compositions of functions with disjoint sets of variables and this facilitated our calculations. But this is not the case for compositions constructed by most neural networks, e.g. networks may be fully connected or may have repeated inputs. For instance, let us consider a superposition of the form 
\begin{equation}\label{toy1}
F(x,y,z)=g(f(x,y),h(x,z))    
\end{equation}
of functions $f(x,y)$, $h(x,z)$ and $g(u,v)$ as implemented in Figure \ref{fig:basic-2}. Applying the chain rule tends to be more complicated than the case of \eqref{3var_form} and results in identities
\begin{equation}\label{auxiliary-basic-combination}
F_x=g_uf_x+g_vh_x,\quad F_y=g_uf_y,\quad F_z=g_vh_z.
\end{equation}
Nevertheless, it is not hard to see that there are again (perhaps cumbersome) non-trivial PDE constraints imposed on the hierarchical function $F$ -- this fact will be established generally in Theorem \ref{main} below. To elaborate, notice that identities in \eqref{auxiliary-basic-combination} together imply 
\begin{equation}\label{auxiliary7}
F_x=A(x,y)F_y+B(x,z)F_z
\end{equation}
where $A:=\frac{f_x}{f_y}$ and $B:=\frac{h_x}{h_z}$ are independent of $z$ and $y$ respectively. Repeatedly differentiating this identity (if possible) with respect to $y,z$ results in linear dependence relations between partial derivatives of $F$ (and hence PDEs) since the number of partial derivatives of $F_x$ of order at most $n$ with respect to $y,z$  grows quadratically with $n$ while on the right-hand side the number of possibilities for coefficients (partial derivatives of $A$ and $B$ with respect to $y$ and $z$ respectively) grows only linearly. Such dependencies could be encoded by the vanishing of determinants of suitable matrices formed by partial derivatives of $F$.
In Example \ref{toy example 1}, by pursuing the strategy just mentioned, we shall complete this treatment of superpositions \eqref{toy1} by deriving the corresponding \textit{characteristic PDEs} that are necessary and (in a sense) sufficient conditions on $F$ that it be in the form of \eqref{toy1}. Moreover, in order to be able to differentiate several times, we shall assume that all functions are smooth (or $C^\infty$) hereafter.   
\end{example}

\subsection{Statements of main results}
Fixing a neural network hierarchy for composing functions, we shall prove that once the constituent functions of corresponding superpositions have fewer inputs (lower arity), there exist universal 
\textit{\textbf{algebraic} \textbf{p}artial \textbf{d}ifferential \textbf{e}quations} (\textbf{algebraic PDE}s) that have these superpositions as their solutions. A conjecture -- which shall be verified in several cases -- states that such PDE constraints characterize a \textit{generic} smooth superposition computable by the network. Here, genericity  means a \textit{non-vanishing condition} imposed on an algebraic expression of partial derivatives. 
Such a condition has already occurred in Example \ref{basic-1} where in the proof of the sufficiency of \eqref{auxiliary-basic-constraint'} for the existence of a representation of the form \eqref{3var_form} for a  function $F(x,y,z)$ we assumed either $F_x$ or $F_y$ is non-zero. 
Before proceeding with the statements of main results, we formally define some of the terms that have appeared so far.          
\begin{terminology}
\begin{minipage}[t]{\linegoal}
\begin{itemize}
\item We take all neural networks to be feedforward. A \textbf{feedforward neural network} is an acyclic hierarchical layer to layer scheme of computation. We also include \textit{\textbf{Res}idual \textbf{Net}works} 
(\textbf{ResNet}s) into this category: an identity function in a layer could be interpreted as a jump in layers. \textit{Tree architectures} are recurring examples of this kind.  We shall always assume that in the first layer the inputs are labeled by (not necessarily distinct) labels chosen from  coordinate functions $x_1,\dots,x_n$; and there is only one node in the output layer. Assigning functions to nodes in layers above the input layer implements a real scalar-valued function $F=F(x_1,\dots,x_n)$ as the superposition of functions appearing at nodes; see Figure \ref{fig:superposition}.
\item In our setting, an \textbf{algebraic PDE} is a non-trivial polynomial relation such as    
\begin{equation}\label{alg PDE}
\Phi\left(F_{x_1},\dots,F_{x_n},F_{x_1^2}, F_{x_1x_2},\dots,F_{\mathbf{x}^{\mathbf{\alpha}}},\dots\right)=0
\end{equation}
among the partial derivatives (up to a certain order) of a smooth function $F=F(x_1,\dots,x_n)$. Here, for a tuple $\mathbf{\alpha}:=(\alpha_1,\dots,\alpha_n)$ of non-negative integers,  the partial derivative
$\frac{\partial^{\alpha_1+\dots+\alpha_n}F}{\partial x_1^{\alpha_1}\dots\partial x_n^{\alpha_n}}$ (which is of order $|\alpha|:=\alpha_1+\dots+\alpha_n$) is denoted by $F_{\mathbf{x}^{\mathbf{\alpha}}}$. For instance,  asking for a polynomial expression of partial derivatives of $F$ to be constant amounts to $n$ algebraic PDEs given by setting the first order partial derivatives of that expression with respect to $x_1,\dots,x_n$ to be zero. 
\item A \textbf{non-vanishing condition} imposed on smooth functions $F=F(x_1,\dots,x_n)$ is asking for these functions not to satisfy a particular algebraic PDE; namely,
\begin{equation}\label{non-vanishing}
\Psi\left(F_{x_1},\dots,F_{x_n},F_{x_1^2}, F_{x_1x_2},\dots,F_{\mathbf{x}^{\mathbf{\alpha}}},\dots\right)\neq 0
\end{equation}
for a non-constant polynomial $\Psi$. Such a condition could be deemed pointwise since if it holds at a point $\mathbf{p}\in\Bbb{R}^n$, it persists throughout a small enough neighborhood. Moreover, \eqref{non-vanishing} determines an open dense subset of the functional space; so, it is satisfied generically.    
\end{itemize}
\end{minipage}
\end{terminology}

\begin{figure}
    \centering
    \includegraphics[height=2.5cm]{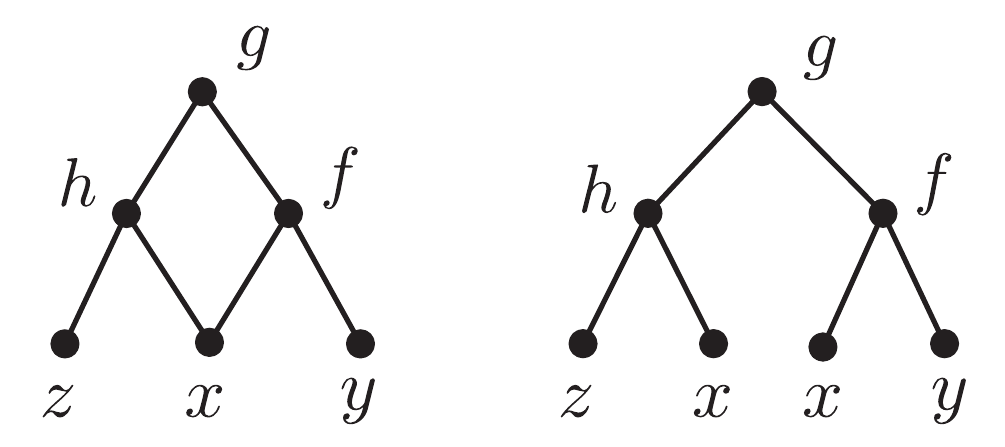}
    \caption{Implementations of superpositions of the form $F(x,y,z)=g(f(x,y),h(x,z))$ (studied in Examples \ref{basic-1'} and \ref{toy example 1}) by three-layer neural networks.}
    \label{fig:basic-2}
\end{figure}

\begin{figure}
    \centering
    \includegraphics[height=4cm]{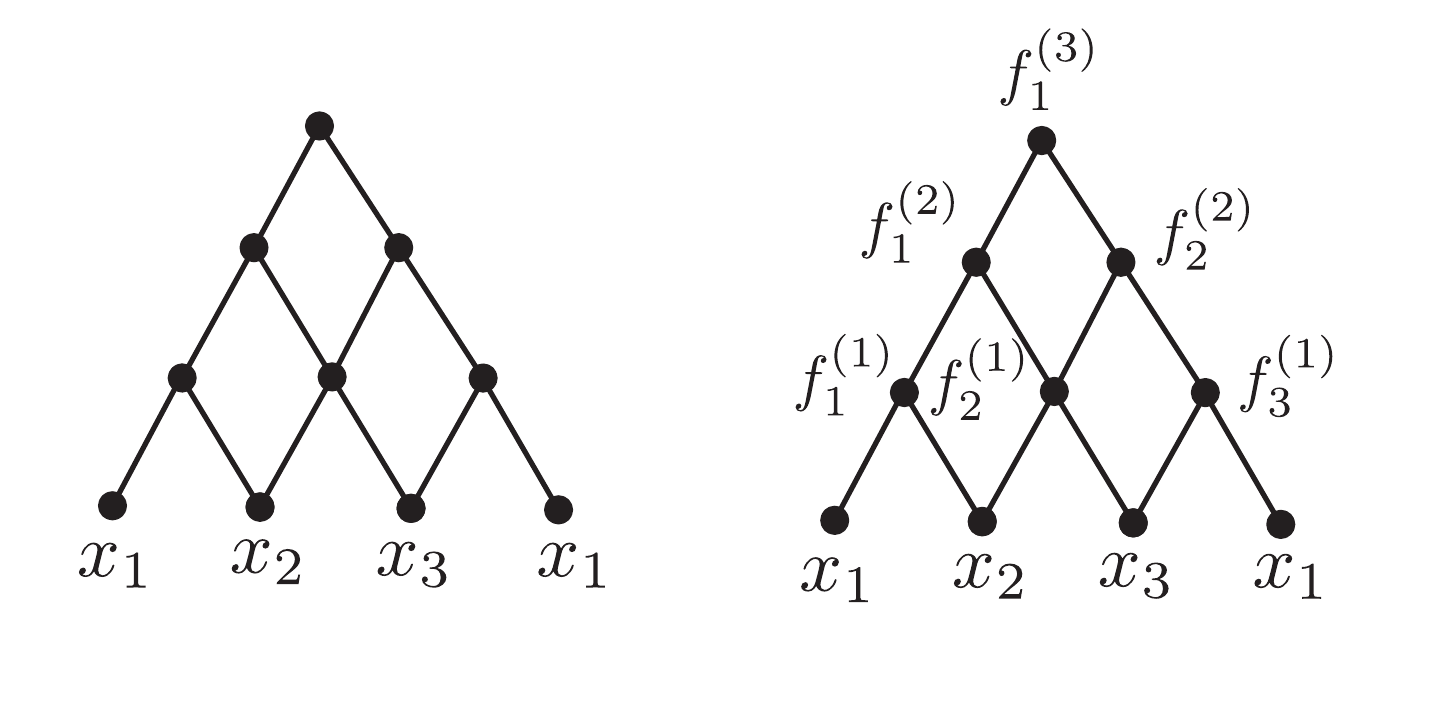}
    \caption{The neural network on the left can compute the hierarchical function 
    $F(x_1,x_2,x_3)=f^{(3)}_1\Bigg(f^{(2)}_1\left(f^{(1)}_1(x_1,x_2),f^{(1)}_2(x_2,x_3)\right),f^{(2)}_2\left(f^{(1)}_2(x_2,x_3),f^{(1)}_3(x_3,x_1)\right)\Bigg)$
    once appropriate functions are assigned to its nodes as on the right.}
    \label{fig:superposition}
\end{figure}

\begin{theorem}\label{main}
Let $\mathcal{N}$ be a feedforward neural network in which the number of inputs to each node is less than the total number of distinct inputs to the network. Superpositions of smooth functions computed by this network satisfy non-trivial constraints in the form of certain algebraic PDEs which are dependent only on the topology of $\mathcal{N}$. 
\end{theorem}

In the context of deep learning, the functions applied at each node are in the form of  
\begin{equation}\label{activation}
\mathbf{y}\mapsto\sigma\left(\langle\mathbf{w},\mathbf{y}\rangle\right);
\end{equation}
that is, they are obtained by applying an activation function $\sigma$  to a linear functional $\mathbf{y}\mapsto\langle\mathbf{w},\mathbf{y}\rangle$. Here, as usual, the bias term is absorbed into the weight vector. The bias term could also be excluded via composing $\sigma$ with a translation since 
throughout our discussion the only requirement for a function $\sigma$ to be the activation function  of a node is smoothness; and activation functions are allowed to vary from a node to another.
In our setting, $\sigma$ in \eqref{activation} could be a \textit{polynomial} or a \textit{sigmoidal} function such as \textit{hyperbolic tangent} or \textit{logistic} functions, but not \textit{ReLU} or \textit{maxout} activation functions. 
We shall study functions computable by neural networks as either superpositions of arbitrary smooth functions or as superpositions of functions of the form  \eqref{activation} which is a  more limited regime.  Indeed, the question of how well arbitrary  compositional functions -- which are the subject of Theorem \ref{main} -- may be approximated by a deep network has been studied in the literature  \cite{mhaskar2017and,poggio2017and}.\\
\indent
In order to guarantee the existence of PDE constraints for superpositions, Theorem \ref{main} assumes a condition on the topology of the network. However, the theorem below states that by restricting the functions that can appear in the superposition, one can still obtain PDE constraints even for a fully connected multi-layer perceptron.

\begin{theorem}\label{main'}
Let $\mathcal{N}$ be an arbitrary feedforward neural network with at least two distinct inputs, with smooth functions of the form \eqref{activation} applied at its nodes.  Any function computed by this network satisfies non-trivial constraints in the form of certain algebraic PDEs which are dependent only on the topology of $\mathcal{N}$. 
\end{theorem}

\begin{example}\label{basic-2}
As the simplest example of PDE constraints imposed on compositions of functions of the form \eqref{activation}, recall that d'Alembert's solution to the wave equation 
\begin{equation}\label{wave}
u_{tt}=c^2u_{xx}
\end{equation}
is famously given by superpositions of the form $f(x+ct)+g(x-ct)$. This function can be implemented by a network with two inputs $x,t$ and with one hidden layer in which the activation functions $f,g$ are applied; see Figure \ref{fig:basic-2s}. Since we wish for a PDE that works for this architecture universally, we should get rid of $c$. The PDE \eqref{wave} may be written as $\frac{u_{tt}}{u_{xx}}=c^2$; that is to say, the ratio  $\frac{u_{tt}}{u_{xx}}$ must be constant. Hence, for our purposes the wave equation should be written as $\left(\frac{u_{tt}}{u_{xx}}\right)_x=\left(\frac{u_{tt}}{u_{xx}}\right)_t=0$, or equivalently
$$
u_{xtt}u_{xx}-u_{tt}u_{xxx}=0,\quad u_{ttt}u_{xx}-u_{tt}u_{xxt}=0.
$$
A crucial point to notice is that the aforementioned constant $c^2$ is non-negative, thus an inequality of the form $\frac{u_{xx}}{u_{tt}}\geq 0$ or $u_{xx}u_{tt}\geq 0$ is imposed as well. 
In Example \ref{network1}, we visit upon this network again and we shall study functions of the form
\begin{equation}\label{2var_form}
F(x,t)=\sigma(a'' f(ax+bt)+b'' g(a'x+b't))    
\end{equation}
via a number of equalities and inequalities involving partial derivatives of $F$.
\begin{figure}
    \centering
    \includegraphics[height=3cm]{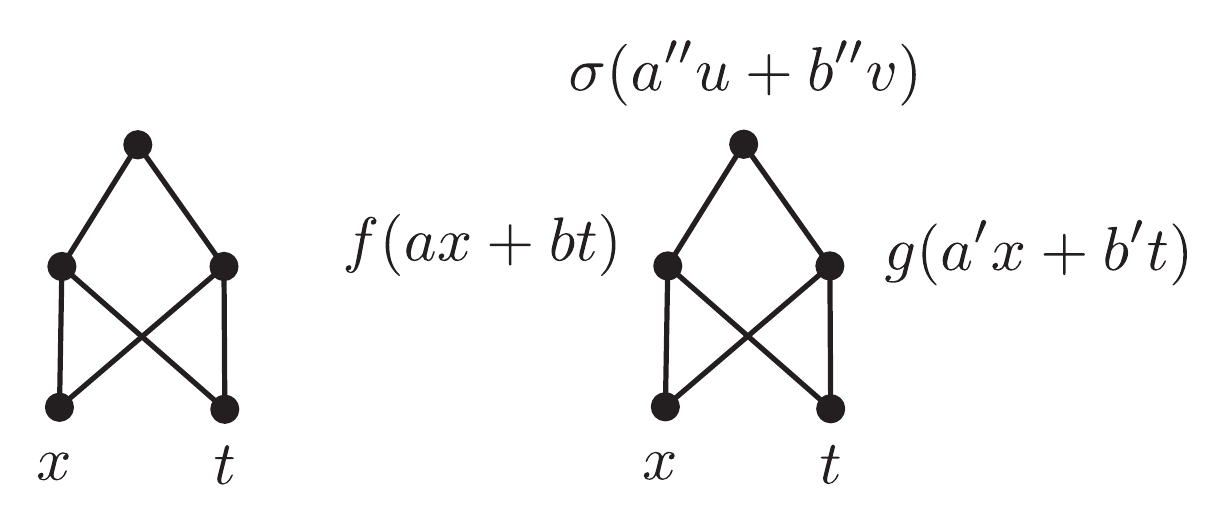}
    \caption{The neural network on the left can compute the function 
    $F(x,t)=\sigma(a'' f(ax+bt)+b'' g(a'x+b't))$ once,  as on the right, the activation functions $\sigma,f,g$ and appropriate weights are assigned to the nodes. Such functions are the subject of Examples \ref{basic-2} and \ref{network1}.}
    \label{fig:basic-2s}
\end{figure}
\end{example}

The preceding example suggests that smooth functions implemented by a neural network may be required to obey a non-trivial 
\textit{\textbf{algebraic} \textbf{p}artial \textbf{d}ifferential \textbf{i}nequality} (\textbf{algebraic PDI}). So it is convenient to have the following set-up of terminology. 
\begin{terminology}
\begin{minipage}[t]{\linegoal}
\begin{itemize}
\item An \textbf{algebraic PDI} is an inequality of the form 
\begin{equation}\label{alg PDI}
\Theta\left(F_{x_1},\dots,F_{x_n},F_{x_1^2}, F_{x_1x_2},\dots,F_{\mathbf{x}^{\mathbf{\alpha}}},\dots\right)>0
\end{equation}
\end{itemize}
involving partial derivatives (up to a certain order) where $\Theta$ is a real polynomial.
\end{minipage}
\end{terminology}

\begin{remark}
Without any loss of generality, we assume that the PDIs are strict since a non-strict one such as  $\Theta\geq 0$ could be written as the union of $\Theta>0$ and the algebraic PDE $\Theta=0$.
\end{remark}

Theorem \ref{main} and Example \ref{basic-1} deal with superpositions of arbitrary smooth functions while Theorem \ref{main'} and Example \ref{basic-2} are concerned with superpositions of a specific class of smooth functions, functions of the form \eqref{activation}. In view of the necessary PDE constraints appeared in both situations, the following question then arises: Are there sufficient conditions in the form of algebraic PDEs and PDIs that guarantee a smooth function can be represented, at least locally, by the neural network in question? 

\begin{conjecture}\label{conjecture}
Let $\mathcal{N}$ be a feedforward neural network whose inputs are labeled by the coordinate functions $x_1,\dots,x_n$.
Suppose we are working in the setting of one of Theorems \ref{main} or \ref{main'}.
Then there exist 
\begin{itemize}
\item finitely many non-vanishing conditions $\left\{\Psi_i\left(\left(F_{\mathbf{x}^{\mathbf{\alpha}}}\right)_{|\alpha|\leq r}\right)\neq 0\right\}_{i}$,
\item finitely many algebraic PDEs $\left\{\Phi_j\left(\left(F_{\mathbf{x}^{\mathbf{\alpha}}}\right)_{|\alpha|\leq r}\right)=0\right\}_{j}$, 
\item finitely many algebraic PDIs $\left\{\Theta_k\left(\left(F_{\mathbf{x}^{\mathbf{\alpha}}}\right)_{|\alpha|\leq r}\right)>0\right\}_{k}$;
\end{itemize}
with the following property: For any arbitrary point $\mathbf{p}\in\Bbb{R}^n$, the space of smooth functions $F=F(x_1,\dots,x_n)$ defined in a vicinity\footnote{To be mathematically precise, the open neighborhood of $\mathbf{p}$ on which $F$ admits a compositional representation in the desired form may be dependent on $F$ and $\mathbf{p}$. So Conjecture \ref{conjecture} is local in nature and must be understood as a statement about \textit{function germs}.} of $\mathbf{p}$
which satisfy $\Psi_i\neq 0$ at $\mathbf{p}$ and are computable by $\mathcal{N}$ (in the sense of the regime under consideration) is non-vacuous and is characterized by PDEs $\Phi_j=0$ and PDIs $\Theta_k>0$.
\end{conjecture}

To motivate the conjecture, notice that it claims the existence of functionals 
$$\left\{F\mapsto\Psi_i\left(\left(F_{\mathbf{x}^{\mathbf{\alpha}}}\right)_{|\alpha|\leq r}\right)\right\}_{i},\,
\left\{F\mapsto\Phi_j\left(\left(F_{\mathbf{x}^{\mathbf{\alpha}}}\right)_{|\alpha|\leq r}\right)\right\}_{j},\,
\left\{F\mapsto\Theta_k\left(\left(F_{\mathbf{x}^{\mathbf{\alpha}}}\right)_{|\alpha|\leq r}\right)\right\}_{k}
$$
which are polynomial expressions of partial derivatives, and hence continuous in the $C^r$-norm\footnote{Convergence in the $C^r$-norm is defined as the uniform convergence of the function and its partial derivatives up to order $r$.}, such that in the space of functions computable by $\mathcal{N}$, the open dense subset given by 
$\left\{\Psi_i\neq 0\right\}_{i}$ (see the remark below) can be described in terms of finitely many equations and inequalities as the locally closed subset 
$
\left\{\Phi_j=0\right\}_{j}\bigcup\left\{\Theta_k>0\right\}_{k}.
$
(Also see Corollary \ref{not dense in C^r}.) The usage of $C^r$-norm here is novel. For instance, with respect to $L^p$-norms, the space of functions computable by $\mathcal{N}$ lacks such a description, and often has undesirable properties like non-closedness \cite{petersen2020topological}. 
\begin{remark}\label{generic}
In Conjecture \ref{conjecture}, the set formed by functions excluded by the equations $\Psi_i=0$ is meager in both analytic and algebraic senses. 
The set cut off by the equations is a closed and (due to the term ``non-vacuous'' appearing in the conjecture) proper subset of the space of functions computable by $\mathcal{N}$, and a function implemented by $\mathcal{N}$ at which a $\Psi_i$ vanishes  could  be perturbed to another computable function at which all of $\Psi_i$'s are non-zero. In the algebraic setting, if the \textit{variety} ascertained by polynomials computable by $\mathcal{N}$ in an ambient polynomial space is irreducible, the non-empty subset defined by  $\left\{\Psi_i\neq 0\right\}_i$ is
\textit{Zariski open} and thus dense in the analytic topology. (See Appendix \ref{Background} for the algebraic geometry terminology.)  
\end{remark}

Conjecture \ref{conjecture} is settled in \cite{Farhoodi2019OnFC} for trees (a particular type of architectures) with distinct inputs, a situation in which no PDI is required and the inequalities should be taken to be trivial. Throughout the paper,  the conjecture above will be established for a number of architectures; in particular,  we shall  characterize \textit{tree functions} (cf. Theorems \ref{main-tree} and \ref{main-tree-activation} below).

\subsection{Related work}\label{Literature review}
There is an extensive literature on the \textit{expressive power} of neural networks. Although shallow networks with sigmoidal activation functions can approximate any continuous function on compact sets \cite{cybenko1989approximation,hornik1989multilayer,hornik1991approximation,mhaskar1996neural}, this cannot be achieved without the hidden layer getting exponentially large \cite{eldan2016power,telgarsky2016benefits,mhaskar2017and,poggio2017and}. Many articles thus try to demonstrate how the expressive power is affected by depth. This line of research draws on a number of different scientific fields including algebraic topology (\cite{bianchini2014complexity}), algebraic geometry (\cite{kileel2019expressive}),  dynamical systems (\cite{chatziafratis2019depth}), tensor analysis (\cite{cohen2016expressive}), Vapnik–Chervonenkis theory (\cite{bartlett1999almost}) and statistical physics (\cite{lin2017does}). One approach is to argue that deeper networks are able to approximate/represent functions of higher complexity after defining a ``complexity measure'' \cite{bianchini2014complexity,montufar2014number,poole2016exponential,telgarsky2016benefits,raghu2017expressive}. Another approach  more in line with this paper is to use the ``size'' of an associated functional space as a measure of representation power. This point of view is adapted in \cite[\S6]{Farhoodi2019OnFC} by enumerating Boolean functions, and in \cite{kileel2019expressive} by regarding dimensions of functional varieties as such a measure. 

A central result in the mathematical study of superpositions of functions is the celebrated \textit{Kolmogorov-Arnold Representation Theorem} \cite{MR0111809} which resolves (in the context of continuous functions) the $13^{\rm{th}}$ problem on Hilbert's famous list of $23$ major mathematical problems \cite{hilbert1902mathematical}. The theorem states that every continuous function $F(x_1,\dots,x_n)$ on the closed unit cube may be written as 
\begin{equation}\label{Kolmogorov-Arnold}
F(x_1,\dots,x_n)=\sum_{i=1}^{2n+1}f_i\left(\sum_{j=1}^n\phi_{i,j}(x_j)\right)
\end{equation}
for suitable continuous univariate functions $f_i$, $\phi_{i,j}$ defined on the unit interval. See \cite[chap. 1]{MR0237729} or \cite{vitushkin2004hilbert} for a historical account. In more refined versions of this theorem (\cite{MR210852,MR0213785}), the outer functions $f_i$ are arranged to be the same; and the inner ones  $\phi_{i,j}$ are be taken to be in the form of $\lambda_j\phi_i$ with $\lambda_j$'s and $\phi_i$'s independent of $F$.  Based on the existence of such an improved representation, Hecht-Nielsen argued that any continuous function $F$ can be implemented by a three-layer neural network whose weights and activation functions are determined by the representation \cite{hecht1987kolmogorov}.  
On the other hand, it is well known that even when $F$ is smooth one cannot arrange for functions appearing in the representation \eqref{Kolmogorov-Arnold} to be smooth  \cite{MR0165138}. As a matter of fact, there exist continuously differentiable functions  of three variables which cannot be represented as sums of superpositions of the form $g\left(f(x,y),z\right)$ with $f$ and $g$ being continuously differentiable as well (\cite{MR0062212}) whereas in the continuous category, one can write any trivariate continuous functions as a sum of nine superpositions of the form $g\left(f(x,y),z\right)$ (\cite{arnold2009representation}). Due to this emergence of non-differentiable functions, it has been argued that Kolmogorov-Arnold's Theorem is not useful for obtaining exact representations of functions via networks \cite{girosi1989representation}, although it may be used for approximation \cite{kuurkova1991kolmogorov,kuurkova1992kolmogorov}. More on algorithmic aspects of the theorem and its applications to the network theory could be found in \cite{brattka2007hilbert}.

Focusing on a superposition  
\begin{equation}
F=f^{(L)}_1\Bigg(f^{(L-1)}_1\left(f^{(L-2)}_{a_1}(\dots),\dots\right),\dots,f^{(L-1)}_j\left(f^{(L-2)}_{a_j}(\dots),\dots\right),\dots,f^{(L-1)}_{N_{L-1}}\left(f^{(L-2)}_{a_{N_{L-1}}}(\dots),\dots\right)\Bigg)    
\end{equation}
of smooth functions (which can be computed by a neural network as in Figure \ref{fig:superposition}), the chain rule provides descriptions for partial derivatives of $F$ in terms of partial derivatives of functions $f^{(i)}_j$ that constitute the superposition. The key insight behind the proof of Theorem \ref{main} is that when the former functions have fewer variables  compared to $F$, one may eliminate the derivatives of $f^{(i)}_j$'s to obtain relations among partial derivatives of $F$. This idea of elimination has been utilized in \cite{MR630992,MR609048} to prove the existence of  \textit{universal} algebraic differential equations whose $C^\infty$ solutions are dense in the space of continuous functions. The fact that there will be constraints imposed on derivatives of a function $F$ which is written as a superposition of differentiable functions was employed by  Hilbert himself to argue that  certain analytic functions of three variables are not superpositions of analytic functions of two variables \cite[p. 28]{arnold2009representation1}; and by Ostrowski to exhibit an analytic bivariate function that cannot be represented as a superposition of univariate smooth functions and multivariate algebraic functions due to the fact that it does not satisfy any non-trivial algebraic PDE \cite[p. 14]{vitushkin2004hilbert}, \cite{ostrowski1920dirichletsche}. The novelty of our approach is to adapt this point of view to demonstrate theoretical limitations of smooth functions that neural networks  compute either as a superposition as in Theorem \ref{main} or as compositions of functions of the form \eqref{activation} as in Theorem \ref{main'}; and to try to characterize these functions via calculating PDE constraints that are sufficient too (cf. Conjecture \ref{conjecture}). Furthermore, necessary PDE constraints enable us to easily exhibit functions that cannot be computed by a particular architecture; see Example \ref{basic-1}. This is reminiscent of the famous Minsky XOR Theorem \cite{minsky2017perceptrons}. An interesting non-example from the literature is $F(x,y,z)=xy+yz+zx$ that cannot be written as a superposition of the form \eqref{toy1} even in the continuous category \cite{MR0015435,MR541075,MR602748,MR606252,arnold2009representation1}. 

To the best of our knowledge, the closest mentions of a characterization of a class of superpositions by necessary and sufficient PDE constraints in the literature are papers \cite{MR541075,MR606252} of R. C. Buck. The first one (along with its earlier version \cite{buck1976approximate}) characterizes superpositions of the form $g(f(x,y),z)$ in a similar fashion as Example \ref{basic-1}. Also in those papers, superpositions such as $g(f(x,y),h(x,z))$ (appeared in Example \ref{basic-1'}) are discussed although only the existence of necessary PDE constraints is shown; see \cite[Lemma 7]{MR541075} and \cite[p. 141]{MR606252}. We shall exhibit a PDE characterization for superpositions of this form in Example \ref{toy example 1}. The aforementioned papers also characterize sufficiently differentiable \textit{nomographic} functions of the form $\sigma(f(x)+g(y))$ and $\sigma(f(x)+g(y)+h(z))$.

A special class of neural network architectures is provided by rooted trees where any output of a layer is passed to exactly one node from one of the layers above; see Figure \ref{fig:terminology}.  Investigating functions computable by trees is of neuroscientific interest because the morphology of the dendrites of a neuron processes information through a tree which is often binary \cite{kollins2005branching, gillette2015topological}. Assuming that the inputs to a tree are distinct, in our previous work \cite{Farhoodi2019OnFC} we have completely characterized the corresponding superpositions through formulating necessary and sufficient PDE constraints; a result which answers Conjecture \ref{conjecture} in positive for such architectures. 
\begin{remark}
The characterization suggested by the theorem below is a generalization of Example \ref{basic-1} which was concerned with smooth superpositions of the form \eqref{3var_form}. The characterization therein of such superpositions as solutions of PDE \eqref{auxiliary-basic-constraint'} has also appeared in paper \cite{MR541075} which we were not aware of while writing \cite{Farhoodi2019OnFC}. 
\end{remark}

\begin{figure}
    \centering
    \includegraphics[height=3cm]{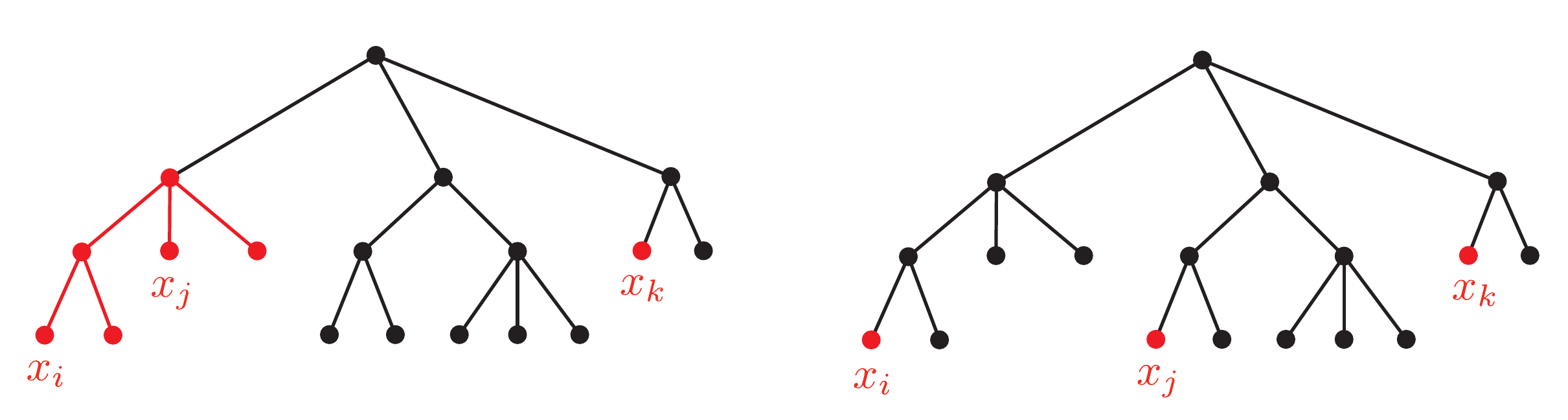}
    \caption{Theorems \ref{main-tree} and \ref{main-tree-activation} impose constraints \eqref{temp1} and \eqref{temp2} for any three leaves $x_i$, $x_j$ and $x_k$. In the former theorem, the constraint should hold whenever (as on the left) there exists a  rooted full sub-tree separating $x_i$ and $x_j$ from $x_k$ while in the latter theorem, the constraint is imposed for certain other triples as well (as on the right).}
    \label{fig:xixjxk}
\end{figure}

\begin{figure}
    \centering
    \includegraphics[height=3cm]{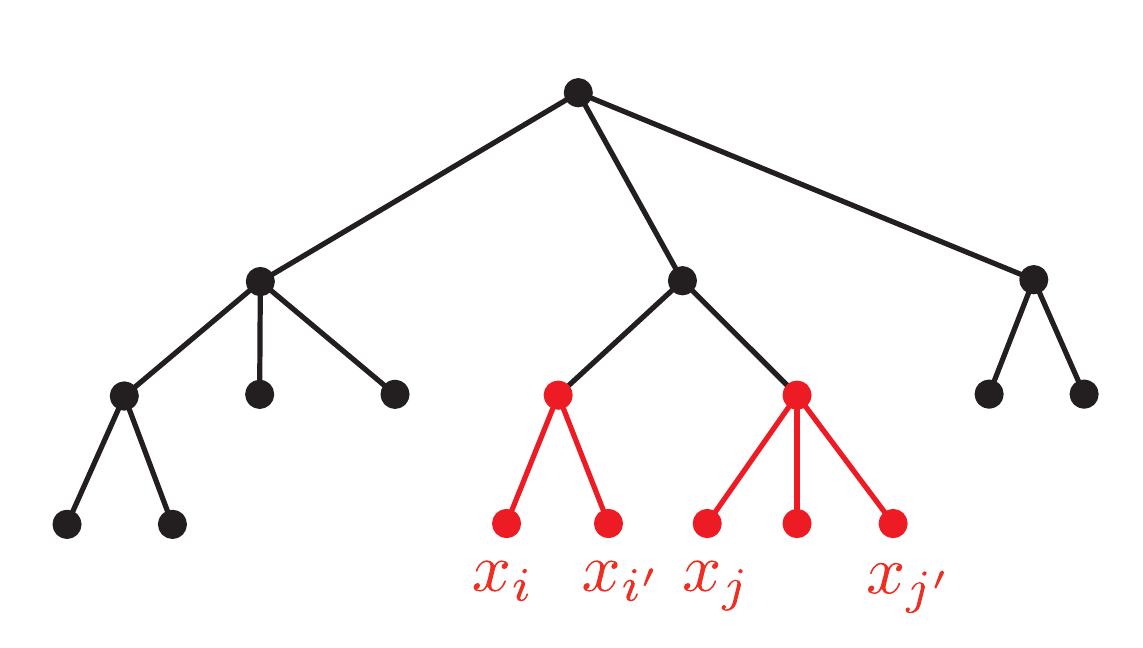}
    \caption{Theorem \ref{main-tree-activation} imposes constraint \eqref{temp3} for any four leaves $x_i,x_{i'}$ and $x_j,x_{j'}$ that belong to two different rooted full  sub-trees emanating from a node.}
    \label{fig:xixj}
\end{figure}

\begin{theorem}[\cite{Farhoodi2019OnFC}]\label{main-tree}
Let $\mathcal{T}$ be a rooted tree with $n$ leaves  that are labeled by the coordinate functions $x_1,\dots,x_n$. 
Let $F=F(x_1,\dots,x_n)$ be a smooth function implemented on this tree. 
Then for  any three leaves of $\mathcal{T}$ corresponding to variables $x_i,x_j,x_k$ of $F$ with the property that there is a (rooted full) sub-tree of $\mathcal{T}$ containing the leaves $x_i,x_j$ while missing the leaf $x_k$ (see Figure \ref{fig:xixjxk}), $F$ must satisfy
\begin{equation}\label{temp1}
F_{x_ix_k}F_{x_j} = F_{x_jx_k}F_{x_i}.
\end{equation}
Conversely, a smooth function $F$ defined in a neighborhood of a point $\mathbf{p}\in\Bbb{R}^n$ can be implemented by the tree $\mathcal{T}$ provided that \eqref{temp1} holds for any triple $(x_i,x_j,x_k)$ of its variables with the above property;  and moreover, the non-vanishing conditions below are satisfied: 
\begin{itemize}
\item [--] for any leaf $x_i$ with siblings either $F_{x_i}(\mathbf{p})\neq 0$ or there is a sibling leaf $x_{i'}$ with $F_{x_{i'}}(\mathbf{p})\neq 0$.
\end{itemize}
\end{theorem}
\noindent
This theorem  was formulated in \cite{Farhoodi2019OnFC} for binary trees and in the context of analytic functions (and also that of Boolean functions). Nevertheless, the proof carries over to the more general setting above. Below, we formulate the analogous characterization of functions that trees  compute via composing functions of the form \eqref{activation}. 
Proofs of Theorems \ref{main-tree} and \ref{main-tree-activation} shall be presented in \S\ref{generalization-tree1}.

\begin{theorem}\label{main-tree-activation}
Let $\mathcal{T}$ be a rooted tree admitting $n$ leaves  that are labeled by the coordinate functions $x_1,\dots,x_n$. We formulate the following constraints on smooth functions $F=F(x_1,\dots,x_n)$.   
\begin{itemize}
\item For any two leaves $x_i$ and $x_j$ of $\mathcal{T}$ we have  
\begin{equation}\label{temp2}
F_{x_ix_k}F_{x_j}=F_{x_jx_k}F_{x_i}
\end{equation}
for any other leaf $x_k$ of $\mathcal{T}$ which is not a leaf of a (rooted full) sub-tree that has exactly one of $x_i$ or $x_j$ (see Figure \ref{fig:xixjxk}). In particular, \eqref{temp2} holds for any $x_k$ if the leaves $x_i$ and $x_j$ are siblings; and for any $x_i$ and $x_j$ if the leaf $x_k$ is adjacent to the root of $\mathcal{T}$. 
\item For any two (rooted full) sub-trees $\mathcal{T}_1$ and $\mathcal{T}_2$ that emanate from a node of $\mathcal{T}$ (see Figure \ref{fig:xixj}), we have  
\begin{equation}\label{temp3}
\begin{split}
&F_{x_i}F_{x_j}\left[F_{x_ix_{i'}x_{j'}}F_{x_j}+F_{x_ix_{i'}}F_{x_jx_{j'}}-F_{x_ix_{j'}}F_{x_jx_{i'}}-F_{x_i}F_{x_jx_{i'}x_{j'}}\right]\\
&=\left(F_{x_ix_{i'}}F_{x_j}-F_{x_i}F_{x_jx_{i'}}\right)\left(F_{x_ix_{j'}}F_{x_j}+F_{x_i}F_{x_jx_{j'}}\right)    
\end{split}
\end{equation}
\normalsize
if $x_i$, $x_{i'}$ are leaves of $\mathcal{T}_1$ and $x_{j}$, $x_{j'}$ are leaves of $\mathcal{T}_2$.
\end{itemize}
These constraints are satisfied if $F(x_1,\dots,x_n)$ is a superposition of functions of the form $\mathbf{y}\mapsto\sigma\left(\langle\mathbf{w},\mathbf{y}\rangle\right)$ according to the hierarchy provided by $\mathcal{T}$. Conversely, a smooth function $F$ defined on an open box-like region\footnote{An open box-like region in $\Bbb{R}^n$ is a product $I_1\times\dots\times I_n$ of open intervals.} $B\subseteq\Bbb{R}^n$ can be written as such a superposition on $B$ provided that the constraints \eqref{temp2} and \eqref{temp3} formulated above hold and moreover, the non-vanishing conditions below are satisfied throughout $B$: 
\begin{itemize}
\item [--] for any leaf $x_i$ with siblings either $F_{x_i}\neq 0$ or there is a sibling leaf $x_{i'}$ with $F_{x_{i'}}\neq 0$;
\item  [--] for any leaf $x_i$ without siblings $F_{x_i}\neq 0$.
\end{itemize}
\end{theorem}

The constraints appeared in Theorems \ref{main-tree} and \ref{main-tree-activation} may seem tedious, but they can be rewritten more conveniently once the intuition behind them is explained. 
Assuming that partial derivatives do not vanish (a non-vanishing condition) so that division is allowed, \eqref{temp1} and \eqref{temp2} may be written as
\small
\begin{equation}\label{constraint-simplified-1}
\left(\frac{F_{x_i}}{F_{x_j}}\right)_{x_k}=0\Leftrightarrow \left(\frac{F_{x_i}}{F_{x_k}}\right)_{x_j}=\left(\frac{F_{x_j}}{F_{x_k}}\right)_{x_i};
\end{equation}
\normalsize
while \eqref{temp3} is
\small
\begin{equation}\label{constraint-simplified-2}
\left(\frac{\left(\frac{F_{x_i}}{F_{x_j}}\right)_{x_{i'}}}{\frac{F_{x_i}}{F_{x_j}}}\right)_{x_{j'}}=0.
\end{equation}
\normalsize
The first one, \eqref{constraint-simplified-1}, simply states that the ratio $\frac{F_{x_i}}{F_{x_j}}$ is independent of $x_k$. 
Notice that in comparison with Theorem \ref{main-tree}, Theorem \ref{main-tree-activation} requires the equation $F_{x_ix_k}F_{x_j}=F_{x_jx_k}F_{x_i}$ to hold in a greater generality and for more triples $(x_i,x_j,x_k)$ of leaves (see Figure \ref{fig:xixjxk}).\footnote{A piece of terminology introduced in \cite{Farhoodi2019OnFC} may be illuminating here: A member of a  triple $(x_i,x_j,x_k)$ of (not necessarily distinct) leaves of $\mathcal{T}$ is called the \textit{outsider} of the triple if there is a (rooted full) sub-tree of $\mathcal{T}$ that misses it but has the other two members. Theorem \ref{main-tree} imposes $F_{x_ix_k}F_{x_j}=F_{x_jx_k}F_{x_i}$ whenever $x_k$ is the outsider while Theorem \ref{main-tree-activation} imposes the constraint whenever $x_i$ and $x_j$ are not outsiders.} 
The second simplified equation \eqref{constraint-simplified-2}
holds once the function $\frac{F_{x_i}}{F_{x_j}}$ of $(x_1,\dots,x_n)$ may be split into a product such as   
$$q_1(\dots,x_i,\dots,x_{i'},\dots)\, q_2(\dots,x_j,\dots,x_{j'},\dots).$$
Lemma \ref{split} discusses the necessity and sufficiency of these equations for the existence of such a splitting.  

\begin{remark}\label{domain}
A significant feature of Theorem \ref{main-tree-activation} is that once the appropriate conditions are satisfied on a box-like domain, the smooth function under consideration may be written as a superposition of the desired form on the entirety of that domain. On the contrary, Theorem \ref{main-tree} is local in nature.
\end{remark}

Aside from neuroscientific interest, studying tree architectures is important also because any neural network can be expanded into a tree network with repeated inputs  through a procedure called \textbf{TENN} (the \textit{\textbf{T}ree \textbf{E}xpansion of the \textbf{N}eural \textbf{N}etwork}); see Figure \ref{fig:tree-expanded-network}. Tree architectures with repeated inputs are relevant in the context of neuroscience too because the inputs to neurons may be repeated \cite{schneider2016quantitative,gerhard2017conserved}. 
We have already seen an example of a network along with its TENN in Figure \ref{fig:basic-2} of Example \ref{basic-1'}. Both networks implement functions of the form $F(x,y,z)=g(f(x,y),h(x,z))$. Even for this simplest example of a tree architecture with repeated inputs, the derivation of characteristic PDEs is computationally involved and will 
be done in Example \ref{toy example 1}. This verifies Conjecture \ref{conjecture} for the tree appeared in Figure \ref{fig:basic-2}. Another small tree with repeated inputs for which Conjecture \ref{conjecture} is established is illustrated in Figure \ref{fig:asymmetric-toy} of Example \ref{toy example 2}. The computations of Examples \ref{toy example 1} and \ref{toy example 2} are generalized in \S\S\ref{generalization-tree2},\ref{generalization-tree3} where Conjecture \ref{conjecture} is verified for two families of trees with repeated inputs (Figures \ref{fig:symmetric} and \ref{fig:asymmetric}). Again, the characteristic PDEs are more cumbersome than those appearing in Theorems \ref{main-tree} and \ref{main-tree-activation} where the inputs are assumed to be distinct.

Neural networks with polynomial activation functions have been studied in the literature \cite{du2018power,soltanolkotabi2018theoretical,venturi2018spurious,kileel2019expressive}. 
By bounding the degrees of constituent functions of superpositions computed by a polynomial neural network, the functional space formed by these superpositions  sits inside a finite-dimensional ambient space of real polynomials and is hence finite-dimensional and amenable to techniques of algebraic geometry. One can, for instance, in each degree associate a \textit{functional variety} to a neural network $\mathcal{N}$ (see Definition \ref{variety}) whose dimension could be interpreted as a measure of expressive power.   This algebraic approach to expressivity has been adapted  in \cite{kileel2019expressive} for the case of neural networks for which the activation functions $\sigma$ in \eqref{activation} is a power map. Our approach of describing real functions computable by neural networks via PDEs and PDIs has ramifications to the study of polynomial neural networks as well. Paper  \cite{Farhoodi2019OnFC} for instance, utilizes Theorem \ref{main-tree} to write equations for algebraic varieties associated with tree architectures.   Indeed,  if $F=F(x_1,\dots,x_n)$ is a polynomial, an algebraic PDE of the form \eqref{alg PDE} translates to a polynomial equation of the coefficients of $F$ (see Corollary \ref{equations}); and the condition that an algebraic PDI such as \eqref{alg PDI} is valid throughout $\Bbb{R}^n$ can again be described via equations and inequalities involving the coefficients of $F$ (see Lemma \ref{semi-algebraic}). In \S\ref{Polynomial}, as the polynomial analogue of Conjecture \ref{conjecture}, we will present  Conjecture \ref{conjecture'} which aims to characterize, in a similar fashion,  functions computed by polynomial neural networks by the means of algebraic PDEs and PDIs. In view of the discussion above, this amounts to equations and inequalities describing the functional variety associated with $\mathcal{N}$ in the ambient polynomial space. This description with equations and inequalities is reminiscent of the notion of a \textit{semi-algebraic set} from real algebraic geometry. A novel feature of Conjecture \ref{conjecture'} is the claim of the existence of a universal characterization  dependent only on the architecture  from which  a description as a semi-algebraic set could be read off in any degree.

\begin{figure}
    \centering
    \includegraphics[width=12cm]{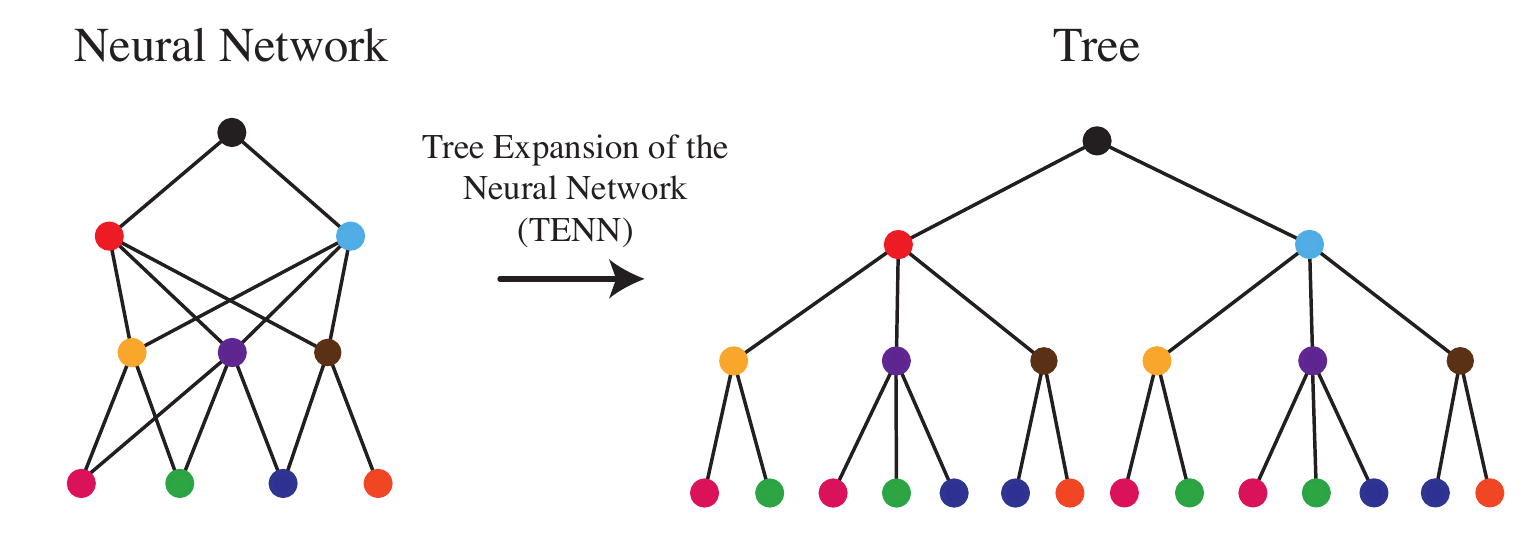}
    \caption{A multi-layer neural network can be expanded to a tree. The figure is adapted from \cite{Farhoodi2019OnFC}.}
    \label{fig:tree-expanded-network}
\end{figure}

\subsection{Outline of the paper}\label{outline}
Theorems \ref{main} and \ref{main'} are proven in  \S\ref{necessity} where it is established that in each setting there are necessary PDE conditions for expressibility of smooth functions by a neural network. In \S\ref{toy examples} we verify Conjecture \ref{conjecture} in several examples by characterizing computable functions via PDE constraints that are necessary and (given certain non-vanishing conditions) sufficient. This starts by studying tree architectures in \S\ref{toy examples-trees}: In Example \ref{toy example 1} we finish our treatment of a tree function with repeated inputs initiated in  Example \ref{basic-1'}; and  moreover, we present a number of examples to exhibit the key ideas of the proofs of Theorems \ref{main-tree} and \ref{main-tree-activation} which are concerned with tree functions with distinct inputs. The section then proceeds with switching from trees to other neural networks in \S\ref{toy examples-networks} where, building on Example \ref{basic-2}, Example \ref{network1} demonstrates why the characterization claimed by Conjecture \ref{conjecture} involves inequalities. Examples in \S\ref{toy examples} are generalized in the next section to a number of results establishing Conjecture \ref{conjecture} for certain families of tree architectures: Proofs of Theorems \ref{main-tree} and \ref{main-tree-activation} are presented in \S\ref{generalization-tree1}; and Examples \ref{toy example 1} and \ref{toy example 2} are generalized to tree architectures with repeated inputs in Propositions \ref{form1} and \ref{form2} of \S\ref{generalization-tree2} and \S\ref{generalization-tree3}. Section \ref{Polynomial} is concerned with polynomial neural networks and includes Conjecture \ref{conjecture'}, an algebraic variant  of Conjecture \ref{conjecture}. The last two section, \S\ref{conclusion} and \S\ref{future}, are devoted to few concluding remarks and possible future research directions. There are three appendices discussing technical proofs of propositions and lemmas (Appendix \ref{Proofs}), and the basic mathematical background on differential forms (Appendix \ref{Frobenius}) and algebraic geometry (Appendix \ref{Background}).

\section{Existence of PDE constraints}\label{necessity}
The goal of the section is to prove Theorems \ref{main} and \ref{main'}. Lemma below (to be proven in Appendix \ref{Background}) is our main tool for establishing the existence of constraints.

\begin{lemma}\label{dependence}
Any collection $p_1(t_1,\dots,t_m),\dots,p_l(t_1,\dots,t_m)$ of polynomials on $m$ indeterminates are algebraically dependent provided that $l>m$. In other words, if $l>m$ there exists a non-constant polynomial 
$\Phi=\Phi(s_1,\dots,s_l)$ dependent only on the coefficients of $p_i$'s for which 
$$
\Phi\left(p_1(t_1,\dots,t_m),\dots,p_l(t_1,\dots,t_m)\right)\equiv 0. 
$$
\end{lemma}

\begin{proof}[Proof of Theorem \ref{main}]
Let $F=F(x_1,\dots,x_n)$ be a superposition of smooth functions 
\begin{equation}\label{functions at nodes}
f^{(1)}_1,\dots,f^{(1)}_{N_1};\dots;f^{(i)}_{1},\dots,f^{(i)}_{N_i};\dots;f^{(L)}_1
\end{equation}
according to the hierarchy provided by $\mathcal{N}$ where $f^{(i)}_{1},\dots,f^{(i)}_{N_i}$ are the functions appearing at the neurons of the $i^{th}$ layer above the input layer (in the last layer, $f^{(L)}_{N_L=1}$ appears at the output neuron). The total number of these functions is $N:=N_1+\dots+N_L$; namely, the number of the neurons of the network. 
By the chain rule, any partial derivative $F_{\mathbf{x}^\alpha}$ of the superposition may be described as a polynomial of partial derivatives of order not greater than $|\alpha|$ of functions appeared in \eqref{functions at nodes}. These polynomials are determined solely by how neurons in consecutive layers are connected to each other; that is, the architecture. The function $F$ of $n$ variables admits $\binom{r+n}{n}-1$ partial derivatives (excluding the function itself) of order at most $r$ whereas the same number for any of the functions listed in \eqref{functions at nodes} is at most $\binom{r+n-1}{n-1}-1$ because by the hypothesis each of them is  dependent on less than $n$ variables. Denote the partial derivatives of order at most $r$ of functions $f^{(i)}_j$ (evaluated at appropriate points as required by the chain rule) by indeterminates $t_1,\dots,t_m$. Following the previous discussion, one has $m\leq N\left(\binom{r+n-1}{n-1}-1\right)$. Hence, the chain rule describes the partial derivatives of order not greater than $r$ of $F$ as polynomials -- dependent only on the architecture of $\mathcal{N}$-- of $t_1,\dots,t_m$. Invoking Lemma \ref{dependence}, the aforementioned partial derivatives of $F$ are algebraically dependent once 
\begin{equation}\label{inequality1}
\binom{r+n}{n}-1>N\left(\binom{r+n-1}{n-1}-1\right).    
\end{equation}
Indeed, the inequality holds for $r$ large enough since the left-hand side is a polynomial of degree $n$ of $r$ while the similar degree for the right-hand side is $n-1$.  
\end{proof}

\begin{proof}[Proof of Theorem \ref{main'}]
In this case $F=F(x_1,\dots,x_n)$ is a superposition of functions of the form 
\begin{equation}\label{functions at nodes'}
\sigma^{(1)}_1\left(\langle\mathbf{w}^{(1)}_1,.\rangle\right),\dots,\sigma^{(1)}_{N_1}\left(\langle\mathbf{w}^{(1)}_{N_1},.\rangle\right);
\dots;\sigma^{(i)}_1\left(\langle\mathbf{w}^{(i)}_1,.\rangle\right),\dots,\sigma^{(i)}_{N_i}\left(\langle\mathbf{w}^{(i)}_{N_i},.\rangle\right);\dots;\sigma^{(L)}_1\left(\langle\mathbf{w}^{(L)}_1,.\rangle\right)
\end{equation}
appearing at neurons: The $j^{\rm{th}}$ neuron of the $i^{\rm{th}}$ layer above the input layer ($1\leq i\leq N, 1\leq j\leq N_i$) corresponds to the function 
$\sigma^{(i)}_j\left(\langle\mathbf{w}^{(i)}_j,.\rangle\right)$
where a univariate smooth activation function $\sigma^{(i)}_j$ is applied to the inner product of the weight vector $\mathbf{w}^{(i)}_j$ with the vector formed by the outputs of neurons in the previous layer which are connected to the aforementioned neuron of the  $i^{\rm{th}}$ layer. We proceed as in the proof of Theorem \ref{main}: The chain rule describes each partial derivative $F_{\mathbf{x}^\alpha}$ as a polynomial -- dependent only on the architecture -- of components of vectors $\mathbf{w}^{(i)}_j$ along with derivatives of functions $\sigma^{(i)}_j$ up to order at most $|\alpha|$ (each evaluated at an appropriate point). The total number of components of all weight vectors coincides with the total number of connections (edges of the underlying graph); and the number of the aforementioned derivatives of activation functions is the number of neurons times $|\alpha|$. We denote  the total numbers of connections and neurons by $C$ and $N$ respectively. There are $\binom{r+n}{n}-1$  partial derivatives $F_{\mathbf{x}^\alpha}$ of order at most $r$ (i.e. $|\alpha|\leq r$) of $F$  and, by the previous discussion, each of them may be written as a polynomial of $C+Nr$ quantities given by components of weight vectors and derivatives of activation functions. Lemma \ref{dependence} implies that these partial derivatives of $F$ are algebraically dependent provided that 
\begin{equation}\label{inequality2}
\binom{r+n}{n}-1>Nr+C;    
\end{equation}
an inequality that holds for sufficiently large $r$  as the degree of the left-hand side with respect to $r$ is $n>1$. 
\end{proof}

\begin{corollary}\label{not dense in C^r}
Let $\mathcal{N}$ be a feedforward neural network  whose inputs are labeled by the coordinate functions $x_1,\dots,x_n$ and satisfies the hypothesis of either of Theorems \ref{main} or \ref{main'}. Define the positive integer $r$ as 
\begin{itemize}
\item $r=n\left(\#\text{neurons}-1\right)$ in the case of Theorem \ref{main},
\item $r=\max\left(\lfloor{n\left(\#\text{neurons}\right)^{\frac{1}{n-1}}\rfloor},\#\text{connections}\right)+2$ in the case of Theorem \ref{main'};
\end{itemize}
where $\#\text{connections}$ and  $\#\text{neurons}$ are respectively the number of edges of the underlying graph of $\mathcal{N}$ and the number of its vertices above the input layer. 
Then the smooth functions $F=F(x_1,\dots,x_n)$ computable  by $\mathcal{N}$ satisfy non-trivial algebraic partial differential equations of order $r$. In particular, the subspace formed by these functions lies in a subset of positive codimension which is closed with respect to the $C^r$-norm.
\end{corollary}

\begin{proof}
One only needs to verify that for the values of $r$ provided by the corollary the inequalities \eqref{inequality1} and \eqref{inequality2} are valid. The former holds if 
\small
$$
\frac{\binom{r+n}{n}}{\binom{r+n-1}{n-1}}=\frac{r+n}{n}
$$
\normalsize
is not smaller than $N$, i.e. if $r\geq n(N-1)$. As for \eqref{inequality2}, notice that 
\small
$$
\binom{r+n}{n}-1-Nr\geq \frac{r^n}{n!}-Nr=r\left(\frac{r^{n-1}}{n!}-N\right);
$$
\normalsize
hence it suffices to have $r\left(\frac{r^{n-1}}{n!}-N\right)> C$. This holds if $r>C$ and $\frac{r^{n-1}}{n!}-N\geq 1$. The latter inequality is valid once 
$r\geq n.N^{\frac{1}{n-1}}+2$, since then:
\small
$$
\frac{r^{n-1}}{n!}=\left(\frac{r}{(n!)^{\frac{1}{n-1}}}\right)^{n-1}\geq\left(\frac{r}{n}\right)^{n-1}\geq\left(N^{\frac{1}{n-1}}+\frac{2}{n}\right)^{n-1}
\geq N+\frac{2(n-1)}{n}.N^{\frac{n-2}{n-1}}\geq N+1.
$$
\normalsize
\end{proof}

\begin{remark}\label{dimension}
It indeed follows from the arguments above that there is a multitude of algebraically independent PDE constraints. By a simple dimension count, this number is $\left(\binom{r+n}{n}-1\right)-N\left(\binom{r+n-1}{n-1}-1\right)$ in the first case of Corollary \ref{not dense in C^r}, and is  $\left(\binom{r+n}{n}-1\right)-Nr$ in the second case. 
\end{remark}

\begin{remark}
The approach here merely establishes the existence of non-trivial algebraic PDEs satisfied by the superpositions. These are not the simplest PDEs of this kind and hence are not the best candidates for the purpose of characterizing superpositions. For instance, for superpositions \eqref{toy1} -- that networks in Figure \ref{fig:basic-2} implement -- one has $n=3$ and $\#\text{neurons}=3$. Corollary \ref{not dense in C^r} thus guarantees that these superpositions satisfy a sixth order PDE. But in Example \ref{toy example 1} we shall characterize them via two fourth order PDEs; compare with \cite[Lemma 7]{MR541075}.
\end{remark}

\begin{remark}\label{ODE}
Prevalent smooth activation functions such as the logistic function $\frac{1}{1+{\rm{e}}^{-x}}$ or tangent hyperbolic $\frac{{\rm{e}}^{x}-{\rm{e}}^{-x}}{{\rm{e}}^{x}+{\rm{e}}^{-x}}$ satisfy certain autonomous algebraic ODEs. Corollary \ref{not dense in C^r} could be improved in such a setting: If each activation function $\sigma=\sigma(x)$ appearing in 
\eqref{functions at nodes'} satisfies a differential equation of the form 
$$
\frac{{\rm{d}}^k\sigma}{{\rm{d}}x^k}=p\left(\sigma,\frac{{\rm{d}}\sigma}{{\rm{d}}x},\dots,\frac{{\rm{d}}^{k-1}\sigma}{{\rm{d}}x^{k-1}}\right)
$$
where  $p$ is a polynomial, one can change \eqref{inequality2} to $\binom{r+n}{n}-1>Nk_{\max}+C$ where $k_{\max}$ is the maximum order of ODEs that activation functions in 
\eqref{functions at nodes'} satisfy.
\end{remark}

\section{Toy examples}\label{toy examples}
The goal of the current section is to examine several elementary examples demonstrating how one can derive a set of necessary or sufficient PDE constraints for an architecture. 
The desired PDEs should be universal, i.e. purely in terms of the derivatives of the function $F$ which is to be implemented and not dependent on any weight vector, activation function or a function of lower dimensionality that has appeared at a node. In this process, it is often necessary to express a smooth function in terms of other functions. If $k<n$ and $f(x_1,\dots,x_n)$ is written as $g(\xi_1,\dots,\xi_k)$ throughout an open neighborhood of a point $\mathbf{p}\in\Bbb{R}^n$ where each $\xi_i=\xi_i(x_1,\dots,x_n)$ is a smooth function, the gradient of $f$ must be a linear combination of those of $\xi_1,\dots,\xi_k$ due to the chain rule. Conversely, if $\nabla f\in{\text{Span}}\{\nabla\xi_1,\dots,\nabla\xi_k\}$ near $\mathbf{p}$, by the Inverse Function Theorem one can extend $(\xi_1,\dots,\xi_k)$ to a coordinate system $(\xi_1,\dots,\xi_k;\xi_{k+1},\dots,\xi_n)$ on a small enough neighborhood of $\mathbf{p}$ provided that $\nabla\xi_1(\mathbf{p}),\dots,\nabla\xi_k(\mathbf{p})$ are linearly independent; a coordinate system in which the partial derivative $f_{\xi_i}$ vanish for $k<i\leq n$; the fact that implies  $f$ can be expressed in terms of $\xi_1,\dots,\xi_k$ near $\mathbf{p}$. Subtle mathematical issues arise if one wants to write $f$ as $g(\xi_1,\dots,\xi_k)$ on a larger domain containing $\mathbf{p}$:
\begin{itemize}
\item A $k$-tuple $(\xi_1,\dots,\xi_k)$ of smooth functions defined on an open subset $U$ of $\Bbb{R}^n$ whose gradient vector fields are linearly independent at all points cannot necessarily be extended to a coordinate system $(\xi_1,\dots,\xi_k;\xi_{k+1},\dots,\xi_n)$ for the whole $U$. As an example, consider $r=\sqrt{x^2+y^2}$ whose gradient is non-zero at any point of $\Bbb{R}^2-\{(0,0)\}$; but there is no smooth function $h:\Bbb{R}^2-\{(0,0)\}\rightarrow\Bbb{R}$ with $\nabla h\not\parallel\nabla r$ throughout $\Bbb{R}^2-\{(0,0)\}$: The level set $r=1$ is compact and so the restriction of $h$ to it achieves its absolute extrema, and at such points $\nabla h=\lambda\nabla f$ ($\lambda$ is the Lagrange multiplier). 
\item Even if one has a coordinate system $(\xi_1,\dots,\xi_k;\xi_{k+1},\dots,\xi_n)$ on a connected open subset $U$ of $\Bbb{R}^n$, a smooth function $f:U\rightarrow\Bbb{R}$ with 
$f_{\xi_{k+1}},\dots,f_{\xi_n}\equiv 0$ cannot necessarily be written globally as $f=g(\xi_1,\dots,\xi_k)$. One example is the function 
$$
f(x,y):=
\begin{cases}
0                       &\text{if } x\leq 0\\
{\rm{e}}^{\frac{-1}{x}} &\text{if } x>0, y>0\\
-{\rm{e}}^{\frac{-1}{x}}&\text{if } x>0, y<0\\
\end{cases}
$$
defined on the open subset $\Bbb{R}^2-[0,\infty)\subset\Bbb{R}^2$ for which $f_y\equiv 0$. It may only locally be written as $f(x,y)=g(x)$; there is no function $g:\Bbb{R}\rightarrow\Bbb{R}$ with $f(x,y)=g(x)$ for all $(x,y)\in\Bbb{R}^2-[0,\infty)$. Defining $g(x_0)$ as the value of $f$ on the intersection of its domain with the vertical line $x=x_0$ does not work because, due to the shape of the domain, such intersections may be disconnected. Finally, notice that $f$, although smooth, is not analytic ($C^\omega$); indeed, examples of this kind do not exist in the analytic category.  
\end{itemize}
This difficulty of needing a representation $f=g(\xi_1,\dots,\xi_k)$ that remains valid not just near a point but over a larger domain comes up only in the proof of Theorem \ref{main-tree-activation} (see Remark \ref{domain}); the representations we work with in the rest of this section are all local. The assumption about the shape of the domain and the special form of functions \eqref{activation} allows us to circumvent the difficulties just mentioned in the proof of Theorem \ref{main-tree-activation}. Below we have two related lemmas that shall be used later. 

\begin{lemma}\label{technical}
Let $B$ and $\mathcal{T}$ be a box-like region in $\Bbb{R}^n$ and a rooted tree with the coordinate functions $x_1,\dots,x_n$ labeling its leaves as in Theorem \ref{main-tree-activation}. Suppose a smooth function $F=F(x_1,\dots,x_n)$ on $B$ is implemented on $\mathcal{T}$ via assigning activation functions and weights to the nodes of $\mathcal{T}$. If $F$ satisfies the non-vanishing conditions described at the end of Theorem \ref{main-tree-activation}, then the level sets of $F$ are connected; and $F$ can be extended to a coordinate system $(F,F_2,\dots,F_n)$ for $B$.
\end{lemma}

\begin{lemma}\label{integrability'}
A smooth function $F(x_1,\dots,x_n)$ of the form $\sigma(a_1x_1+\dots+a_nx_n)$ satisfies $F_{x_ix_k}F_{x_j}=F_{x_jx_k}F_{x_i}$ for any $1\leq i,j,k\leq n$. Conversely,  if $F$ has a first order partial derivative $F_{x_j}$ which is non-zero throughout an open box-like region $B$ in  its domain, each  identity $F_{x_ix_k}F_{x_j}=F_{x_jx_k}F_{x_i}$ could be written as $\left(\frac{F_{x_i}}{F_{x_j}}\right)_{x_k}=0$; that is, for any $1\leq i\leq n$ the ratio 
$\frac{F_{x_i}}{F_{x_j}}$ should be constant on $B$; and such  requirements guarantee that $F$ admits a representation of the form $\sigma(a_1x_1+\dots+a_nx_n)$ on $B$.
\end{lemma}

In view of the discussion so far, it is important to know when a smooth vector field 
\begin{equation}\label{vector field}
\mathbf{V}(x_1,\dots,x_n)=[V_1(x_1,\dots,x_n)\,\,\dots\,\, V_n(x_1,\dots,x_n)]^{\rm{T}}
\end{equation}
on an open subset $U\subset\Bbb{R}^n$  is locally given by a gradient. Clearly, a necessary condition is to have
\begin{equation}\label{closed}
(V_i)_{x_j}=(V_j)_{x_i}\quad \forall\, i,j\in\{1,\dots,n\}.     
\end{equation}
It is well known that if $U$ is simply connected this condition is sufficient too and guarantees the existence of a smooth \textit{potential} function $\xi$ on $U$ satisfying 
$\nabla \xi=\mathbf{V}$ \cite{MR1886084}. A succinct way of writing \eqref{closed} is ${\rm{d}}\omega=0$ where $\omega$ is defined as the  \textit{differential form}
\begin{equation}\label{1-form}
\omega:=V_1\,{\rm{d}}x_1+\dots+V_n\,{\rm{d}}x_n.
\end{equation}
Here is a more subtle question also pertinent to our discussion: When $\mathbf{V}$ may be rescaled to a gradient vector field? As the reader may recall from the elementary theory of differential equations, for a planer vector field such a rescaling amounts to finding an \textit{integration factor} for the corresponding first order ODE \cite{boyce2012elementary}. 
It turns out that the answer could again be encoded in terms of differential forms: 
\begin{theorem}\label{integrability}
A smooth vector field $\mathbf{V}$ is parallel to a gradient vector field near each point only if the corresponding differential $1$-form $\omega$ satisfies $\omega\wedge{\rm{d}}\omega=0$. Conversely, if $\mathbf{V}$ is non-zero at a point $\mathbf{p}\in\Bbb{R}^n$ in vicinity of which  $\omega\wedge{\rm{d}}\omega=0$ holds, there exists a smooth function $\xi$ defined on a suitable open neighborhood of $\mathbf{p}$ that satisfies $\mathbf{V}\parallel\nabla\xi\neq\mathbf{0}$. 
In particular, in dimension two,  a nowhere vanishing vector field $\mathbf{V}$ is locally parallel to a nowhere vanishing gradient vector field while in dimension three that is the case if and only if 
$\mathbf{V}\,.\,{\rm{curl}}\mathbf{V}=\mathbf{0}$. 
\end{theorem}
\noindent
We refer the reader to Appendix \ref{Frobenius} for a proof and background on differential forms.

\subsection{Trees with four inputs}\label{toy examples-trees}
\begin{figure}
    \centering
    \includegraphics[height=5cm]{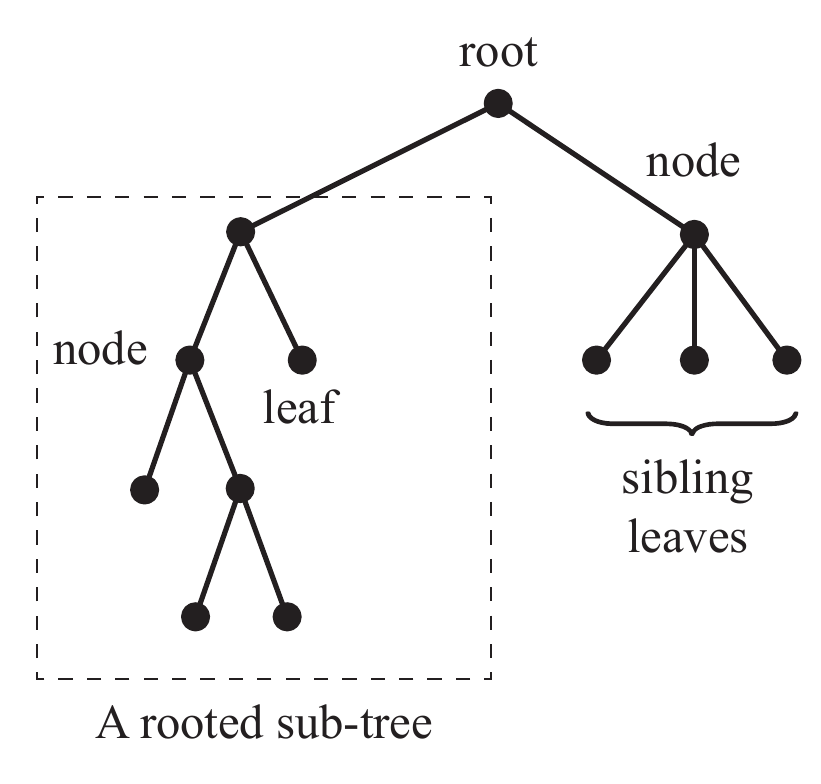}
    \caption{A tree architecture and the relevant terminology.}
    \label{fig:terminology}
\end{figure}

We begin with officially defining the terms related to tree architectures; see Figure \ref{fig:terminology}.

\begin{terminology} A \textbf{tree} is a connected acyclic graph. Singling out a vertex as its \textbf{root} turns it into a directed acyclic graph in which each vertex has a unique predecessor/parent. We take all trees to be rooted. The following notions come up frequently: \\
\begin{minipage}[t]{\linegoal}
\begin{itemize}
\item \textbf{Leaf}: a vertex with no successor/child.
\item \textbf{Node}: a vertex which is not a leaf, i.e. has children.
\item \textbf{Sibling leaves}: leaves with the same parent. 
\item \textbf{Sub-tree}: all descendants of a vertex along with the vertex itself. Hence in our convention all sub-trees are full and rooted. \vspace{2mm}
\end{itemize}
\end{minipage}
To implement a function, the leaves pass the inputs to the functions assigned to the nodes. The final output is received from the root. 
\end{terminology}

\begin{figure}
    \centering
    \includegraphics[height=2.2cm]{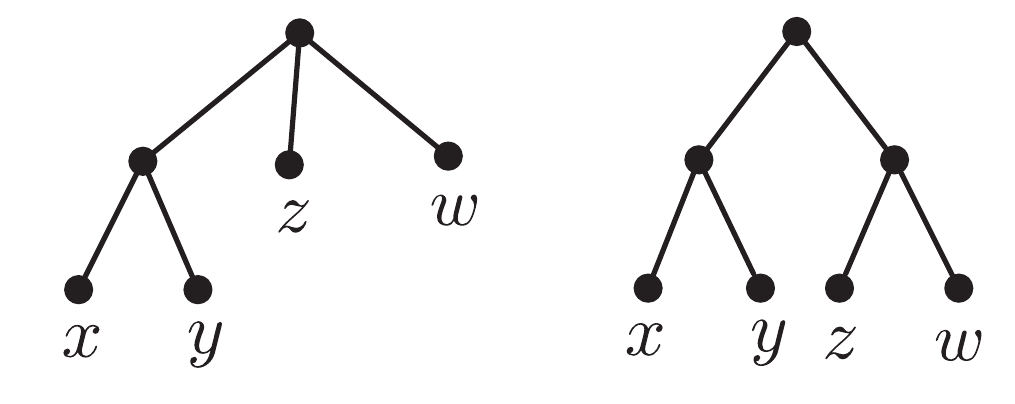}
    \caption{Two tree architectures with four distinct inputs. Examples \ref{example-four-1}, \ref{example-four-2} and \ref{example-four-3} characterize functions computable by them.}
    \label{fig:tree-four}
\end{figure}

The first example of the section elucidates Theorem \ref{main-tree}. 
\begin{example}\label{example-four-1}
Let us characterize superpositions $F(x,y,z,w)=g(f(x,y),z,w)$ of smooth functions $f,g$ which correspond to the first tree architecture in Figure \ref{fig:tree-four}. Necessary PDE constraints are more convenient to write for certain ratios. So to derive them, we assume for a moment that first order partial derivatives of $F$ are non-zero although, by a simple continuity argument, the constraints will hold regardless. Computing the numerator and the denominator of 
$\frac{F_x}{F_y}$ via the chain rule indicates that this ratio coincides with $\frac{f_x}{f_y}$ and is hence independent of $z,w$. One thus obtains 
\small
$$\left(\frac{F_y}{F_x}\right)_z=0, \quad\left(\frac{F_y}{F_x}\right)_w=0;$$ 
\normalsize
or equivalently
$$
F_{yz}F_x=F_{xz}F_y,\quad F_{yw}F_x=F_{xw}F_y.
$$
Assuming $F_x\neq 0$, the preceding constraints are sufficient: The gradient $\nabla F$ is parallel with 
\small
$$
\begin{bmatrix}
1\\
\frac{F_y}{F_x}\\[0.3em]
\frac{F_z}{F_x}\\[0.3em]
\frac{F_w}{F_x}
\end{bmatrix}
$$
\normalsize
where the second entry $\frac{F_y}{F_x}$ is dependent only on $x$ and $y$ and thus may be written as  $\frac{F_y}{F_x}=\frac{f_y}{f_x}$
for an appropriate bivariate function $f=f(x,y)$ defined throughout a small enough neighborhood of the point under consideration (at which $F_x$ is assumed to be non-zero). Such a function exists due to Theorem \ref{integrability}. Now we have
\small
$$
\nabla F\parallel
\begin{bmatrix}
1\\
\frac{f_y}{f_x}\\
0\\
0
\end{bmatrix}
+\frac{F_z}{F_x}
\begin{bmatrix}
0\\
0\\
1\\
0
\end{bmatrix}
+\frac{F_w}{F_x}
\begin{bmatrix}
0\\
0\\
0\\
1
\end{bmatrix}
\in{\text{Span}}\left\{\nabla f,\nabla z,\nabla w\right\},
$$
\normalsize
which guarantees that $F(x,y,z,w)$ may be written as a function of $f(x,y),z,w$.
\end{example}

The next two examples serve as an invitation to the proof of Theorem \ref{main-tree-activation} in \S\ref{generalization-tree1} and are concerned with  trees illustrated in Figure \ref{fig:tree-four}. 
 
\begin{example}\label{example-four-2}
Let us study the example above in the regime of activation functions: The goal is to characterize functions of the form $F(x,y,z,w)=\sigma(\tau(ax+by)+cz+dw)$. The ratios $\frac{F_y}{F_x}, \frac{F_z}{F_w}$ must be constant while $\frac{F_x}{F_z}$ and $\frac{F_x}{F_w}$ are dependent merely on  $x,y$ as they are equal to $\frac{a}{c}\tau'(ax+by)$ and $\frac{a}{d}\tau'(ax+by)$ respectively. Equating the corresponding partial derivatives with zero  we obtain the following PDEs:
\small
\begin{equation*}
\begin{matrix}
&F_{xy}F_x=F_{xx}F_{y}, &F_{yy}F_x=F_{xy}F_{y}, &F_{yz}F_x=F_{xz}F_{y}, &F_{yw}F_x=F_{xw}F_{y};\\
&F_{xz}F_{w}=F_{xw}F_z, &F_{yz}F_{w}=F_{yw}F_z, &F_{zz}F_{w}=F_{zw}F_z, &F_{zw}F_{w}=F_{ww}F_z;\\
&F_{xz}F_z=F_{zz}F_x, &F_{xw}F_z=F_{zw}F_{x}; &F_{xz}F_w=F_{zw}F_{x}, &F_{xw}F_w=F_{ww}F_{x}.
\end{matrix}
\end{equation*}
\normalsize
One can easily verify that they always hold for functions of the form above.  
We claim that under the assumptions of $F_x\neq 0$ and $F_w\neq 0$ these conditions guarantee the existence of a local representation of the form $\sigma(\tau(ax+by)+cz+dw)$ of $F$. Denoting $\frac{F_x}{F_w}$ by $\beta(x,y)$ and the constant functions $\frac{F_y}{F_x}$ and $\frac{F_z}{F_w}$ by $c_1$ and $c_2$ respectively, we have: 
\small
$$
\nabla F=
\begin{bmatrix}
F_x\\
F_y\\
F_z\\
F_w
\end{bmatrix}
\mathlarger{\parallel}
\begin{bmatrix}
\frac{F_x}{F_w}\\[0.3em]
\frac{F_y}{F_w}\\[0.3em]
\frac{F_z}{F_w}\\[0.3em]
1
\end{bmatrix}
=
\begin{bmatrix}
\beta(x,y)\\
c_1\beta(x,y)\\
c_2\\
1
\end{bmatrix}
\mathlarger{\parallel}\,
\nabla (f(x,y)+c_2z+w), $$
\normalsize
where $\nabla f=
\begin{bmatrix}
\beta(x,y)\\
c_1\beta(x,y)
\end{bmatrix}$. 
Such a potential function $f$ for 
$
\begin{bmatrix}
\beta(x,y)\\
c_1\beta(x,y)
\end{bmatrix}
=\begin{bmatrix}
\frac{F_x}{F_w}\\[0.3em]
\frac{F_y}{F_w}
\end{bmatrix}
$
exists since 
\small
$$\left(\frac{F_x}{F_w}\right)_y=\left(\frac{F_y}{F_w}\right)_x\Leftrightarrow\left(\frac{F_y}{F_x}\right)_w=0;$$
\normalsize
and it must be in the form of $\tau(ax+by)$ as $\frac{f_y}{f_x}=c_1$ is constant (Lemma \ref{integrability'}). Thus, $F$ is a function of $\tau(ax+by)+c_2z+w$ because the gradients are parallel. 
\end{example}

The next example is concerned with the symmetric tree in Figure \ref{fig:tree-four}. We shall need the following lemma:

\begin{lemma}\label{split}
Suppose a smooth function $q=q\left(y^{(1)}_1,\dots,y^{(1)}_{n_1};y^{(2)}_1,\dots,y^{(2)}_{n_2}\right)$ is written as a product 
\begin{equation}\label{product}
q_1\left(y^{(1)}_1,\dots,y^{(1)}_{n_1}\right)q_2\left(y^{(2)}_1,\dots,y^{(2)}_{n_2}\right)
\end{equation}
of smooth functions $q_1, q_2$. Then $q\,q_{y^{(1)}_{a}y^{(2)}_{b}}=q_{y^{(1)}_{a}}\,q_{y^{(2)}_{b}}$ for any $1\leq a\leq n_1$ and $1\leq b\leq n_2$. Conversely, for a smooth function $q$ defined on an open box-like region $B_1\times B_2\subseteq\Bbb{R}^{n_1}\times\Bbb{R}^{n_2}$, once $q$ is non-zero 
these identities guarantee the existence of such a product representation on $B_1\times B_2$.
\end{lemma}

\begin{example}\label{example-four-3}
We aim for characterizing  smooth functions of four variables which are of the form $F(x,y,z,w)=\sigma(\tau_1(ax+by)+\tau_2(cz+dw))$. Assuming for a moment that all first order partial derivatives are non-zero, the ratios $\frac{F_y}{F_x}, \frac{F_z}{F_w}$ must be constant 
while $\frac{F_x}{F_w}$ is equal to $\frac{a\tau'_1(ax+by)}{d\tau'_2(cz+dw)}$ and hence (along with its constant multiples $\frac{F_x}{F_z},  \frac{F_y}{F_z}, \frac{F_y}{F_w}$) splits into a product of bivariate functions of $x,y$ and $z,w$; a requirement which by Lemma \ref{split} is equivalent to the following identities:
\small
\begin{equation*}
\begin{split}
&\frac{F_x}{F_w}\left(\frac{F_x}{F_w}\right)_{xz}=\left(\frac{F_x}{F_w}\right)_{x}\left(\frac{F_x}{F_w}\right)_{z}, \quad \frac{F_x}{F_w}\left(\frac{F_x}{F_w}\right)_{xw}=\left(\frac{F_x}{F_w}\right)_{x}\left(\frac{F_x}{F_w}\right)_{w}, \\
&\frac{F_x}{F_w}\left(\frac{F_x}{F_w}\right)_{yz}=\left(\frac{F_x}{F_w}\right)_{y}\left(\frac{F_x}{F_w}\right)_{z}, \quad
\frac{F_x}{F_w}\left(\frac{F_x}{F_w}\right)_{yw}=\left(\frac{F_x}{F_w}\right)_{y}\left(\frac{F_x}{F_w}\right)_{w}.
\end{split}
\end{equation*}
\normalsize
After expanding and cross-multiplying, the identities above result in PDEs of the form \eqref{temp3} imposed on $F$ that hold for any smooth function of the form $F(x,y,z,w)=\sigma(\tau_1(ax+by)+\tau_2(cz+dw))$.
Conversely, we claim that if $F_x\neq 0$ and $F_w\neq 0$ then the aforementioned constraints guarantee that $F$ locally admits a representation of this form.  
Denoting the constants $\frac{F_y}{F_x}$ and $\frac{F_z}{F_w}$ by $c_1$ and $c_2$ respectively and writing $\frac{F_x}{F_w}\neq 0$ in the split form  
$\frac{\beta(x,y)}{\gamma(z,w)}$, we obtain
\small
$$
\nabla F=
\begin{bmatrix}
F_x\\
F_y\\
F_z\\
F_w
\end{bmatrix}
\mathlarger{\parallel}
\begin{bmatrix}
\frac{F_x}{F_w}\\[0.3em]
\frac{F_y}{F_w}\\[0.3em]
\frac{F_z}{F_w}\\[0.3em]
1
\end{bmatrix}
=
\begin{bmatrix}
\frac{\beta(x,y)}{\gamma(z,w)}\\[0.3em]
c_1\frac{\beta(x,y)}{\gamma(z,w)}\\
c_2\\
1
\end{bmatrix}
\mathlarger{\parallel}
\begin{bmatrix}
\beta(x,y)\\
c_1\beta(x,y)\\
c_2\gamma(z,w)\\
\gamma(z,w)
\end{bmatrix}.
$$
\normalsize
We desire functions $f=f(x,y)$ and $g=g(z,w)$ with 
$\nabla f=
\begin{bmatrix}
\beta(x,y)\\
c_1\beta(x,y)\\
\end{bmatrix}$
and
$\nabla g=
\begin{bmatrix}
c_2\gamma(z,w)\\
\gamma(z,w)\\
\end{bmatrix}$;
because then $\nabla F\parallel\nabla (f(x,y)+g(z,w))$ and hence $F=\sigma\left(f(x,y)+g(z,w)\right)$ for an appropriate $\sigma$. Notice that $f(x,y)$ and $g(z,w)$ are automatically in the forms of $\tau_1(ax+by)$ and $\tau_2(cz+dw)$ because $\frac{f_y}{f_x}=c_1$ and $\frac{f_z}{f_w}=c_2$ are constants (see Lemma \ref{integrability'}). To establish the existence of $f$ and $g$ one should verify the integrability conditions 
$\beta_y=c_1\beta_x$ and  $c_2\gamma_w=\gamma_z$.
We only verify the first one and the second one is similar. Notice that $\frac{F_y}{F_x}=c_1$ is constant, and  $\frac{F_x}{F_w}=\frac{\beta(x,y)}{\gamma(z,w)}$
implies that
$\beta_x=\beta\frac{\left(\frac{F_x}{F_w}\right)_x}{\frac{F_x}{F_w}}$ while $\beta_y=\beta\frac{\left(\frac{F_x}{F_w}\right)_y}{\frac{F_x}{F_w}}$. So the question is whether 
\small
$$
\frac{F_y}{F_x}\left(\frac{F_x}{F_w}\right)_x=\left(\frac{F_y}{F_x}\frac{F_x}{F_w}\right)_x=\left(\frac{F_y}{F_w}\right)_x
$$
\normalsize
and $\left(\frac{F_x}{F_w}\right)_y$ coincide, which is the case since $\left(\frac{F_y}{F_w}\right)_x=\left(\frac{F_x}{F_w}\right)_y$ can be rewritten as $\left(\frac{F_y}{F_x}\right)_w=0$.  
\end{example}

\begin{remark}\label{reverse-engineer}
Examples \ref{example-four-2} and \ref{example-four-3} demonstrate an interesting phenomenon: One can deduce non-trivial facts about the weights once a formula for the implemented function is available. In Example \ref{example-four-2}, for a function $F(x,y,z,w)=\sigma(\tau(ax+by)+cz+dw)$ we have $\frac{F_y}{F_x}\equiv\frac{b}{a}$ and $\frac{F_z}{F_w}\equiv\frac{c}{d}$. The same identities are valid for functions of the form  $F(x,y,z,w)=\sigma(\tau_1(ax+by)+\tau_2(cz+dw))$ appeared in Example \ref{example-four-3}. Notice that this is the best one can hope to recover because through scaling the weights and inversely scaling the inputs of activation functions, the function $F$ could also be written as $\sigma(\tilde{\tau}(\lambda ax+\lambda by)+cz+dw)$ or 
$\sigma(\tilde{\tau_1}(\lambda ax+\lambda by)+\tau_2(cz+dw))$ where $\tilde{\tau}(y):=\tau\left(\frac{y}{\lambda}\right)$ and $\tilde{\tau_1}(y):=\tau_1\left(\frac{y}{\lambda}\right)$. Thus the other ratios $\frac{a}{c}$ and $\frac{b}{d}$ are completely arbitrary.
\end{remark}

\begin{example}\label{toy example 1}
Let us go back to Example \ref{basic-1'}: In \cite[\S 7.2]{Farhoodi2019OnFC}, a PDE constraint on functions of the form \eqref{toy1}  is obtained via differentiating \eqref{auxiliary7} several times and forming a matrix equation which implies that a certain determinant of partial derivatives must vanish. The paper then raises the question of existence of PDE constraints that are both necessary and sufficient. The goal of this example is to derive such a characterization. Applying differentiation operators $\partial_y$, $\partial_z$ and $\partial_{yz}$ to \eqref{auxiliary7} results in the matrix equation below:  
\small
$$
\begin{bmatrix}
F_y    &F_z    &0     &0  \\
F_{yy} &F_{yz} &F_y   &0  \\
F_{yz} &F_{zz} &0     &F_z\\
F_{yyz}&F_{yzz}&F_{yz}&F_{yz}
\end{bmatrix}
\begin{bmatrix}
A\\
B\\
A_y\\
B_z
\end{bmatrix}
=
\begin{bmatrix}
F_x\\
F_{xy}\\
F_{xz}\\
F_{xyz}
\end{bmatrix}.
$$
\normalsize
If the matrix above is non-singular -- which is a non-vanishing condition -- Cramer's rule provides descriptions of $A,B$ in terms of partial derivatives of $F$, and then $A_z=B_y=0$ yield PDE constraints. Reversing this procedure, we show that these conditions are sufficient too. Let us assume that 
\small
\begin{equation}\label{toy1-1}
\Psi:=
\begin{vmatrix}
F_y    &F_z    &0     &0\\
F_{yy} &F_{yz} &F_y   &0\\
F_{yz} &F_{zz} &0     &F_z\\
F_{yyz}&F_{yzz}&F_{yz}&F_{yz}
\end{vmatrix}
=(F_y)^2F_zF_{yzz}-(F_y)^2F_{yz}F_{zz}-F_y(F_z)^2F_{yyz}+(F_z)^2F_{yz}F_{yy}\neq 0.
\end{equation}
\normalsize
Notice that this condition is non-vacuous for functions $F(x,y,z)$ of the form \eqref{toy1} since they include all functions of the form $g(y,z)$.
Then the linear system 
\small
\begin{equation}\label{system}
\begin{bmatrix}
F_x\\
F_{xy}\\
F_{xz}\\
F_{xyz}
\end{bmatrix}
=
\begin{bmatrix}
F_y    &F_z    &0     &0\\
F_{yy} &F_{yz} &F_y   &0\\
F_{yz} &F_{zz} &0     &F_z\\
F_{yyz}&F_{yzz}&F_{yz}&F_{yz}
\end{bmatrix}
\begin{bmatrix}
A\\
B\\
C\\
D
\end{bmatrix}
\end{equation}
\normalsize
may be solved as
\footnotesize
\begin{equation}\label{huge1}
A=\frac{\begin{vmatrix}
F_x&F_z&0&0\\
F_{xy}&F_{yz}&F_y&0\\
F_{xz}&F_{zz}&0&F_z\\
F_{xyz}&F_{yzz}&F_{yz}&F_{yz}
\end{vmatrix}
}
{\begin{vmatrix}
F_y&F_z&0&0\\
F_{yy}&F_{yz}&F_y&0\\
F_{yz}&F_{zz}&0&F_z\\
F_{yyz}&F_{yzz}&F_{yz}&F_{yz}
\end{vmatrix}
}
=\frac{1}{\Psi}\left[-F_y(F_z )^2F_{xyz}+F_yF_zF_{xz}F_{yz}+F_xF_yF_zF_{yzz}-F_xF_yF_{yz}F_{zz}+(F_z )^2F_{xy}F_{yz}-F_xF_z(F_{yz})^2\right];
\end{equation}
\normalsize
and 
\footnotesize
\begin{equation}\label{huge2}
B=\frac{\begin{vmatrix}
F_y&F_x&0&0\\
F_{yy}&F_{xy}&F_y&0\\
F_{yz}&F_{xz}&0&F_z\\
F_{yyz}&F_{xyz}&F_{yz}&F_{yz}
\end{vmatrix}
}
{\begin{vmatrix}
F_y&F_z&0&0\\
F_{yy}&F_{yz}&F_y&0\\
F_{yz}&F_{zz}&0&F_z\\
F_{yyz}&F_{yzz}&F_{yz}&F_{yz}
\end{vmatrix}
}
=\frac{1}{\Psi}\left[(F_y)^2F_zF_{xyz}-(F_y)^2F_{xz}F_{yz}-F_yF_zF_{xy}F_{yz}-F_xF_yF_zF_{yyz}+F_xF_y(F_{yz})^2+F_xF_zF_{yy}F_{yz}\right].
\end{equation}
\normalsize
Denote the numerators of \eqref{huge1} and \eqref{huge2} by $\Psi_1$ and $\Psi_2$ respectively:
\small
\begin{equation}\label{toy1-2}
\begin{split}
&\Psi_1=-F_y(F_z )^2F_{xyz}+F_yF_zF_{xz}F_{yz}+F_xF_yF_zF_{yzz}-F_xF_yF_{yz}F_{zz}+(F_z )^2F_{xy}F_{yz}-F_xF_z(F_{yz})^2,\\
&\Psi_2=(F_y)^2F_zF_{xyz}-(F_y)^2F_{xz}F_{yz}-F_yF_zF_{xy}F_{yz}-F_xF_yF_zF_{yyz}+F_xF_y(F_{yz})^2+F_xF_zF_{yy}F_{yz}.
\end{split}
\end{equation}
\normalsize
Requiring $A=\frac{\Psi_1}{\Psi}$ and $B=\frac{\Psi_2}{\Psi}$ to be independent of $z$ and $y$ respectively amounts to the following:  
\begin{equation}\label{toy1-3}
\Phi_1:=(\Psi_1)_z\Psi-\Psi_1\Psi_z=0,\quad \Phi_2:=(\Psi_2)_y\Psi-\Psi_2\Psi_y=0.
\end{equation}
A simple continuity argument demonstrates that the constraints $\Phi_1=0$ and $\Phi_2=0$ above are necessary even if the determinant \eqref{toy1-1} vanishes: If $\Psi$ is identically zero on a neighborhood of a point $\mathbf{p}\in\Bbb{R}^3$, the identities \eqref{toy1-3} obviously hold throughout that neighborhood. Another possibility is that $\Psi(\mathbf{p})=0$ but there is a sequence $\{\mathbf{p}_n\}_n$ of nearby points with $\mathbf{p}_n\to\mathbf{p}$ and $\Psi(\mathbf{p}_n)\neq 0$. Then the polynomial expressions $\Phi_1$, $\Phi_2$ of partial derivatives vanish at any $\mathbf{p}_n$ and hence at $\mathbf{p}$ by continuity.   \\  
\indent 
To finish the verification of Conjecture \ref{conjecture} for superpositions of the form \eqref{toy1}, one should establish that PDEs $\Phi_1=0$, $\Phi_2=0$ from \eqref{toy1-3} are  sufficient for the existence of such a representation provided that the non-vanishing condition $\Psi\neq 0$ from \eqref{toy1-1} holds: In that case, the functions $A$ and $B$ from \eqref{huge1} and \eqref{huge2} satisfy \eqref{auxiliary7}. According to Theorem \ref{integrability}, there exist smooth locally defined $f(x,y)$ and $h(x,z)$ with $\frac{f_x}{f_y}=A(x,y)$ and $\frac{h_x}{h_z}=B(x,z)$. We have:
\small
$$
\nabla F=
\begin{bmatrix}
A(x,y)F_y+B(x,z)F_z\\
F_y\\
F_z
\end{bmatrix}
=F_y\begin{bmatrix} A(x,y)\\1\\0   \end{bmatrix}+F_z\begin{bmatrix} B(x,z)\\0\\1    \end{bmatrix}=
\frac{F_y}{f_y}\begin{bmatrix} f_x\\f_y\\0   \end{bmatrix}+\frac{F_z}{h_z}\begin{bmatrix} h_x\\0\\h_z   \end{bmatrix}\in\text{Span}\{\nabla f,\nabla h\};
$$
\normalsize
hence $F$ can be written as a function $g(f(x,y),h(x,z))$ of $f$ and $h$ for an appropriate $g$.

\end{example}

\begin{example}\label{toy example 2}

\begin{figure}
    \centering
    \includegraphics[height=3cm]{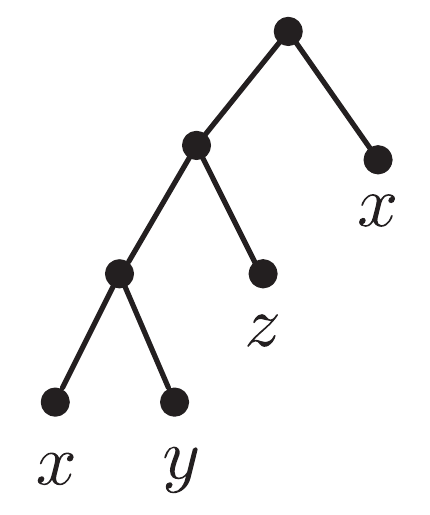}
    \caption{An asymmetric tree architecture that computes the superpositions of the form $F(x,y,z)=g(x,f(h(x,y),z))$. These are characterized in Example \ref{toy example 2}.}
    \label{fig:asymmetric-toy}
\end{figure}

We now turn into the asymmetric tree with four repeated inputs in Figure \ref{fig:asymmetric-toy} with the corresponding superpositions
\begin{equation}\label{toy2}
F(x,y,z)=g(x,f(h(x,y),z)).
\end{equation}
In our treatment here, the steps are reversible and  we hence derive PDE constraints that are simultaneously necessary and sufficient. The existence of a representation of the form \eqref{toy2} for $F(x,y,z)$ is equivalent to the existence of a locally defined coordinate system 
$$
(\xi:=x,\zeta,\eta)
$$
with respect to which $F_\eta=0$; and moreover, $\zeta=\zeta(x,y,z)$ must be in the form of  $f(h(x,y),z)$ which, according to Example \ref{basic-1}, is the case if and only if $\left(\frac{\zeta_y}{\zeta_x}\right)_z=0$. Here, we assume that $\zeta_x,\zeta_y\neq 0$ so that  $\frac{\zeta_y}{\zeta_x}$ is well defined and $\nabla\xi,\nabla\zeta$ are linearly independent. We denote the preceding ratio by $\beta=\beta(x,y)\neq 0$. Conversely, Theorem \ref{integrability} guarantees that there exists $\zeta$ with $\frac{\zeta_y}{\zeta_x}=\beta$ for any smooth $\beta(x,y)$. The function $F$ could be locally written as a function of $\xi=x$ and $\zeta$ if and only if  
\small
$$
\nabla F=
\begin{bmatrix}
F_x\\
F_y\\
F_z
\end{bmatrix}\in
\text{Span}\left\{\nabla x,\nabla \zeta=
\begin{bmatrix}
\zeta_x\\
\zeta_y=\beta(x,y)\zeta_x\\
\zeta_z
\end{bmatrix}
\right\}.
$$
\normalsize
Clearly, this occurs if and only if $\frac{F_z}{F_y}$ coincides with $\frac{\zeta_z}{\zeta_y}$. Therefore, one only needs to arrange for $\beta(x,y)$ so that the vector field
\small
$$
\frac{1}{\zeta_x}\nabla\zeta=
\begin{bmatrix}
1\\
\vspace{1mm}
\frac{\zeta_y}{\zeta_x}\\
\frac{\zeta_z}{\zeta_x}
\end{bmatrix}
=
\begin{bmatrix}
1\\
\beta(x,y)\\
\beta(x,y)\frac{F_z}{F_y}
\end{bmatrix}
$$
\normalsize
is parallel to a gradient vector field $\nabla\zeta$. That is to say, we want the vector field to be perpendicular to its curl; cf. Theorem \ref{integrability}. We have:
\small
\begin{equation*}
\begin{split}
&\left(\frac{\partial}{\partial x}+\beta(x,y)\frac{\partial}{\partial y}+\beta(x,y)\frac{F_z}{F_y}\frac{\partial}{\partial z}\right)\,.\,{\rm{curl}}\left(\frac{\partial}{\partial x}+\beta(x,y)\frac{\partial}{\partial y}+\beta(x,y)\frac{F_z}{F_y}\frac{\partial}{\partial z}\right)\\
&=\beta_y\frac{F_z}{F_y}+\beta\left(\frac{F_z}{F_y}\right)_y-\beta^2\left(\frac{F_z}{F_y}\right)_x.
\end{split}    
\end{equation*}
\normalsize
The vanishing of the expression above results in a description of $\left(\frac{F_z}{F_y}\right)_x$ as the linear combination
\small
\begin{equation}\label{auxiliary8}
\left(\frac{F_z}{F_y}\right)_x=\frac{\beta_y}{\beta^2}\frac{F_z}{F_y}+\frac{1}{\beta}\left(\frac{F_z}{F_y}\right)_y
\end{equation}
\normalsize
whose coefficients $\frac{\beta_y}{\beta^2}=-\left(\frac{1}{\beta}\right)_y$ and $\frac{1}{\beta}$ are independent of $z$. Thus, we are in a situation similar to that of Examples \ref{basic-1'} and \ref{toy example 1} where we encountered identity \eqref{auxiliary7}. The same idea used there could be applied again to obtain PDE constraints: Differentiating \eqref{auxiliary8} with respect to $z$ results in a linear system 
\small
$$
\begin{bmatrix}
\frac{F_z}{F_y} & \left(\frac{F_z}{F_y}\right)_y\\
\left(\frac{F_z}{F_y}\right)_z &\left(\frac{F_z}{F_y}\right)_{yz}
\end{bmatrix}
\begin{bmatrix}
-\left(\frac{1}{\beta}\right)_y\\
\frac{1}{\beta}
\end{bmatrix}
=\begin{bmatrix}
\left(\frac{F_z}{F_y}\right)_x\\
\left(\frac{F_z}{F_y}\right)_{xz}
\end{bmatrix}.
$$
\normalsize
Assuming the matrix above is non-singular, the Cramer's rule implies:
\small
\begin{equation}\label{huge3}
-\left(\frac{1}{\beta}\right)_y=\frac{
\begin{vmatrix}
\left(\frac{F_z}{F_y}\right)_x & \left(\frac{F_z}{F_y}\right)_y\\
\left(\frac{F_z}{F_y}\right)_{xz} &\left(\frac{F_z}{F_y}\right)_{yz}
\end{vmatrix}}
{\begin{vmatrix}
\frac{F_z}{F_y} & \left(\frac{F_z}{F_y}\right)_y\\
\left(\frac{F_z}{F_y}\right)_z &\left(\frac{F_z}{F_y}\right)_{yz}
\end{vmatrix}}, 
\quad 
\frac{1}{\beta}=\frac{
\begin{vmatrix}
\frac{F_z}{F_y} & \left(\frac{F_z}{F_y}\right)_x\\
\left(\frac{F_z}{F_y}\right)_z &\left(\frac{F_z}{F_y}\right)_{xz}
\end{vmatrix}}
{\begin{vmatrix}
\frac{F_z}{F_y} & \left(\frac{F_z}{F_y}\right)_y\\
\left(\frac{F_z}{F_y}\right)_z &\left(\frac{F_z}{F_y}\right)_{yz}
\end{vmatrix}}.  
\end{equation}
\normalsize
We now arrive at the desired PDE characterization of superpositions \eqref{toy2}: In each of the ratios of determinants appearing in \eqref{huge3}, the numerator and denominator are in the form of polynomials of partial derivatives divided by $(F_y)^4$. So we introduce the following polynomial expressions:
\small
\begin{equation}\label{toy2-1}
\Psi_1=(F_y)^4\begin{vmatrix}
\frac{F_z}{F_y} & \left(\frac{F_z}{F_y}\right)_y\\
\left(\frac{F_z}{F_y}\right)_z &\left(\frac{F_z}{F_y}\right)_{yz}
\end{vmatrix},\quad 
\Psi_2=(F_y)^4\begin{vmatrix}
\frac{F_z}{F_y} & \left(\frac{F_z}{F_y}\right)_x\\
\left(\frac{F_z}{F_y}\right)_z &\left(\frac{F_z}{F_y}\right)_{xz}
\end{vmatrix},\quad 
\Psi_3=(F_y)^4\begin{vmatrix}
\left(\frac{F_z}{F_y}\right)_x & \left(\frac{F_z}{F_y}\right)_y\\
\left(\frac{F_z}{F_y}\right)_{xz} &\left(\frac{F_z}{F_y}\right)_{yz}
\end{vmatrix}.
\end{equation}
\normalsize
Then in view of \eqref{huge3}: 
\small
\begin{equation}\label{toy2-1'}
\frac{\Psi_2}{\Psi_1}=\frac{1}{\beta}, \quad \frac{\Psi_3}{\Psi_1}=-\left(\frac{1}{\beta}\right)_y.
\end{equation}
\normalsize
Hence $\left(\frac{\Psi_2}{\Psi_1}\right)_y+\frac{\Psi_3}{\Psi_1}=0$; and furthermore, $\left(\frac{\Psi_2}{\Psi_1}\right)_z=0$ since $\beta$ is independent of $z$:
\begin{equation}\label{toy2-2}
\Phi_1:=\Psi_1(\Psi_2)_y-(\Psi_1)_y\Psi_2+\Psi_1\Psi_3=0,\quad
\Phi_2:=\Psi_1(\Psi_2)_z-(\Psi_1)_z\Psi_2=0.    
\end{equation}
Again as in Example \ref{toy example 1}, a continuity argument implies that the algebraic PDEs above are necessary even when the denominator in \eqref{huge3} (i.e. $\Psi_1$) is zero. As for the non-vanishing conditions, in view of \eqref{toy2-1} and \eqref{toy2-1'}, we require $F_y$ to be non-zero as well as  $\Psi_1$ and $\Psi_2$ (recall that $\beta\neq 0$):
\begin{equation}\label{toy2-3}
\Psi_1\neq 0, \Psi_2\neq 0, F_y\neq 0.  
\end{equation}
It is easy to see that these conditions are not vacuous for functions of the form \eqref{toy2}: If $F(x,y,z)=(xy)^2z+z^3$ neither $F_y$ nor the expressions $\Psi_1$ or $\Psi_2$ is identically zero.\\
\indent In summary, a special case of Conjecture \ref{conjecture} has been verified in this example: A function $F=F(x,y,z)$ of the form \eqref{toy2} satisfies the constraints \eqref{toy2-2}; and conversely, a smooth functions satisfying them along with the non-vanishing conditions \eqref{toy2-3} admits a local representation of that form.  
\end{example}

\subsection{Examples of functions computed by neural networks}\label{toy examples-networks}
We now switch from trees to examples of PDE constraints for neural networks. The first two examples are concerned with the network illustrated on the left of Figure \ref{fig:basic-tree}; this is a ResNet with two hidden layers that has $x,y,z,w$ as its inputs. The functions it implements are in the form of 
\begin{equation}\label{auxiliary13}
F(x,y,z,w)=g(h_1(f(x,y),z),h_2(f(x,y),w))    
\end{equation}
where $f$ and $h_1,h_2$ are the functions appearing in the hidden layers.

\begin{example}\label{network2}
On the right of Figure \ref{fig:basic-tree}, the tree architecture corresponding to the neural network discussed above is illustrated. The functions implemented by this tree are in the form of 
\begin{equation}\label{auxiliary13'}
F(x,y,z,w)=g(h_1(f_1(x,y),z),h_2(f_2(x,y),w))    
\end{equation}
which is a  form more general than the form \eqref{auxiliary13} of functions  computable by the network. As a matter of fact, there are PDEs satisfied by the latter class which functions in the former class  \eqref{auxiliary13'} do not necessarily satisfy.   To see this, observe that for a function $F(x,y,z,w)$ of the form \eqref{auxiliary13} the ratio $\frac{F_y}{F_x}$ coincides with $\frac{f_y}{f_x}$ and is thus independent of $z$ and $w$; hence the PDEs $F_{xz}F_{y}=F_{yz}F_{x}$ and $F_{xw}F_{y}=F_{yw}F_{x}$. Neither of them holds for the function $F(x,y,z,w)=xyz+(x+y)w$ which is of the form \eqref{auxiliary13'}.
We deduce that the set of PDE constraints for a network may be strictly larger than that of the  corresponding TENN. 
\begin{figure}
    \centering
    \includegraphics[height=3cm]{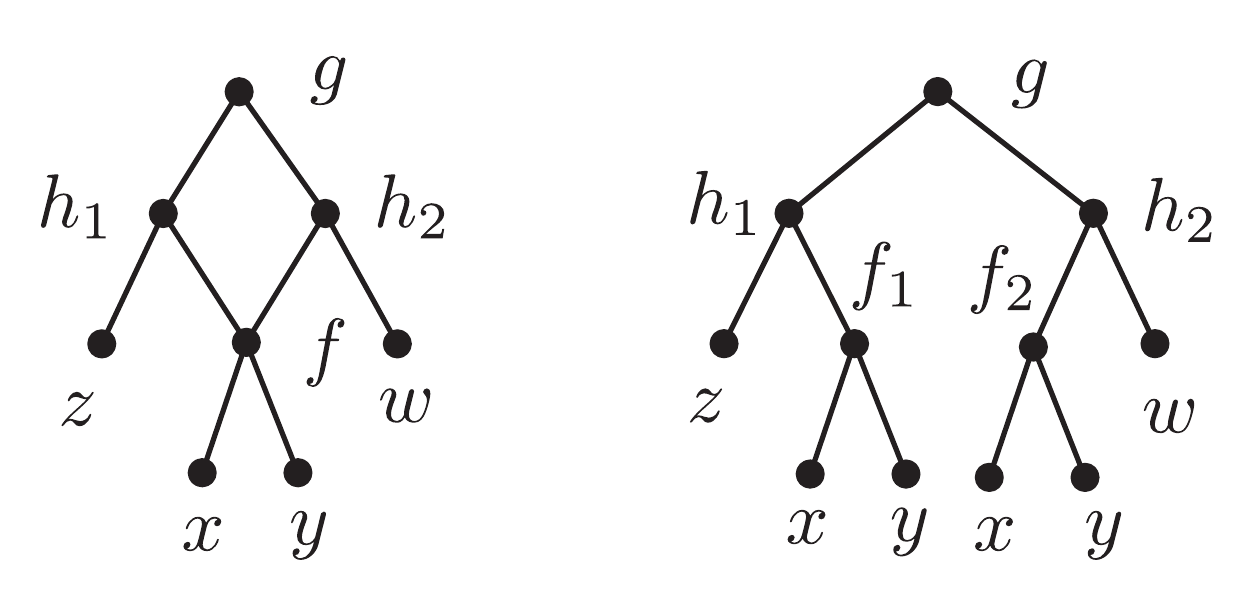}
    \caption{The space of functions computed by the neural network on the left is strictly smaller than that of its TENN on the right; see Example \ref{network2}.}
    \label{fig:basic-tree}
\end{figure}
\end{example}

\begin{example}\label{network2'}
Here we briefly argue that Conjecture \ref{conjecture} holds for the network in Figure \ref{fig:basic-tree} (which has two hidden layers). The goal is to obtain PDEs that, given suitable non-vacuous non-vanishing conditions, characterize smooth functions $F(x,y,z,w)$ of the form \eqref{auxiliary13}. We seek a description of the form $g(F_1(x,y,z),F_2(x,y,w))$ of $F(x,y,z,w)$ where the trivariate functions $F_1(x,y,z)$ and $F_2(x,y,w)$ are superpositions $h_1(f(x,y),z)$ and $h_2(f(x,y),w)$ with the same bivariate function $f$ appearing in both of them. Invoking the logic that has been used repeatedly in \S\ref{toy examples-trees}; $\nabla F$ must be a linear combination of $\nabla F_1$ and $\nabla F_2$. Following Example \ref{basic-1}, the only restriction on the latter two gradients is 
\small
$$
\nabla F_1\,\mathlarger{\parallel}
\begin{bmatrix}
1\\
\frac{(F_1)_y}{(F_1)_x}=\frac{f_y}{f_x}\\[0.3em]
\alpha(x,y,z):=\frac{(F_1)_z}{(F_1)_x}\\
0
\end{bmatrix},
\quad 
\nabla F_2\,\mathlarger{\parallel}
\begin{bmatrix}
1\\
\frac{(F_2)_y}{(F_2)_x}=\frac{f_y}{f_x}\\
0\\
\beta(x,y,w):=\frac{(F_2)_w}{(F_2)_x}
\end{bmatrix};
$$
\normalsize
and as observed in Example \ref{network2}, the ratio $\frac{f_y}{f_x}$ coincides with $\frac{F_y}{F_x}$. Thus, the existence of a representation of the form \eqref{auxiliary13} is equivalent 
to the existence of a linear relation such as  
\small
$$
\begin{bmatrix}
F_x\\
F_y\\
F_z\\
F_w
\end{bmatrix}
=
\frac{F_z}{\alpha}
\begin{bmatrix}
1\\
\frac{F_y}{F_x}\\
\alpha(x,y,z)\\
0
\end{bmatrix}
+
\frac{F_w}{\beta}
\begin{bmatrix}
1\\
\frac{F_y}{F_x}\\
0\\
\beta(x,y,w)
\end{bmatrix}.
$$
\normalsize
This amounts to the equation 
\small
$$F_z\left(\frac{1}{\alpha}\right)+F_w\left(\frac{1}{\beta}\right)=F_x.$$
\normalsize
Now the idea of Examples \ref{basic-1'} and \ref{toy example 1} applies: As $\left(\frac{1}{\alpha}\right)_w=0$ and $\left(\frac{1}{\beta}\right)_z=0$, applying the operators $\partial_z$, $\partial_w$ and $\partial_{zw}$ to the last equation results in a linear system with four equations and four unknowns $\frac{1}{\alpha}$, $\frac{1}{\beta}$,  $\left(\frac{1}{\alpha}\right)_z$ and $\left(\frac{1}{\beta}\right)_w$. If non-singular (a non-vanishing condition), the system may be solved to obtain expressions purely in terms of partial derivatives of $F$ for the aforementioned unknowns. Now $\left(\frac{1}{\alpha}\right)_w=0$ and $\left(\frac{1}{\beta}\right)_z=0$ along with the equations $F_{xz}F_{y}=F_{yz}F_{x}$, $F_{xw}F_{y}=F_{yw}F_{x}$ from Example \ref{network2} yield four algebraic PDEs characterizing superpositions \eqref{auxiliary13}. 
\end{example}

The final example of this section finishes Example \ref{basic-2} from the introduction. 
\begin{example}\label{network1}
We go back to Example \ref{basic-2} to study PDEs and PDIs satisfied by functions of the form \eqref{2var_form}. Absorbing $a'',b''$ into inner functions, we can focus on the simpler form
\begin{equation}\label{auxiliary5}
F(x,t)=\sigma(f(ax+bt)+g(a'x+b't)). 
\end{equation}
Let us for the time being  forget about the outer activation function $\sigma$: Consider functions such as 
$$G(x,t)=f(ax+bt)+g(a'x+b't).$$
Smooth functions of this form constitute solutions of a second order linear homogeneous PDE with constant coefficients 
\begin{equation}\label{auxiliary6}
UG_{xx}+VG_{xt}+WG_{tt}=0,
\end{equation}
where $(a,b)$ and $(a',b')$ satisfy 
\begin{equation}\label{P1}
UA^2+VAB+WB^2=0.
\end{equation}
The reason is that when $(a,b)$ and $(a',b')$ satisfy \eqref{P1}, the differential operator
$
U\partial_{xx}+V\partial_{xt}+W\partial_{tt}
$
can be factorized as 
$$
(b\partial_x-a\partial_t)(b'\partial_x-a'\partial_t)
$$
to a composition of operators that annihilate the linear forms $ax+bt$ and $a'x+b't$. If $(a,b)$ and $(a',b')$ are not multiples of each other,  then they constitute a new coordinate system $(ax+bt, a'x+b't)$ in which the mixed partial derivatives of $F$ all vanish; so, at least locally, $F$ must be a sum of univariate functions of $ax+bt$ and $a'x+b't$.\footnote{Compare with the proof of Lemma \ref{split} in Appendix \ref{Proofs}.} We conclude that assuming $V^2-4UW>0$, functions of the form $G(x,t)=f(ax+bt)+g(a'x+b't)$ may be identified with solutions of PDEs of the form  \eqref{auxiliary6}. As in Example \ref{basic-1}, we desire algebraic PDEs purely in terms of $F$ and without constants $U,V$ and $W$. One way to do so is to differentiate \eqref{auxiliary6} further, for instance: 
\begin{equation}\label{auxiliary6'}
UG_{xxx}+VG_{xxt}+WG_{xtt}=0.
\end{equation}
Notice that \eqref{auxiliary6} and \eqref{auxiliary6'} could be interpreted as $(U,V,W)$ being perpendicular to $(G_{xx},G_{xt},G_{tt})$ and $(G_{xxx},G_{xxt},G_{xtt})$. Thus, the cross product 
$$
(G_{xt}G_{xtt}-G_{tt}G_{xxt},G_{tt}G_{xxx}-G_{xx}G_{xtt},G_{xx}G_{xxt}-G_{xt}G_{xxx})
$$
of the latter two vectors must be parallel to a constant vector. Under the non-vanishing condition that one of the entries of the cross product, say the last one, is non-zero, the constancy may be thought of as ratios of the other two components to the last one being constants. The result is a characterization (in the vein of Conjecture \ref{conjecture}) of functions $G$ of the form   $f(ax+bt)+g(a'x+b't)$ which are subjected to 
$G_{xx}G_{xxt}-G_{xt}G_{xxx}\neq 0$ as
\small
\begin{equation}\label{characterization}
\begin{split}
&\frac{G_{xt}G_{xtt}-G_{tt}G_{xxt}}{G_{xx}G_{xxt}-G_{xt}G_{xxx}} \text{ and } \frac{G_{tt}G_{xxx}-G_{xx}G_{xtt}}{G_{xx}G_{xxt}-G_{xt}G_{xxx}} \text{ are constants,}\\
& (G_{tt}G_{xxx}-G_{xx}G_{xtt})^2>4(G_{xt}G_{xtt}-G_{tt}G_{xxt})(G_{xx}G_{xxt}-G_{xt}G_{xxx}).
\end{split}
\end{equation}
\normalsize
Notice that the PDI is not redundant here: For a solution $G=G(x,t)$ of Laplace's equation the fractions from the first line of \eqref{characterization} are constants while on the second line, the left-hand side of the inequality is zero but its right-hand side is $4(G_{xt}G_{xtt}-G_{tt}G_{xxt})^2\geq 0$.\\
\indent
Composing $G$ with $\sigma$ makes the derivation of PDEs and PDIs imposed on functions of the form \eqref{auxiliary5} even more cumbersome. We only provide a sketch. Under the assumption that the gradient of 
$F=\sigma\circ G$ is non-zero, the univariate function $\sigma$ admits a local inverse $\tau$. Applying the chain rule to  $G=\tau\circ F$ yields: 
$$
G_{xx}=\tau''(F)(F_x)^2+\tau'(F)F_{xx},\quad  G_{xt}= \tau''(F)F_xF_t+\tau'(F)F_{xt},\quad G_{tt}=\tau''(F)(F_t)^2+\tau'(F)F_{tt}.
$$
Plugging them in the PDE \eqref{auxiliary6} that $G$ satisfies results in:
$$
\tau''(F)\left(U(F_x)^2+VF_xF_t+W(F_t)^2\right)+\tau'(F)\left(UF_{xx}+VF_{xt}+WF_{tt}\right)=0;
$$
or equivalently:
\small
\begin{equation}\label{complicated}
\frac{UF_{xx}+VF_{xt}+WF_{tt}}{U(F_x)^2+VF_xF_t+W(F_t)^2}=-\frac{\tau''(F)}{\tau'(F)}=-\left(\frac{\tau''}{\tau'}\right)(F).    
\end{equation}
\normalsize
It suffices for the ratio $\frac{UF_{xx}+VF_{xt}+WF_{tt}}{U(F_x)^2+VF_xF_t+W(F_t)^2}$ to be a function of $F$ such as $\nu(F)$ since then $\tau$ may be recovered as  $\tau=\bigintsss{\rm{e}}^{-\int\nu}$.
Following the discussion in the beginning of \S\ref{toy examples}, this is equivalent to 
\small
$$
\nabla\left(\frac{UF_{xx}+VF_{xt}+WF_{tt}}{U(F_x)^2+VF_xF_t+W(F_t)^2}\right)\mathlarger{\parallel}\,\nabla F.
$$
\normalsize
This amounts to an identity of the form 
$$
\Phi_1(F)U^2+\Phi_2(F)V^2+\Phi_3(F)W^2+\Phi_4(F)UV+\Phi_5(F)VW+\Phi_6(F)UW=0.
$$
where $\Phi_i(F)$'s are complicated non-constant polynomial expressions of partial derivatives of $F$. In the same way that the parameters  $U,V$ and $W$ in PDE $\eqref{auxiliary6}$ were eliminated to arrive at \eqref{characterization}, one may solve the homogeneous linear system consisting of the identity above and its derivatives in order to derive a six-dimensional vector 
\begin{equation}\label{vector1}
\big(\Xi_1(F),\Xi_2(F),\Xi_3(F),\Xi_4(F),\Xi_5(F),\Xi_6(F)\big)
\end{equation}
of rational expressions of partial derivatives of $F$ parallel to the constant vector
\begin{equation}\label{vector2}  
(U^2,V^2,W^2,UV,VW,UW).
\end{equation}
The parallelism amounts to a number of PDEs, e.g. $\Xi_1(F)\,\Xi_2(F)=\Xi_4(F)^2$ and the ratios $\frac{\Xi_i(F)}{\Xi_j(F)}$ must be constant because they coincide with the ratios of components of \eqref{vector2}. Moreover, $V^2-4UW>0$ implies $\left(\frac{V^2}{UW}-4\right)\frac{V^2}{UW}\geq 0$. Replacing with the corresponding ratios of components of \eqref{vector1}, we obtain the PDI  
$$
\big(\Xi_2(F)-4\,\Xi_6(F)\big)\,\Xi_2(F)\,\Xi_6(F)\geq 0,
$$
which must be satisfied by any function of the form \eqref{auxiliary5}.
\end{example}

\section{General results}\label{generalization}
Building on the examples of the previous section, we establish Conjecture \ref{conjecture} for a number of cases. 

\subsection{Characterizing tree functions with distinct inputs}\label{generalization-tree1}

The goal of this subsection is to prove Theorems \ref{main-tree} and \ref{main-tree-activation}.

\begin{proof}[Proof of Theorem \ref{main-tree}]

\begin{figure}
    \centering
    \includegraphics[height=4cm]{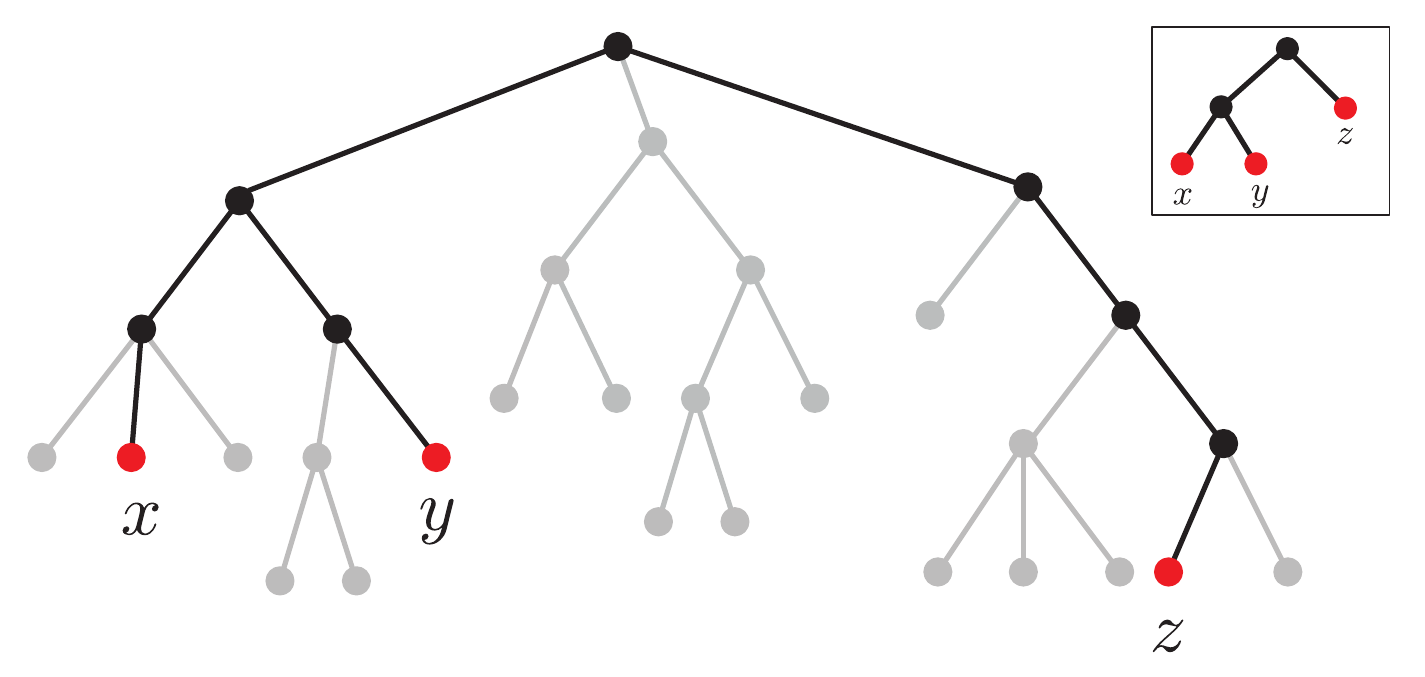}
    \caption{The necessity of constraint \eqref{temp1} in Theorem \ref{main-tree} follows from the case of trivariate tree functions discussed in Example \ref{basic-1}: Choosing three of the variables (red leaves) and fixing the rest (gray leaves) results in a superposition of the form $g(f(x,y),z)$  that must obey constraint \eqref{auxiliary-basic-constraint'}.}
    \label{fig:necessity}
\end{figure}

The necessity of constraint \eqref{temp1} follows from Example \ref{basic-1}: As demonstrated in Figure \ref{fig:necessity}, picking three of variables $x_i=x$, $x_j=y$ and $x_k=z$ where the former two are separated from the latter by a sub-tree and taking the rest of variables to be constant, we obtain a superposition of the form $F(x,y,z)=g(f(x,y),z)$ studied in Example \ref{basic-1}; it should  satisfy $F_{xz}F_{y}=F_{yz}F_{x}$ or equivalently \eqref{temp1}.\\
\indent
We induct on the number of variables -- which coincides with the number of leaves -- to prove the sufficiency of constraint \eqref{temp1} and the non-vanishing conditions in Theorem \ref{main-tree} for the existence of a local implementation -- in the form of a superposition of functions of lower arity -- on the tree architecture in hand. Consider a rooted tree $\mathcal{T}$ with $n$ leaves which are labeled by the coordinate functions $x_1,\dots,x_n$. The inductive step is illustrated in Figure \ref{fig:general_tree_many_subtrees}: Removing the root results in a number of smaller trees $\mathcal{T}_1,\dots,\mathcal{T}_l$ and a number of single vertices\footnote{A single vertex is not considered to be a rooted tree in our convention.} corresponding to the leaves adjacent to the root of $\mathcal{T}$. By renumbering $x_1,\dots,x_n$ one may write the leaves as 
\begin{equation}\label{list}
x_1,\dots,x_{m_1};\, x_{m_1+1},\dots,x_{m_1+m_2};\,\dots;\,x_{m_1+\dots+m_{l-1}+1},\dots,x_{m_1+\dots+m_l};\,x_{m_1+\dots+m_l+1};\,\dots;\,x_n
\end{equation}
where $x_{m_1+\dots+m_{s-1}+1},\dots,x_{m_1+\dots+m_{s-1}+m_s}$ $(1\leq s\leq l)$ are the leaves of the sub-tree $\mathcal{T}_s$ while $x_{m_1+\dots+m_l+1}$  through $x_n$ are the leaves adjacent to the root of $\mathcal{T}$. The goal is to write $F(x_1,\dots,x_n)$ as 
\begin{equation}\label{auxiliary9'}
g(G_1(x_1,\dots,x_{m_1}),\dots,G_l(x_{m_1+\dots+m_{l-1}+1},\dots,x_{m_1+\dots+m_l}),x_{m_1+\dots+m_l+1},\dots,x_n)
\end{equation}
where each smooth function 
$$G_s(x_{m_1+\dots+m_{s-1}+1},\dots,x_{m_1+\dots+m_{s-1}+m_s})$$ 
satisfies the constraints coming from $\mathcal{T}_s$ and thus, by invoking the induction hypothesis, is computable by the tree $\mathcal{T}_s$. Following the discussion before Theorem  \ref{integrability}, it suffices to express $\nabla F$ as a linear combination of the gradients 
$\nabla G_1,\dots,\nabla G_l,\nabla x_{m_1+\dots+m_l+1},\dots,\nabla x_n$. The non-vanishing conditions in Theorem \ref{main-tree} require the first order partial derivative with respect to at least one of the leaves of each $\mathcal{T}_s$ to be non-zero; we may assume $F_{x_{m_1+\dots+m_{s-1}+1}}\neq 0$ without any loss of generality. We should have:   
\small
\begin{equation*}
\begin{split}
\nabla F&=\left[F_{x_1}\,\,\dots\,\, F_{x_{m_1}}\,\,\dots\,\,F_{x_{m_1+\dots+m_{l-1}+1}}\,\,\dots\,\, F_{x_{m_1+\dots+m_l}}\,\,F_{x_{m_1+\dots+m_l+1}}\,\,\dots\,\,F_{x_n}\right]^{\rm{T}}\\
&=\mathlarger{\sum}_{s=1}^l F_{x_{m_1+\dots+m_{s-1}+1}}\left[\overbrace{0\,\cdots\, 0}^{m_1+\dots+m_{s-1}}\,\,1\,\,\frac{F_{x_{m_1+\dots+m_{s-1}+2}}}{F_{x_{m_1+\dots+m_{s-1}+1}}}\,\,\cdots \,\, \frac{F_{x_{m_1+\dots+m_{s-1}+m_s}}}{F_{x_{m_1+\dots+m_{s-1}+1}}}\,\,\overbrace{0\,\cdots\, 0}^{n-(m_1+\dots+m_s)}\right]^{\rm{T}}\\
&+F_{x_{m_1+\dots+m_l+1}}\frac{\partial}{\partial x_{m_1+\dots+m_l+1}}+\dots+F_{x_n}\frac{\partial}{\partial x_n}\\
&\in{\text{Span}}\left\{\nabla G_1(x_1,\dots,x_{m_1}),\dots,\nabla G_l(x_{m_1+\dots+m_{l-1}+1},\dots,x_{m_1+\dots+m_l}),\nabla x_{m_1+\dots+m_l+1},\dots,\nabla x_n\right\}.
\end{split}
\end{equation*}
\normalsize
In expressions above, the vector 
$\left[1\,\,\frac{F_{x_{m_1+\dots+m_{s-1}+2}}}{F_{x_{m_1+\dots+m_{s-1}+1}}}\,\,\cdots \,\, \frac{F_{x_{m_1+\dots+m_{s-1}+m_s}}}{F_{x_{m_1+\dots+m_{s-1}+1}}}\right]^{\rm{T}}$
(which is of size $m_s$) is dependent only on the variables $x_{m_1+\dots+m_{s-1}+1},\dots,x_{m_1+\dots+m_s}$ which are the leaves of $\mathcal{T}_s$: Any other leaf $x_k$ is separated from them by the sub-tree $\mathcal{T}_s$ of $\mathcal{T}$ and hence for any leaf $x_i$ with $m_1+\dots+m_{s-1}<i\leq m_1+\dots+m_s$ we have  $\left(\frac{F_{x_i}}{F_{x_{m_1+\dots+m_{s-1}+1}}}\right)_{x_k}=0$
due to the simplified form \eqref{constraint-simplified-1} of \eqref{temp1}. To finish the proof, one should  establish the existence of functions $G_s(x_{m_1+\dots+m_{s-1}+1},\dots,x_{m_1+\dots+m_s})$ appearing in \eqref{auxiliary9'}; that is,
$\left[1\,\,\frac{F_{x_{m_1+\dots+m_{s-1}+2}}}{F_{x_{m_1+\dots+m_{s-1}+1}}}\,\,\cdots \,\, \frac{F_{x_{m_1+\dots+m_{s-1}+m_s}}}{F_{x_{m_1+\dots+m_{s-1}+1}}}\right]^{\rm{T}}$
should be shown to be parallel to a gradient vector field $\nabla G_s$. Notice that the induction hypothesis would be applicable to $G_s$ since any ratio $\frac{(G_s)_{x_i}}{(G_s)_{x_j}}$ of partial derivatives is the same as the corresponding ratio of partial derivatives of $F$.  
Invoking Theorem \ref{integrability}, to prove the existence of $G_s$ we should verify that the $1$-form 
\small
$$
\omega_s:=\mathlarger{\sum}_{i=m_1+\dots+m_{s-1}+1}^{m_1+\dots+m_{s-1}+m_s}\frac{F_{x_i}}{F_{x_{m_1+\dots+m_{s-1}+1}}}\,{\rm{d}}x_i\quad (1\leq s\leq l)
$$
\normalsize
satisfies $\omega_s\wedge{\rm{d}\omega_s}=0$. We finish the proof by showing this in the case of $s=1$; other cases are completely similar. We have: 
\small
\begin{equation*}
\begin{split}
\omega_1\wedge{\rm{d}}\omega_1&=\left(\sum_{i=1}^{m_1}\frac{F_{x_i}}{F_{x_1}}\,{\rm{d}}x_i\right)\wedge \left(\sum_{j=1}^{m_1}{\rm{d}}\left(\frac{F_{x_j}}{F_{x_1}}\right)\wedge{\rm{d}}x_j\right)
=\left(\sum_{i=1}^{m_1}\frac{F_{x_i}}{F_{x_1}}\,{\rm{d}}x_i\right)\wedge \left(\sum_{j=1}^{m_1}\left(\sum_{k=1}^{m_1}\left(\frac{F_{x_j}}{F_{x_1}}\right)_{x_k}{\rm{d}}x_k\right)\wedge{\rm{d}}x_j\right)\\
&=\sum_{i,j,k\in\{1,\dots,m_1\}}\left[\frac{F_{x_i}F_{x_jx_k}}{(F_{x_1})^2}-\frac{F_{x_i}F_{x_j}F_{x_1x_k}}{(F_{x_1})^3}\right]{\rm{d}}x_i\wedge{\rm{d}}x_k\wedge{\rm{d}}x_j\\
&=\left(\sum_{i=1}^{m_1}\frac{F_{x_i}}{(F_{x_1})^2}\,{\rm{d}}x_i\right)\wedge\left(\sum_{j,k\in\{1,\dots,m_1\}}F_{x_jx_k}\,{\rm{d}}x_k\wedge{\rm{d}}x_j\right)
+\left(\sum_{i,j\in\{1,\dots,m_1\}}F_{x_i}F_{x_j}\,{\rm{d}}x_i\wedge{\rm{d}}x_j\right)\wedge\left(\sum_{k=1}^{m_1}\frac{F_{x_1x_k}}{(F_{x_1})^3}\,{\rm{d}}x_k\right).
\end{split}
\end{equation*}
\normalsize
The last line is zero because, in the parentheses, the $2$-forms  
\small
$$\sum_{j,k\in\{1,\dots,m_1\}}F_{x_jx_k}\,{\rm{d}}x_j\wedge{\rm{d}}x_k, \quad  \sum_{i,j\in\{1,\dots,m_1\}}F_{x_i}F_{x_j}\,{\rm{d}}x_i\wedge{\rm{d}}x_j,$$
\normalsize
are zero since interchanging $j$ and $k$, or $i$ and $j$ in the summations results in the opposite of the original differential form.

\begin{figure}
    \centering
    \includegraphics[height=7cm, width=10cm]{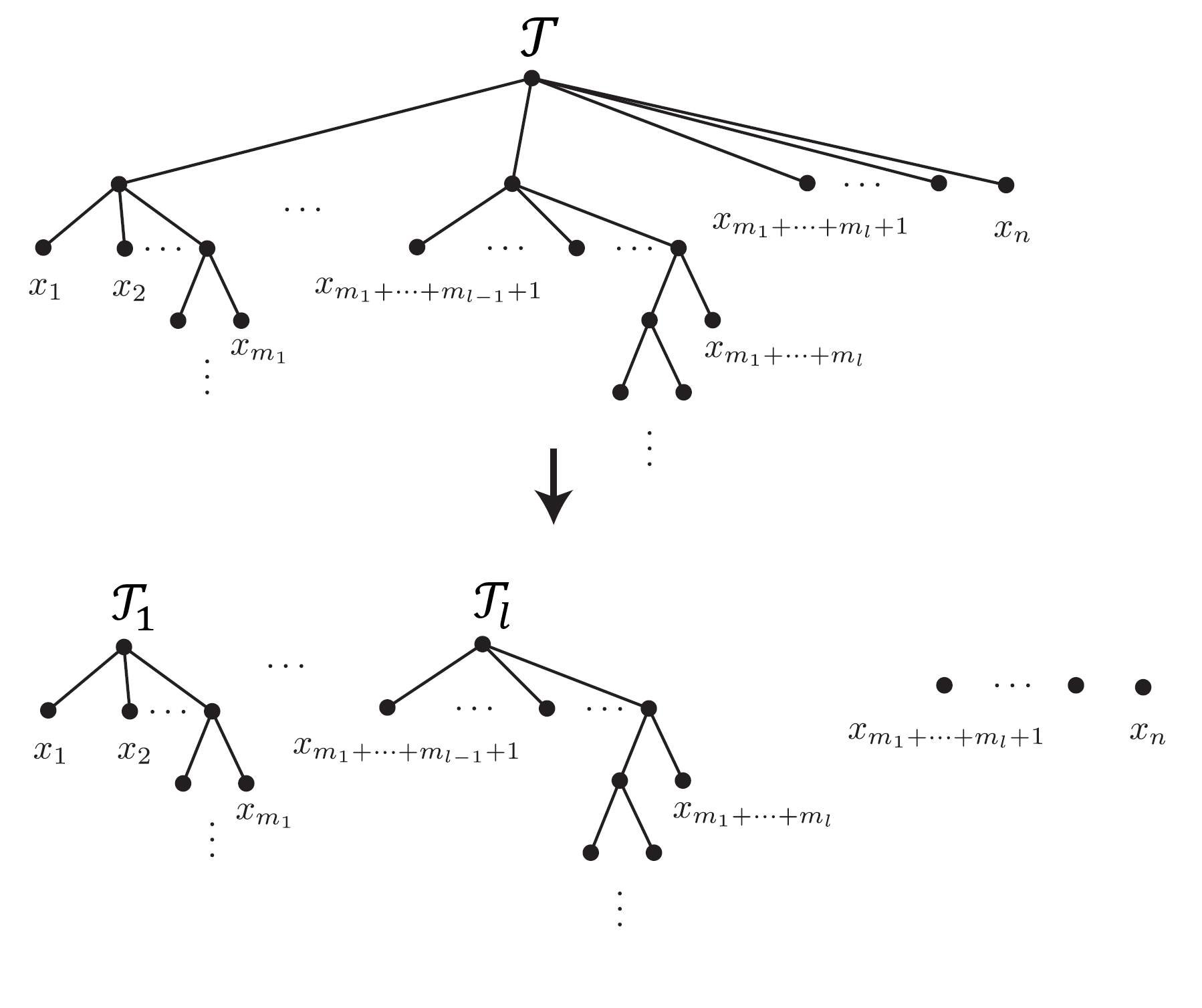}
    \caption{The inductive step in the proof of Theorem \ref{main-tree}: The removal of the root of $\mathcal{T}$ results in a number of smaller rooted trees along with single vertices which were the leaves adjacent to the root of $\mathcal{T}$ (if any).}
    \label{fig:general_tree_many_subtrees}
\end{figure}

\end{proof}

\begin{remark}
The formulation of Theorem \ref{main-tree} in \cite{Farhoodi2019OnFC} is concerned with analytic functions and binary trees. The proof presented above follows the same inductive procedure  but utilizes Theorem \ref{integrability} instead of Taylor expansions. Of course, Theorem \ref{integrability} remains valid in the analytic category; so the tree representation of $F$ constructed in the proof here consists of analytic functions if $F$ is analytic. An advantage of working with analytic functions is that in certain cases the non-vanishing conditions may be relaxed. For instance, if in Example \ref{basic-1}  the  function $F(x,y,z)$ satisfying \eqref{auxiliary-basic-constraint'} is analytic, it admits a local representation of the form \eqref{3var_form} while if $F$ is only smooth, at least one of the conditions $F_x\neq 0$ of $F_y\neq 0$ is required. 
See \cite[\S\S 5.1,5.3]{Farhoodi2019OnFC} for the details.
\end{remark}

\begin{proof}[Proof of Theorem \ref{main-tree-activation}]
Establishing the necessity of constraints \eqref{temp2} and \eqref{temp3} is straightforward. An implementation of a smooth function $F=F(x_1,\dots,x_n)$ on the tree $\mathcal{T}$
is in a form such as 
\small
\begin{equation}\label{huge5}
\begin{split}
&\sigma\Bigg(\cdots\Bigg(\tilde{w}.\tilde{\sigma}\bigg(
w_1.\tau_1\Big(\cdots\Big(\tilde{w}_1.\tilde{\tau}_1(cx_i+\cdots)+\dbtilde{w}_1.\dbtilde{\tau}_1(c'x_{i'}+\cdots)+\cdots\Big)\cdots\Big)\\
&+w_2.\tau_2\Big(\cdots\Big(\tilde{w}_2.\tilde{\tau}_2(dx_j+\cdots)+\dbtilde{w}_2.\dbtilde{\tau}_2(d'x_{j'}+\cdots)+\cdots\Big)\cdots\Big)
+w_3.\tau_3\Big(\cdots\Big)+\cdots\bigg)
+\cdots\Bigg)\cdots\Bigg)
\end{split}
\end{equation}
\normalsize
for appropriate activation functions and weights. In the expression above, variables $x_s$ appearing in
\footnotesize
$$
\tilde{\sigma}\bigg(
w_1.\tau_1\Big(\cdots\Big(\tilde{w}_1.\tilde{\tau}_1(cx_i+\cdots)+\dbtilde{w}_1.\dbtilde{\tau}_1(c'x_{i'}+\cdots)+\cdots\Big)\cdots\Big)
+w_2.\tau_2\Big(\cdots\Big(\tilde{w}_2.\tilde{\tau}_2(dx_j+\cdots)+\dbtilde{w}_2.\dbtilde{\tau}_2(d'x_{j'}+\cdots)+\cdots\Big)\cdots\Big)
+w_3.\tau_3\Big(\cdots\Big)+\cdots\bigg)
$$
\normalsize
are the leaves of the smallest (full) sub-tree of $\mathcal{T}$ in which both $x_i$ and $x_j$ appear as leaves. Denoting this sub-tree by $\widetilde{\mathcal{T}}$, the activation function applied at the root of $\widetilde{\mathcal{T}}$ is $\tilde{\sigma}$, and the sub-trees emanating from the root of $\widetilde{\mathcal{T}}$ -- which we write as $\widetilde{\mathcal{T}_1},\widetilde{\mathcal{T}_2},\widetilde{\mathcal{T}_3},\dots$ -- have 
$\tau_1,\tau_2,\tau_3,\dots$ assigned to their roots. Here, $\widetilde{\mathcal{T}_1}$ and $\widetilde{\mathcal{T}_2}$ contain $x_i$ and $x_j$ respectively, and are the largest (full) sub-trees that have exactly one of $x_i$ and $x_j$. To verify \eqref{temp2}, notice that $\frac{F_{x_i}}{F_{x_j}}$ is proportional to 
\small
\begin{equation}\label{auxiliary14}
\frac{\tau'_1\Big(\cdots\Big(\tilde{w}_1.\tilde{\tau}_1(cx_i+\cdots)+\dbtilde{w}_1.\dbtilde{\tau}_1(c'x_{i'}+\cdots)+\cdots\Big)\cdots\Big)\dots\tilde{\tau}'_1(cx_i+\cdots)}
{\tau'_2\Big(\cdots\Big(\tilde{w}_2.\tilde{\tau}_2(dx_j+\cdots)+\dbtilde{w}_2.\dbtilde{\tau}_2(d'x_{j'}+\cdots)+\cdots\Big)\cdots\Big)\dots\tilde{\tau}'_2(dx_j+\cdots)}
\end{equation}
\normalsize
with the constant of proportionality being a quotient of two products of certain weights of the network. 
The ratio \eqref{auxiliary14} is dependent only on those variables that appear as leaves of $\widetilde{\mathcal{T}_1}$ and $\widetilde{\mathcal{T}_2}$; so 
\small
$$
\left(\frac{F_{x_i}}{F_{x_j}}\right)_{x_k}=0\Leftrightarrow F_{x_ix_k}F_{x_j}=F_{x_jx_k}F_{x_i}
$$
\normalsize
unless there is a sub-tree of $\mathcal{T}$ containing the leaf $x_k$ and exactly one of $x_i$ or $x_j$ (which forcibly will be a sub-tree of $\widetilde{\mathcal{T}_1}$ or $\widetilde{\mathcal{T}_2}$). Before switching to constraint \eqref{temp3}, we point out that the description \eqref{huge5} of $F$ assumes that the leaves $x_i$ and $x_j$ are not siblings. If they are, $F$ may be written as 
$$
\sigma\Big(\cdots\Big(\tilde{w}.\tilde{\sigma}\big(w.\tau(cx_i+dx_j+\cdots)+\cdots\big)+\cdots\Big)\cdots\Big)
$$
in which case $\frac{F_{x_i}}{F_{x_j}}=\frac{c}{d}$ is a constant and hence \eqref{constraint-simplified-1} holds for all $1\leq k\leq n$.
To finish the proof of necessity of the constraints introduced in Theorem \ref{main-tree-activation}, consider the fraction \eqref{auxiliary14} which is a multiple of $\frac{F_{x_i}}{F_{x_j}}$. This has a description as a product of a function of $x_i,x_{i'},\dots$ (leaves of $\widetilde{\mathcal{T}_1}$) by a function of $x_j,x_{j'},\dots$ (leaves of $\widetilde{\mathcal{T}_2}$). Lemma \ref{split} now implies that for any leaf $x_{i'}$ of $\widetilde{\mathcal{T}_1}$ and any leaf $x_{j'}$ of $\widetilde{\mathcal{T}_2}$:
$$
\left(\frac{\left(\frac{F_{x_i}}{F_{x_j}}\right)_{x_{i'}}}{\frac{F_{x_i}}{F_{x_j}}}\right)_{x_{j'}}=0;
$$
hence the simplified form \eqref{constraint-simplified-2} of \eqref{temp3}.\\
\indent
We induct on the number of leaves to prove the sufficiency of constraints \eqref{temp2} and \eqref{temp3} (accompanied by suitable non-vanishing conditions) for the existence of a tree implementation of a smooth function $F=F(x_1,\dots,x_n)$ as a composition of functions of the form \eqref{activation}. Given a rooted tree $\mathcal{T}$ with $n$ leaves labeled by $x_1,\dots,x_n$, the inductive step has two cases demonstrated in Figures \ref{fig:inductive1} and \ref{fig:inductive2}:
\begin{itemize}
\item There are leaves, say $x_{m+1},\dots,x_n$, directly adjacent to the root of $\mathcal{T}$; their removal results in a smaller tree $\mathcal{T'}$ with leaves $x_1,\dots,x_m$ ; see Figure \ref{fig:inductive1}. The goal is to write $F(x_1,\dots,x_n)$ as 
\begin{equation}\label{auxiliary9}
\sigma(G(x_1,\dots,x_m)+c_{m+1}x_{m+1}+\dots+c_nx_n)    
\end{equation}
with $G$ satisfying appropriate constraints that, invoking the induction hypothesis, guarantee that $G$ is computable by $\mathcal{T'}$. 
\item There is no leaf adjacent to the root of $\mathcal{T}$, but there are smaller sub-trees. Denote one of them with $\mathcal{T}_2$ and show its leaves by $x_{m+1},\dots,x_{n}$. Removing this sub-tree results in a smaller tree $\mathcal{T}_1$ with leaves $x_1,\dots,x_m$; see Figure \ref{fig:inductive2}. The goal is to write $F(x_1,\dots,x_n)$ as
\begin{equation}\label{auxiliary10}
\sigma(G_1(x_1,\dots,x_m)+G_2(x_{m+1},\dots,x_n))    
\end{equation}
with $G_1$ and $G_2$ satisfying constraints corresponding to $\mathcal{T}_1$ and $\mathcal{T}_2$,  and hence may be implemented on these trees by invoking the induction hypothesis.
\end{itemize}
Following the discussion in the beginning of \S\ref{toy examples},  $F$ may be locally written as a function of another function with non-zero gradient if the gradients are parallel. 
This idea has been frequently used so far, but there is a twist here: We want such a description of $F$ to persist on the box-like region $B$ which is the domain of $F$. 
Lemma \ref{technical} resolves this issue. The tree function in the argument of $\sigma$ in either \eqref{auxiliary9} or \eqref{auxiliary10} -- which here we denote by $\tilde{F}$  -- shall be constructed below by invoking the induction hypothesis, so $\tilde{F}$ is defined at every point of $B$. Besides, our description of $\nabla\tilde{F}$ below (cf. \eqref{auxiliary11},\eqref{auxiliary12'}) readily indicates that, just like $F$, it satisfies the non-vanishing conditions of Theorem \ref{main-tree-activation}. Applying  Lemma \ref{technical} to $\tilde{F}$, any level set $\left\{\mathbf{x}\in B\,|\, \tilde{F}(\mathbf{x})=c \right\}$ is connected; and  $\tilde{F}$ can be extended to a coordinate system $(\tilde{F},F_2,\dots,F_n)$ for $B$. Thus, $F$ -- whose partial derivatives with respect to other coordinate functions vanish -- realizes precisely one value on any coordinate hypersurface 
$\left\{\mathbf{x}\in B\,|\, \tilde{F}(\mathbf{x})=c \right\}$. Setting $\sigma(c)$ to be the aforementioned value of $F$ defines a function $\sigma$ with  $F=\sigma(\tilde{F})$. 
After this discussion on the domain of definition of the desired representation of $F$, we proceed with constructing $\tilde{F}=\tilde{F}(x_1,\dots,x_n)$ as either $G(x_1,\dots,x_m)+c_{m+1}x_{m+1}+\dots+c_nx_n$ 
in the case of \eqref{auxiliary9} or as $G(x_1,\dots,x_m)+G_2(x_{m+1},\dots,x_n)$ in the case of \eqref{auxiliary10}.\\
\indent
In the case of \eqref{auxiliary9}, assuming that -- as Theorem \ref{main-tree-activation} requires -- one of the partial derivatives $F_{x_{m+1}},\dots,F_{x_n}$, e.g. $F_{x_n}$, is non-zero, we should have:    
\small 
\begin{equation}\label{auxiliary11}
\begin{split}
\nabla F&=\left[F_{x_1}\,\,\dots \,\, F_{x_m}\,\,F_{x_{m+1}}\,\,\dots\,\, F_{x_{n-1}}\,\,F_{x_n}\right]^{\rm{T}}\mathlarger{\parallel}
\left[\frac{F_{x_1}}{F_{x_n}}\,\,\cdots \,\, \frac{F_{x_m}}{F_{x_n}}\,\,\frac{F_{x_{m+1}}}{F_{x_n}}\,\,\cdots \,\, \frac{F_{x_{n-1}}}{F_{x_n}}\,\,1\right]^{\rm{T}}\\
&=\left[G_{x_1}\,\,\cdots \,\,G_{x_m}\,\,c_{m+1}\,\,\cdots c_{n-1}\,\,1\right]^{\rm{T}}= \nabla(G(x_1,\dots,x_m)+c_{m+1}x_{m+1}+\dots+c_{n-1}x_{n-1}+x_n).
\end{split}
\end{equation}
\normalsize
Here, each ratio $\frac{F_{x_j}}{F_{x_n}}$ where $m<j\leq n$ must be a constant -- which we show by $c_j$ -- due to the simplified form \eqref{constraint-simplified-1} of \eqref{temp2}: The only (full) sub-tree of $\mathcal{T}$ containing either $x_j$ or $x_n$ is the whole tree since these leaves are adjacent to the root of $\mathcal{T}$. On the other hand,  
$\left[\frac{F_{x_1}}{F_{x_n}}\,\,\cdots \,\, \frac{F_{x_m}}{F_{x_n}}\right]^{\rm{T}}$ appearing in \eqref{auxiliary11} is a gradient vector field of the form $\nabla G(x_1,\dots,x_m)$ again as a byproduct of \eqref{temp2} and \eqref{constraint-simplified-1}: each ratio  $\frac{F_{x_i}}{F_{x_n}}$ where $1\leq i\leq m$ is independent of $x_{m+1},\dots,x_n$ by the same reasoning as above; and this vector function of $(x_1,\dots,x_m)$ is integrable because for any $1\leq i,i'\leq m$
\small
$$
\left(\frac{F_{x_i}}{F_{x_n}}\right)_{x_{i'}}=\left(\frac{F_{x_{i'}}}{F_{x_n}}\right)_{x_i}\Leftrightarrow F_{x_ix_n}F_{x_{i'}}=F_{x_{i'}x_n}F_{x_{i}}.
$$
\normalsize
Hence, such a $G(x_1,\dots,x_m)$ exists; and moreover, it satisfies constraints from the inductions hypothesis since any ratio $\frac{G_{x_j}}{G_{x_{j'}}}$ coincides with the corresponding ratio of partial derivatives of $F$, a function which is assumed to satisfy \eqref{constraint-simplified-1} and \eqref{constraint-simplified-2}.\\
\indent
Next, in the second case of the inductive step, let us turn to \eqref{auxiliary10}. The non-vanishing conditions of Theorem \ref{main-tree-activation} require a partial derivative among $F_{x_1},\dots,F_{x_m}$ and also a partial derivative among $F_{x_{m+1}},\dots,F_{x_n}$ to be non-zero. Without any loss of generality, we assume $F_{x_1}\neq 0$ and  $F_{x_n}\neq 0$. We want to apply Lemma \ref{split} 
to split the ratio $\frac{F_{x_1}}{F_{x_n}}\neq 0$ as 
\small
\begin{equation}\label{auxiliary12}
\frac{F_{x_1}}{F_{x_n}}=\beta(x_1,\dots,x_m)\,\frac{1}{\gamma}(x_{m+1},\dots,x_n)=\frac{\beta(x_1,\dots,x_m)}{\gamma(x_{m+1},\dots,x_n)}.
\end{equation}
\normalsize
To do so, it needs to be checked that 
\small
$$\left(\frac{\left(\frac{F_{x_1}}{F_{x_n}}\right)_{x_i}}{\frac{F_{x_1}}{F_{x_n}}}\right)_{x_j}=0$$
\normalsize
for any two indices $1\leq i\leq m$ and $m<j\leq n$. This is the content of \eqref{temp3}, or its simplified form \eqref{constraint-simplified-2}, when $x_i$ belongs to the same maximal sub-tree of $\mathcal{T}$ adjacent to the root that has $x_1$; and holds for other choices of $x_i\in\{x_1,\dots,x_m\}$ too since in that situation, by the  simplified form \eqref{constraint-simplified-1} of \eqref{temp2},  the derivative  $\left(\frac{F_{x_1}}{F_{x_n}}\right)_{x_i}$ must be zero because $x_1$, $x_i$ and $x_n$
belong to different maximal sub-trees of $\mathcal{T}$.  Next, the gradient of $F$ could be written as
\small
\begin{equation*}
\begin{split}
\nabla F&=\left[F_{x_1}\,\, F_{x_2}\,\,\dots \,\, F_{x_m}\,\,F_{x_{m+1}}\,\,\dots\,\, F_{x_{n-1}}\,\,F_{x_n}\right]^{\rm{T}}\mathlarger{\parallel}
\left[\frac{F_{x_1}}{F_{x_n}}\,\, \frac{F_{x_2}}{F_{x_n}}\,\,\cdots \,\, \frac{F_{x_m}}{F_{x_n}}\,\,\frac{F_{x_{m+1}}}{F_{x_n}}\,\,\cdots \,\, \frac{F_{x_{n-1}}}{F_{x_n}}\,\,1\right]^{\rm{T}}\\
&=\left[\frac{F_{x_1}}{F_{x_n}}\,\, \frac{F_{x_2}}{F_{x_1}}.\frac{F_{x_1}}{F_{x_n}}\,\,\cdots \,\, \frac{F_{x_m}}{F_{x_1}}.\frac{F_{x_1}}{F_{x_n}}\,\,\frac{F_{x_{m+1}}}{F_{x_n}}\,\,\cdots \,\, \frac{F_{x_{n-1}}}{F_{x_n}}\,\,1\right]^{\rm{T}}\\
&=\frac{F_{x_1}}{F_{x_n}}\left[1\,\, \frac{F_{x_2}}{F_{x_1}}\,\,\cdots \,\, \frac{F_{x_m}}{F_{x_1}}\,\,\overbrace{0\,\cdots\, 0}^{n-m}\right]^{\rm{T}}
+\left[\overbrace{0\,\cdots\, 0}^{m}\,\,\frac{F_{x_{m+1}}}{F_{x_n}}\,\,\cdots \,\, \frac{F_{x_{n-1}}}{F_{x_n}}\,\,1\right]^{\rm{T}}.
\end{split}
\end{equation*}
\normalsize
Combining with \eqref{auxiliary12}: 
\small
\begin{equation}\label{auxiliary12'}
\begin{split}
\nabla F \,\mathlarger{\parallel}& \left[\beta(x_1,\dots,x_m)\,\, \beta(x_1,\dots,x_m).\frac{F_{x_2}}{F_{x_1}}\,\,\cdots \,\, \beta(x_1,\dots,x_m).\frac{F_{x_m}}{F_{x_1}}\,\,\overbrace{0\,\cdots\, 0}^{n-m}\right]^{\rm{T}}\\
&+\left[\overbrace{0\,\cdots\, 0}^{m}\,\,\gamma(x_{m+1},\dots,x_n).\frac{F_{x_{m+1}}}{F_{x_n}}\,\,\cdots \,\, \gamma(x_{m+1},\dots,x_n).\frac{F_{x_{n-1}}}{F_{x_n}}\,\,\gamma(x_{m+1},\dots,x_n)\right]^{\rm{T}}
\end{split}
\end{equation}
\normalsize
To establish \eqref{auxiliary10}, it suffices to argue that the vectors on the right-hand side are in the form of $\nabla G_1$ and $\nabla G_2$ for suitable functions $G_1(x_1,\dots,x_m)$ and $G_2(x_{m+1}, \dots,x_n)$ -- to which then the induction hypothesis can be applied by the same logic as before. Notice that the first one is dependent only on $x_1,\dots,x_m$ while the second one is dependent only on $x_{m+1},\dots,x_n$ again by \eqref{temp2} and \eqref{constraint-simplified-1}: for any 
$1\leq i\leq m$   and $m<j\leq n$ we have $\left(\frac{F_{x_i}}{F_{x_1}}\right)_{x_j}=0$ (respectively $\left(\frac{F_{x_j}}{F_{x_n}}\right)_{x_i}=0$) since there is 
no sub-tree of $\mathcal{T}$ that has only one of $x_1$ and $x_i$ (resp. only one of $x_n$ and $x_j$) and also $x_j$ (resp. also $x_i$). We finish the proof by verifying the corresponding integrability conditions 
\small
$$
\left(\beta\frac{F_{x_i}}{F_{x_1}}\right)_{x_{i'}}=\left(\beta\frac{F_{x_{i'}}}{F_{x_1}}\right)_{x_i},\quad  \left(\gamma\frac{F_{x_j}}{F_{x_n}}\right)_{x_{j'}}=\left(\gamma\frac{F_{x_{j'}}}{F_{x_n}}\right)_{x_j},
$$
\normalsize
for any $1\leq i,i'\leq m$ and $m<j,j'\leq n$. In view of \eqref{auxiliary12}, one can change $\beta$ and $\gamma$ above to $\frac{F_{x_1}}{F_{x_n}}$ or $\frac{F_{x_n}}{F_{x_1}}$ respectively and write the desired identities as the new ones
\small
$$
\left(\frac{\cancel{F_{x_1}}}{F_{x_n}}\frac{F_{x_i}}{\cancel{F_{x_1}}}\right)_{x_{i'}}=\left(\frac{\cancel{F_{x_1}}}{F_{x_n}}\frac{F_{x_{i'}}}{\cancel{F_{x_1}}}\right)_{x_i},\quad \left(\frac{\cancel{F_{x_n}}}{F_{x_1}}\frac{F_{x_j}}{\cancel{F_{x_n}}}\right)_{x_{j'}}=\left(\frac{\cancel{F_{x_n}}}{F_{x_1}}\frac{F_{x_{j'}}}{\cancel{F_{x_n}}}\right)_{x_j},
$$
\normalsize
which hold due to \eqref{constraint-simplified-1}.
\end{proof}

\begin{remark}
As mentioned in Remark \ref{domain}, working with functions of the form \eqref{activation} in Theorem \ref{main-tree-activation} rather than general smooth functions has the advantage of enabling us to determine a domain on which a superposition representation exists. In contrast, the sufficiency part of Theorem \ref{main-tree} is a local statement since it relies on the Implicit Function Theorem. It is possible to say something non-trivial about the domains when functions are furthermore analytic. This is because the Implicit Function Theorem holds in the analytic category as well (\cite[\S6.1]{MR1894435}) where lower bounds on the domain of validity of the theorem exist in the literature \cite{MR1965992}. 
\end{remark}

\begin{figure}
    \centering
    \includegraphics[height=3cm]{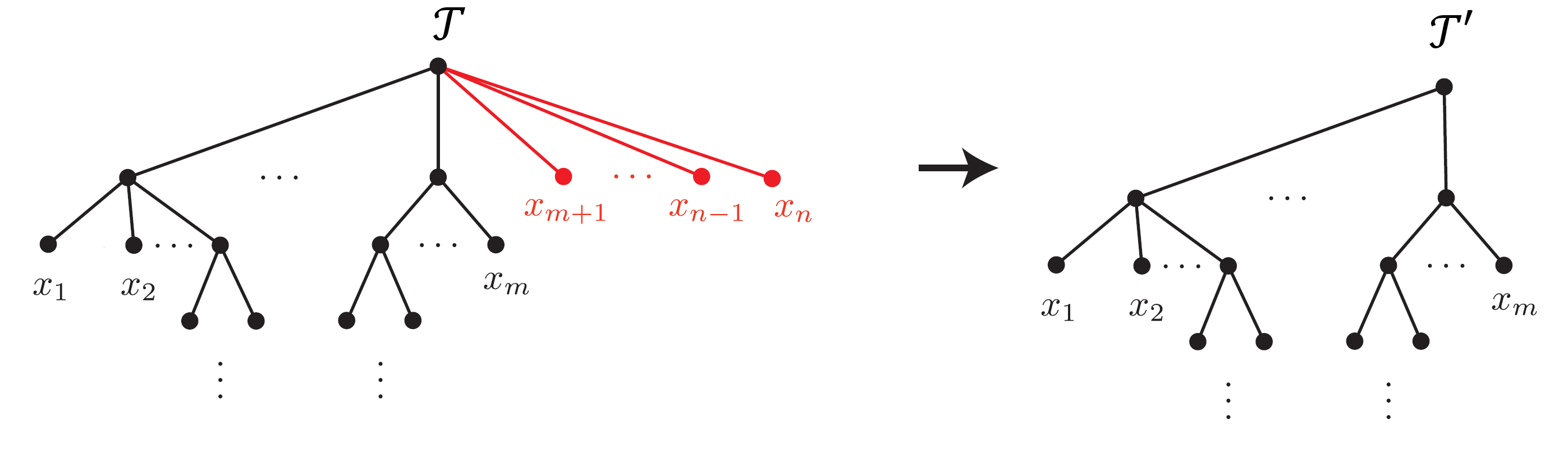}
    \caption{The first case of the inductive step in the proof of Theorem \ref{main-tree-activation}: The removal of the leaves directly connected to the root of $\mathcal{T}$  results in a smaller rooted tree.}
    \label{fig:inductive1}
\end{figure}

\begin{figure}
    \centering
    \includegraphics[height=3cm]{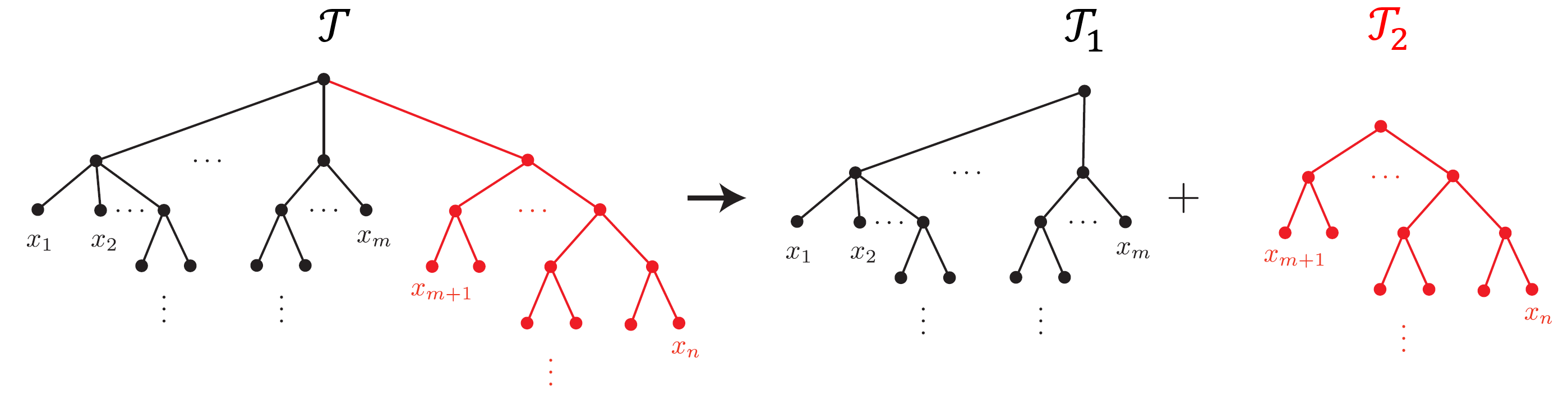}
    \caption{The second case of the inductive step in the proof of Theorem \ref{main-tree-activation}: There is  no leaf directly connected to the root of $\mathcal{T}$. Separating one of the rooted sub-trees adjacent to the root  results in two smaller rooted trees.}
    \label{fig:inductive2}
\end{figure}

\subsection{A family of symmetric tree architectures}\label{generalization-tree2}
\begin{figure}
    \centering
    \includegraphics[height=3cm]{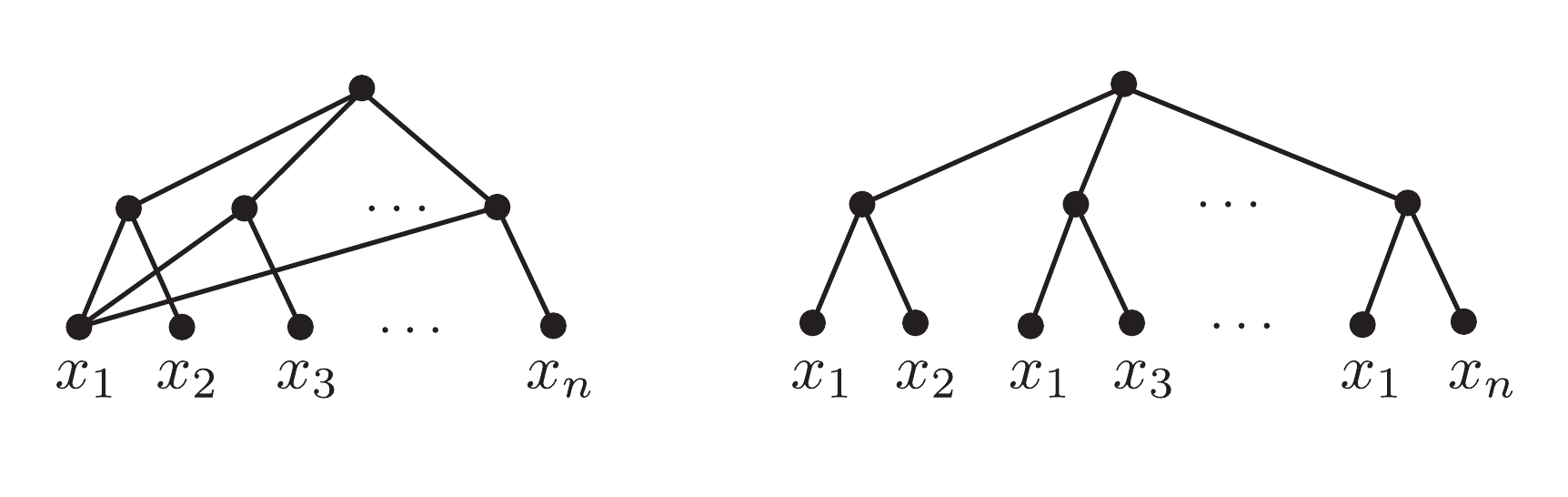}
    \caption{Proposition \ref{form1} describes functions computable by the networks on the left, or equivalently, by the symmetric trees on the right (the corresponding tree expansions).}
    \label{fig:symmetric}
\end{figure}
This subsection, generalizing Examples \ref{basic-1'} and \ref{toy example 1}, studies the family of networks with one hidden layer illustrated in Figure \ref{fig:symmetric} that could be expanded to a family of symmetric trees.  The corresponding superpositions, generalizing \eqref{toy1}, are in the form of  
\begin{equation}\label{superposition1}
F(x_1,x_2,\dots,x_n)=g(f_2(x_1,x_2),\dots,f_n(x_1,x_n)).
\end{equation}
Differentiation with respect to $x_1$ yields:
\begin{equation}\label{auxiliary16}
F_{x_1}=A_2(x_1,x_2)F_{x_2}+\dots+A_n(x_1,x_n)F_{x_n};    
\end{equation}
where $A_i:=\frac{(f_i)_{x_1}}{(f_i)_{x_i}}$.
As before, the idea is to differentiate the identity above enough times with respect to the variables $x_2,\dots,x_n$ to form a balanced or overdetermined  linear system 
with partial derivatives $\left(A_i\right)_{x_i\dots x_i}$ as its unknowns. Such a system, provided suitable non-vanishing conditions hold, can be solved to obtain each $A_i$ in terms of partial derivatives of $F$. The equalities $(A_i)_{x_j}=0$ where $j\neq 1,i$ then result in non-trivial PDE constraints on $F$. There is no canonical way of forming such a system. For proposition below, a balanced system of dimension $2(n-1)$ is considered which is a generalization of \eqref{system}.

\begin{proposition}\label{form1}
Let $F(x_1,\dots,x_n)$ be a smooth function. Consider the matrix equation
\scriptsize
\begin{equation}\label{huge4}
\left[\begin{array}{@{}l|l@{}}
  \begin{matrix}
  F_{x_2}\quad & \hspace{1cm} F_{x_3}\quad & \hspace{2mm}\cdots & \quad\, F_{x_n}
  \end{matrix} &
  \hspace{2.5cm}\bigzero \\ 
\hline
  \begin{matrix}
  F_{x_2x_2} & \hspace{0.9cm} F_{x_2x_3} & \quad\cdots & \quad F_{x_2x_n}\\
  \vdots     & \hspace{0.9cm} \vdots     & \quad\ddots & \quad \vdots    \\
  F_{x_nx_2} & \hspace{0.9cm} F_{x_nx_3} & \quad \cdots & \quad F_{x_nx_n}
  \end{matrix} &
  \begin{matrix}
 \,\,F_{x_2}   &  \quad  0     & \quad\; 0      & \quad \cdots  & \quad 0      \\
  \,\,0        &\quad  F_{x_3} & \quad\;  0     & \quad  \cdots & \quad 0      \\
  \,\,\vdots   &\quad  \vdots  &\quad\; \vdots  & \quad \vdots  & \quad \vdots \\
  \,\,0        &\quad 0        &  \quad\;  0    &  \quad\cdots  & \quad F_{x_n}
  \end{matrix}\\
\hline
  \begin{matrix}
  F_{x_2x_3x_2}\quad & F_{x_2x_3x_3}     & \cdots & F_{x_2x_3x_n}\\
  \vdots  \quad\quad & \vdots            & \ddots & \vdots        \\
  F_{x_{n-1}x_nx_2}  & F_{x_{n-1}x_nx_3} & \cdots & F_{x_{n-1}x_nx_n}
  \end{matrix} &
  \begin{matrix}
  F_{x_2x_3}  &    F_{x_2x_3}  & 0                &  \cdots & 0          \\
  0           & F_{x_3x_4}     &    F_{x_3x_4}    &  \cdots & 0          \\
  \vdots      & \vdots         & \vdots           &  \ddots & \vdots     \\
  0           & 0              &    0             &  \cdots & F_{x_{n-1}x_n}
  \end{matrix}
\end{array}\right]^{-1}
\left[\begin{array}{c}
 F_{x_1}\\
\hline 
 \begin{matrix}
 F_{x_1x_2}\\
 F_{x_1x_3}\\
 \vdots\\
 F_{x_1x_n}
 \end{matrix}\\
\hline
 \begin{matrix}
 F_{x_1x_2x_3}\\
 F_{x_1x_3x_4}\\
 \vdots\\
 F_{x_1x_{n-1}x_n}
 \end{matrix}
\end{array}
\right]
=\frac{1}{\Psi}\left[\begin{array}{c}
 \begin{matrix}
 \Psi_{2}\\
 \Psi_{3}\\
 \vdots\\
 \Psi_{n}
 \end{matrix}\\
\hline
 \begin{matrix}
 *\\
 *\\
 \vdots\\
 *
 \end{matrix}
\end{array}\right]
\end{equation}
\normalsize
where $\Psi$ is the determinant of the matrix whose inverse appears on the left. Thus, $\Psi$ and $\Psi_2,\dots,\Psi_{n}$ are polynomial expressions of partial derivatives of $F$. If $F$ is a smooth superposition of the form \eqref{superposition1}, for any two distinct indices $i,j\in\{2,\dots,n\}$ the following PDE holds
\begin{equation}\label{form1-1}
\Phi_{ij}:=(\Psi_i)_{x_j}\Psi-\Psi_i(\Psi)_{x_j}=0.    
\end{equation}
Conversely, a smooth function $F(x_1,\dots,x_n)$ satisfying the algebraic PDEs \eqref{form1-1} can be locally written in the form of  \eqref{superposition1} if $\Psi\neq 0$. 
\end{proposition}
The proof will be presented in Appendix \ref{Proofs}.

\subsection{A family of asymmetric tree architectures}\label{generalization-tree3} 
\begin{figure}
    \centering
    \includegraphics[height=4cm]{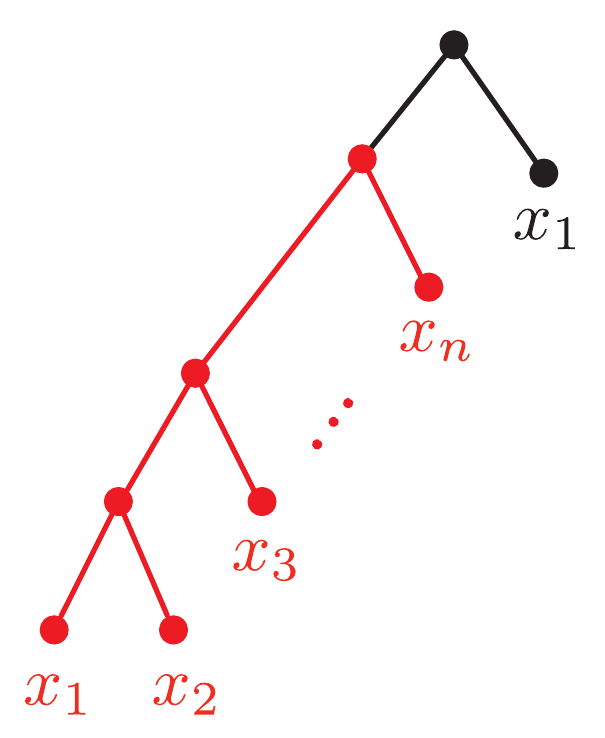}
    \caption{Proposition \ref{form2} describes functions computable by the asymmetric trees illustrated here. These trees are obtained by adding a leaf to a smaller tree (colored in red) whose inputs are distinct.}
    \label{fig:asymmetric}
\end{figure}
This subsection, generalizing Example \ref{toy example 2}, studies the family of asymmetric tree architectures illustrated in Figure \ref{fig:asymmetric} which expand in both depth and width. The corresponding superpositions, generalizing  \eqref{toy2}, are in the form of 
\begin{equation}\label{superposition2}
F(x_1,\dots,x_n)=g(x_1,f_n(...f_3(f_2(x_1,x_2),x_3)...,x_n)),
\end{equation}
or equivalently 
\begin{equation}\label{superposition2'}
F(x_1,\dots,x_n)=g(x_1,G(x_1,\dots,x_n))    
\end{equation}
where $G$ is a function implemented on an asymmetric tree with distinct inputs $x_1,\dots,x_n$ (colored in red in Figure \ref{fig:asymmetric}). Theorem \ref{main-tree} characterizes such functions in terms of PDE constraints of the form \eqref{temp1}.
We present a conceptual way of deriving necessary and sufficient PDE constraints for \eqref{superposition2'}.  
 Writing $g$ as $g(x_1,u)$, for any $2\leq i<j\leq n$ we get:
\small
\begin{equation}\label{auxiliary}
\frac{F_{x_i}}{F_{x_j}}=\frac{g_uG_{x_i}}{g_uG_{x_j}}=\frac{G_{x_i}}{G_{x_j}}.
\end{equation}
\normalsize
Any $x_k$ with $k>i,j$ is separated from $x_i,x_j$ with a sub-tree. Thus, $\left(\frac{G_{x_i}}{G_{x_j}}\right)_{x_k}=0$ which amounts to 
\small
\begin{equation}\label{extraconstraint1}
\forall\, 2\leq i<j<k\leq n: \left(\frac{F_{x_i}}{F_{x_j}}\right)_{x_k}=0.  
\end{equation}
\normalsize
We now observe that when $n>3$ there are new constraints  imposed on $F$ that did not appear in Example \ref{toy example 2}. Assuming  $3\leq i<j$ in  \eqref{auxiliary}, differentiation with respect to $x_1$ and $x_2$ yields:
\small
$$
\frac{(\frac{F_{x_i}}{F_{x_j}})_{x_1}}{(\frac{F_{x_i}}{F_{x_j}})_{x_2}}=\frac{(\frac{G_{x_i}}{G_{x_j}})_{x_1}}{(\frac{G_{x_i}}{G_{x_j}})_{x_2}}
=\frac{G_{x_1x_i}G_{x_j}-G_{x_1x_j}G_{x_i}}{G_{x_2x_i}G_{x_j}-G_{x_2x_j}G_{x_i}}.
$$
\normalsize
But in the asymmetric, tree any of $x_i$ and $x_j$ is separated from the adjacent leaves $x_1,x_2$. Hence 
$\frac{G_{x_1x_i}}{G_{x_2x_i}}=\frac{G_{x_1x_j}}{G_{x_2x_j}}=\frac{G_{x_1}}{G_{x_2}}$, and the fraction above may be simplified as $\frac{G_{x_1}}{G_{x_2}}$, a function which is independent of $x_3,\dots,x_n$. We thus get the following new constraints: 
\small
\begin{equation}\label{extraconstraint2} 
\text{The ratio } \frac{(\frac{F_{x_i}}{F_{x_j}})_{x_1}}{(\frac{F_{x_i}}{F_{x_j}})_{x_2}}=\frac{F_{x_1x_i}F_{x_j}-F_{x_1x_j}F_{x_i}}{F_{x_2x_i}F_{x_j}-F_{x_2x_j}F_{x_i}}
\text{ is independent of } 3\leq i<j\leq n, \text{ and is a function of }x_1,x_2.  
\end{equation}
\normalsize
Next, proceeding as in Example \ref{toy example 2}, in order to write  $F(x_1,\dots,x_n)$  as $g(x_1,G(x_1,\dots,x_n))$, we seek a coordinate system 
$$
(\xi:=x_1,\zeta,\eta_3,\dots,\eta_{n})
$$
in which $F$ is dependent only on the first two coordinates; and $\zeta=\zeta(x_1,\dots,x_n)$ is a tree function without repetitions as described before (e.g. $G$). So in the original coordinate system $(x_1,\dots,x_n)$, $\nabla F$ must be a linear combination of $\nabla x_1$ and $\nabla\zeta$. 
Therefore, $\frac{\zeta_{x_i}}{\zeta_{x_2}}$ should coincide with $\frac{F_{x_i}}{F_{x_2}}$. So
\small
\begin{equation}\label{auxiliary1}
\nabla\zeta\,\mathlarger{\parallel}\left[1\,\,\frac{\zeta_{x_2}}{\zeta_{x_1}}\,\,\frac{\zeta_{x_3}}{\zeta_{x_1}}\,\cdots\,\frac{\zeta_{x_n}}{\zeta_{x_1}}\right]^{\rm{T}}
=\left[1\,\,\frac{\zeta_{x_2}}{\zeta_{x_1}}\,\,\frac{\zeta_{x_2}}{\zeta_{x_1}}.\frac{F_{x_3}}{F_{x_2}}\,\cdots\,\frac{\zeta_{x_2}}{\zeta_{x_1}}.\frac{F_{x_n}}{F_{x_2}}\right]^{\rm{T}}.
\end{equation}
\normalsize
The ratio $\frac{\zeta_{x_2}}{\zeta_{x_1}}$ should be independent of $x_3,\dots,x_n$ since $\zeta=\zeta(x_1,\dots,x_n)$ is computable by the red tree in Figure \ref{fig:asymmetric}. Denoting it by $\beta=\beta(x_1,x_2)$, \eqref{auxiliary1} translates to 
\small
\begin{equation}\label{auxiliary1'}
\nabla\zeta\parallel\frac{\partial}{\partial x_1}+\frac{\beta(x_1,x_2)}{F_{x_2}}\left(\nabla F-F_{x_1}\frac{\partial}{\partial x_1}\right).
\end{equation}
\normalsize
Switching to differential forms, one should aim for a function $\zeta$ with the property that ${\rm{d}}\zeta$ is a multiple of 
\small
\begin{equation}\label{auxiliary2}
\omega:={\rm{d}}x_1+\frac{\beta(x_1,x_2)}{F_{x_2}}\left({\rm{d}}F-F_{x_1}\,{\rm{d}}x_1\right).    
\end{equation}
\normalsize
Just like Example \ref{toy example 2}, this may be encoded by the integrability condition $\omega\wedge{\rm{d}}\omega=0$. After a lengthy calculation (which is postponed to Appendix \ref{Proofs}), 
$\omega\wedge{\rm{d}}\omega=0$ results in identities
\small
\begin{equation}\label{constraint}
\left(\frac{F_{x_i}}{F_{x_2}}\right)_{x_1}=\frac{1}{\beta}\left(\frac{F_{x_i}}{F_{x_2}}\right)_{x_2}-\left(\frac{1}{\beta}\right)_{x_2}\frac{F_{x_i}}{F_{x_2}}
\end{equation}
\normalsize
which are similar to the identity \eqref{auxiliary8} in Example \ref{toy2}. The difference is that when $n>3$, it is not necessary to differentiate  \eqref{auxiliary8}  to form a linear system which can be solved to obtain $\beta=\beta(x_1,x_2)$ in terms of $F$ as in \eqref{huge3}:  $\frac{1}{\beta}$ turns out to be the ratio $\frac{F_{x_1x_i}F_{x_j}-F_{x_1x_j}F_{x_i}}{F_{x_2x_i}F_{x_j}-F_{x_2x_j}F_{x_i}}$ appeared in \eqref{extraconstraint2}. Proposition \ref{form2} below claims that constraints \eqref{extraconstraint1}, \eqref{extraconstraint2} and \eqref{constraint} provide the PDE characterization of superpositions \eqref{superposition2}. To state it, as usual, we write the constraints as algebraic PDEs. Writing the fraction 
$\frac{F_{x_1x_i}F_{x_j}-F_{x_1x_j}F_{x_i}}{F_{x_2x_i}F_{x_j}-F_{x_2x_j}F_{x_i}}$ as  $\frac{\Psi^{(1)}_{ij}}{\Psi^{(2)}_{ij}}$ with 
\begin{equation}\label{form2-1}
\Psi^{(1)}_{ij}:=F_{x_1x_i}F_{x_j}-F_{x_1x_j}F_{x_i}\quad    \Psi^{(2)}_{ij}:=F_{x_2x_i}F_{x_j}-F_{x_2x_j}F_{x_i}\quad (3\leq i<j\leq n);
\end{equation}
the condition \eqref{extraconstraint2} states that any two fractions  $\frac{\Psi^{(1)}_{ij}}{\Psi^{(2)}_{ij}}$,
$\frac{\Psi^{(1)}_{i'j'}}{\Psi^{(2)}_{i'j'}}$ are the same and independent of any $x_k$ with $k>3$. Furthermore, plugging $\beta=\frac{\Psi^{(2)}_{i'j'}}{\Psi^{(1)}_{i'j'}}$
in \eqref{constraint} results in another set of algebraic PDEs after cross-multiplication.

\begin{proposition}\label{form2}
 Assuming $n>3$, a smooth superposition $F(x_1,\dots,x_n)$ of the form \eqref{superposition2} satisfies  
\begin{equation}\label{form2-2}
F_{x_ix_k}F_{x_j}=F_{x_jx_k}F_{x_i}
\end{equation}
for any three indices $i<j<k$ form $\{2,3,\dots,n\}$ along with the constraints below for any index $l$ and any two pairs $i<j$ and $i'<j'$ of indices all from $\{3,\dots,n\}$:
\small
\begin{equation}\label{form2-3}
\begin{split}
&\Psi^{(1)}_{ij}\Psi^{(2)}_{i'j'}=\Psi^{(1)}_{i'j'}\Psi^{(2)}_{ij}, \\
&\left(\Psi^{(1)}_{ij}\right)_{x_l}\Psi^{(2)}_{ij}-\Psi^{(1)}_{ij}\left(\Psi^{(2)}_{ij}\right)_{x_l}=0,\\
&\left(\Psi^{(2)}_{i'j'}\right)^2\left(F_{x_1x_i}F_{x_2}-F_{x_1x_2}F_{x_i}\right)=\Psi^{(1)}_{i'j'}\Psi^{(2)}_{i'j'}\left(F_{x_2x_i}F_{x_2}-F_{x_2x_2}F_{x_i}\right)
-\left(\left(\Psi^{(1)}_{i'j'}\right)_{x_2}\Psi^{(2)}_{i'j'}-\Psi^{(1)}_{i'j'}\left(\Psi^{(2)}_{i'j'}\right)_{x_2}\right)F_{x_2}F_{x_i};
\end{split}
\end{equation}
\normalsize
where $\Psi^{(1)}_{ij}$ and $\Psi^{(2)}_{ij}$ are defined as in \eqref{form2-1}.
Conversely, a smooth function $F(x_1,\dots,x_n)$ satisfying the algebraic PDEs \eqref{form2-2} and \eqref{form2-3} admits a local representation of the form
\eqref{superposition2} if $F_{x_2}\neq 0$ and there are indices $i_0<j_0$ from  $\{3,\dots,n\}$ for which $\Psi^{(1)}_{i_0j_0}$ and  $\Psi^{(2)}_{i_0j_0}$ are non-zero. 
\end{proposition}
\noindent
The proof will be completed in Appendix \ref{Proofs}.

\section{Superpositions of polynomial functions}\label{Polynomial}
The superpositions we study in this section are constructed out of polynomials. Again, there are two different regimes to discuss: composing general polynomial functions of low dimensionality or composing polynomials of arbitrary dimensionality but in the simpler form of $\mathbf{y}\mapsto\sigma\left(\langle\mathbf{w},\mathbf{y}\rangle\right)$ where the activation function $\sigma$ is a polynomial of a single variable. The latter regime deals with polynomial neural networks. Different aspects of such networks have been studied in the literature   \cite{du2018power,soltanolkotabi2018theoretical,venturi2018spurious,kileel2019expressive}. In the spirit of this paper, we are interested in the spaces formed by such polynomial superpositions. Bounding the total degree of polynomials from the above, these functional spaces are subsets of an ambient polynomial space, say the space $\mathbf{Poly}_{d,n}$ of real polynomials $P(x_1,\dots,x_n)$ of total degree at most $d$ which is an affine space of dimension $\binom{d+n}{n}$. For any degree $d$, there are several subsets of the ambient space $\mathbf{Poly}_{d,n}$ associated with a neural network $\mathcal{N}$ that receives $x_1,\dots,x_n$ as its inputs. We shall use the language of algebraic geometry to discuss them; see Appendix \ref{Background} for a brief introduction.

\begin{definition}\label{variety}
Let $\mathcal{N}$ be a feedforward neural network whose inputs are labeled by the coordinate functions $x_1,\dots,x_n$.
We associate two \textit{polynomial functional spaces} with $\mathcal{N}$:
\begin{enumerate}
\item The subset $\mathbf{F}_{d}(\mathcal{N})$ of $\mathbf{Poly}_{d,n}$ consisting of polynomials $P(x_1,\dots,x_n)$ of total degree at most $d$ that can be computed by $\mathcal{N}$ via assigning real polynomial functions to its neurons;
\item the smaller subset $\mathbf{F}^{\text{act}}_{d}(\mathcal{N})$ of $\mathbf{Poly}_{d,n}$ consisting of polynomials $P(x_1,\dots,x_n)$ of total degree at most $d$ that can be computed by $\mathcal{N}$ via assigning real polynomials of the form $\mathbf{y}\mapsto\sigma\left(\langle\mathbf{w},\mathbf{y}\rangle\right)$ to the neurons where $\sigma$ is a polynomial activation function. 
\end{enumerate} 
The corresponding \textit{functional varieties} $\mathbf{V}_{d}(\mathcal{N})$ and $\mathbf{V}^{\text{act}}_{d}(\mathcal{N})$ are defined as the Zariski closures of 
$\mathbf{F}_{d}(\mathcal{N})$ and $\mathbf{F}^{\text{act}}_{d}(\mathcal{N})$ respectively. 
\end{definition}

Hence $\mathbf{V}_{d}(\mathcal{N})$ and $\mathbf{V}^{\text{act}}_{d}(\mathcal{N})$ are the closures of $\mathbf{F}_{d}(\mathcal{N})$ and $\mathbf{F}^{\text{act}}_{d}(\mathcal{N})$ in the \textit{Zariski topology} on $\mathbf{Poly}_{d,n}$; that is, the smallest subsets defined as zero loci of polynomial equations that contain them.  Notice that by writing a polynomial $P(x_1,\dots,x_n)$ of degree $d$ as 
\begin{equation}\label{general polynomial}
P(x_1,x_2,\dots,x_n)=\sum_{\substack{a_1,a_2,\dots,a_n\geq 0 \\a_1+a_2+\dots+a_n\leq d}}c_{a_1,a_2,\dots,a_n}\,x_1^{a_1}x_2^{a_2}\dots x_n^{a_n},
\end{equation}
the coefficients $c_{a_1,a_2,\dots,a_n}$ provide a natural coordinate system on $\mathbf{Poly}_{d,n}$. Each of the subsets $\mathbf{V}_{d}(\mathcal{N})$ and $\mathbf{V}^{\text{act}}_{d}(\mathcal{N})$ of $\mathbf{Poly}_{d,n}$ could be described with finitely many polynomial equations in terms of $c_{a_1,a_2,\dots,a_n}$'s. The PDE constraints  from \S\ref{necessity} provide non-trivial examples of equations satisfied on the functional varieties: In any degree $d$, substituting \eqref{general polynomial} in an algebraic PDE that  smooth functions computed by $\mathcal{N}$ must obey results in equations in terms of the coefficients that are satisfied at any point of  $\mathbf{F}_{d}(\mathcal{N})$ or $\mathbf{F}^{\text{act}}_{d}(\mathcal{N})$ and hence at the points of $\mathbf{V}_{d}(\mathcal{N})$ or $\mathbf{V}^{\text{act}}_{d}(\mathcal{N})$.
\begin{corollary}\label{equations}
Let $\mathcal{N}$ be a neural network whose inputs are labeled by the coordinate functions $x_1,\dots,x_n$. Then there exist non-trivial polynomials on affine spaces $\mathbf{Poly}_{d,n}$ that are dependent only on the topology of $\mathcal{N}$ and become zero on functional varieties $\mathbf{V}^{\text{act}}_{d}(\mathcal{N})\subset\mathbf{Poly}_{d,n}$. The same holds for functional varieties $\mathbf{V}_{d}(\mathcal{N})$ provided that the number of inputs to each neuron of $\mathcal{N}$ is less than $n$.
\end{corollary}

\begin{proof}
Immediately follows from Theorem \ref{main'} (in the case of $\mathbf{V}^{\text{act}}_{d}(\mathcal{N})$) and from Theorem \ref{main} (in the case of $\mathbf{V}_{d}(\mathcal{N})$). Plugging a polynomial $P(x_1,\dots,x_n)$
in a PDE constraint 
$$\Phi\left(P_{x_1},\dots,P_{x_n},P_{x_1^2}, P_{x_1x_2},\dots,P_{\mathbf{x}^{\mathbf{\alpha}}},\dots\right)=0$$
that these theorems suggest for $\mathcal{N}$, equating the coefficient of a monomial $x_1^{a_1}x_2^{a_2}\dots x_n^{a_n}$ with zero results in a polynomial equation in ambient polynomial spaces that must be satisfied on the associated functional varieties. 
\end{proof}

\begin{remark}
Paper \cite{kileel2019expressive} investigates functional varieties in the case of a multi-layer perceptron $\mathcal{N}$ with a power map $x\mapsto x^r$ as activation functions. In our notations, the functional varieties therein are subvarieties of $\mathbf{V}^{\text{act}}_{d}(\mathcal{N})$ for $d$ large enough. The paper also discusses \textit{filling} architectures, the architectures for which the functional variety is the whole ambient polynomial space  under consideration \cite[Theorem 10]{kileel2019expressive}. The existence of filling architectures may seem to contradict  Corollary \ref{equations} which suggests polynomial equations for $\mathbf{V}^{\text{act}}_{d}(\mathcal{N})$. In such a situation, the PDE constraints hold for all polynomials in the ambient space which  of course could happen only if $r$ is relatively small. Indeed, as $r$ grows the dimensions of functional varieties stabilize while the dimensions of ambient spaces tend to infinity \cite[Theorem 14]{kileel2019expressive}.
\end{remark}

\begin{example}\label{tree polynomial example}
Let $\mathcal{N}$ be a rooted tree $\mathcal{T}$ with distinct inputs $x_1,\dots,x_n$. Constraints of the form $F_{x_ix_k}F_{x_j}=F_{x_jx_k}F_{x_i}$ are not only necessary conditions for a smooth function $F=F(x_1,\dots,x_n)$ to be computable by $\mathcal{T}$; but by the virtue of Theorem \ref{main-tree}, they are also sufficient for the existence of a local representation of $F$ on $\mathcal{T}$ if suitable non-vanishing conditions are satisfied. An interesting feature of this setting is that when $F$ is a polynomial $P=P(x_1,\dots,x_n)$, one can relax the non-vanishing conditions; and $P$ actually admits a global representation as a composition of polynomials if it satisfies the characteristic PDEs \cite[Proposition 4]{Farhoodi2019OnFC}.
The basic idea is that if $P$ is locally written as a superposition of smooth functions according to the hierarchy provided by $\mathcal{T}$, then comparing the Taylor series shows that the constituent parts of the superposition could be chosen to be polynomials as well. Now $P$ and such a polynomial superposition must be the same since they agree on a non-empty open set.  Consequently, each $\mathbf{F}_{d}(\mathcal{N})$ coincides with its closure $\mathbf{V}_{d}(\mathcal{N})$ and can be described by equations of the form $P_{x_ix_k}P_{x_j}=P_{x_jx_k}P_{x_i}$ in the polynomial space.
Substituting an expression of the form 
$$P(x_1,\dots,x_n)=\sum_{a_1,\dots,a_n\geq 0}c_{a_1,\dots,a_n}\,x_1^{a_1}\dots x_n^{a_n}$$
in $P_{x_ix_k}P_{x_j}-P_{x_jx_k}P_{x_i}=0$ and equating the coefficient of a monomial $x_1^{a_1}\dots x_n^{a_n}$ with zero yields: 
\begin{equation}\label{huge6}
\mathlarger{\sum}_{\substack{a'_i+a''_i=a_i+1,\, a'_j+a''_j=a_j+1,\, a'_k+a''_k=a_k+1\\a'_s+a''_s=a_s\, \forall s\in\{1,\dots,n\}-\{i,j,k\}}} a'_k\big(a'_ia''_j-a'_ja''_i\big)\,c_{a'_1,\dots,a'_n}c_{a''_1,\dots,a''_n}=0. 
\end{equation}
We deduce that equations \eqref{huge6} written for $a_1,\dots,a_n\geq 0$ and for triples $(i,j,k)$ with the property that $x_k$ is separated from $x_i$ and $x_j$ by a sub-tree of $\mathcal{T}$ (as in Theorem \ref{main-tree}) describe the functional varieties associated with $\mathcal{T}$. In a given degree $d$, to obtain equations describing $\mathcal{\mathcal{T}}$ in $\mathbf{Poly}_{d,n}$ one should set any $c_{b_1,\dots,b_n}$  with $b_1+\dots+b_n>d$ to be zero in \eqref{huge6}. No such a coefficient occurs if  $d\geq a_1+\dots+a_n+3$, and thus for $d$ large enough \eqref{huge6} defines an equation in $\mathbf{Poly}_{d,n}$ as is.\\
\indent Similarly, Theorem \ref{main-tree-activation} can be used to write equations for $\mathbf{F}^{\text{act}}_{d}(\mathcal{N})=\mathbf{V}^{\text{act}}_{d}(\mathcal{N})$. In that situation, a new family of equations corresponding to \eqref{temp3} emerge that are expected to be extremely complicated. 
\end{example}

\begin{example}\label{network polynomial example}
Let $\mathcal{N}$ be the neural network appearing in Figure \ref{fig:basic-2s}. The functional space $\mathbf{F}^{\text{act}}_{d}(\mathcal{N})$ is formed by polynomials $P(x,t)$ of total degree at most $d$ that are in the form of  $\sigma(f(ax+bt)+g(a'x+b't))$. By examining the Taylor expansions, it is not hard to see that if $P(x,t)$ is written in this form for univariate smooth functions $\sigma$, $f$ and $g$, then these functions could be chosen to be polynomials. Therefore, in any degree $d$, our characterization of superpositions of this form in Example \ref{network1} in terms of PDEs and PDIs results in polynomial equations and inequalities that describe  a Zariski open subset of $\mathbf{F}^{\text{act}}_{d}(\mathcal{N})$  which is the complement of  the locus where the non-vanishing conditions fail. The inequalities disappear after taking the closure, so   $\mathbf{V}^{\text{act}}_{d}(\mathcal{N})$ is strictly larger than $\mathbf{F}^{\text{act}}_{d}(\mathcal{N})$ here.
\end{example}

The emergence of inequalities in describing the functional spaces is not surprising due to the \textit{Tarski–Seidenberg Theorem} (see \cite{coste2000introduction}) which implies that the image of a \textit{polynomial map} between real varieties (i.e. a map whose components are polynomials) is semi-algebraic; that is, could be described as a union of finitely many sets defined by polynomial equations and inequalities. To elaborate, fix a neural network architecture $\mathcal{N}$. Composing polynomials of bounded degrees according to the hierarchy provided by $\mathcal{N}$ yields polynomial superpositions lying in $\mathbf{F}_D(\mathcal{N})$ for $D$ sufficiently large. The composition thus amounts to a map
$$
\mathbf{Poly}_{d_1,n_1}\times\dots\times \mathbf{Poly}_{d_N,n_N}\rightarrow \mathbf{Poly}_{D,n}
$$
where on the left-hand side the polynomials assigned to the neurons of $\mathcal{N}$ appear, and $D\gg d_1,\dots,d_n$. The image, which is a subset of $\mathbf{F}_D(\mathcal{N})$, is semi-algebraic and thus admits a description in terms of finitely many polynomial equations and inequalities. The same logic applies to the regime of activation functions too; the map just mentioned must be replaced with
$$
\Bbb{R}^C\times\mathbf{Poly}_{d_1,1}\times\dots\times \mathbf{Poly}_{d_N,1}\rightarrow \mathbf{Poly}_{D,n}
$$
whose image lies in $\mathbf{F}^{\text{act}}_{D}(\mathcal{N})$; and its domain is the Cartesian product of spaces of polynomial activation functions assigned to the neurons by the space $\Bbb{R}^C$ of weights assigned to the connections of the network.\\
\indent The novelty of our approach is to derive the equations and inequalities discussed above from PDEs and PDIs that are imposed generally and  depend only on the architecture. We have demonstrated how this could be done in the case of equations in Corollary \ref{equations} and Example \ref{tree polynomial example}. As for the PDIs, one needs to argue that an algebraic PDI such as \eqref{alg PDI} defines a semi-algebraic set in a polynomial space. This is a result of the lemma below whose proof will appear in Appendix \ref{Background}.
\begin{lemma}\label{semi-algebraic}
Real polynomials $P(x_1,\dots,x_n)$ of total degree at most $d$ that are positive at every point of $\Bbb{R}^n$ form a semi-algebraic subset of  $\mathbf{Poly}_{d,n}$.
\end{lemma}

We conclude the section by formulating the algebraic counterpart of Conjecture \ref{conjecture}. First, notice that one could only hope for a characterization of a Zariski open subset of polynomial functional spaces because the non-vanishing conditions are necessary in general. For instance, the function $xyz+x+y+z$ from \eqref{non-example} satisfies the PDE constraints \eqref{toy1-3} derived in Example \ref{toy example 1} but cannot be represented as $g(f(x,y),h(x,z))$ even in the  continuous category \cite[\S 7.2]{Farhoodi2019OnFC}. The reason is that the non-vanishing condition does not hold; the expression $\Psi$ defined in \eqref{toy1-1} is identically zero for the polynomial $xyz+x+y+z$. A more subtle issue is that, unlike Examples \ref{tree polynomial example} and \ref{network polynomial example}, in general there may be polynomials that satisfy the characteristic equations and the non-vanishing conditions but lack a representation in terms of polynomials; they only admit a local representation of the desired form in terms of smooth (analytic) functions. 
Going back to superpositions $g(f(x,y),h(x,z))$ implemented by a tree architecture with repeated inputs (Figure \ref{fig:basic-2}), the polynomial $P(x,y,z)=(x+y)z+y^3z^3$ is an example of this type: It may be written as   
\begin{equation*}
\begin{split}
&(x+y)z+y^3z^3=\left(1+\frac{y}{x}\right)xz+\left(\frac{y}{x}\right)^3(xz)^3=g(f(x,y),h(x,z)),\\
&\text{ where } f(x,y)=\frac{y}{x},\, h(x,z)=xz \text{ and } g(u,v)=(u+1)v+u^3v^3;
\end{split}
\end{equation*}
\normalsize
so it satisfies the characteristic PDEs \eqref{toy1-3}; it even satisfies the non-vanishing condition \eqref{toy1-1}. But the proposition below (whose proof will be presented in Appendix \ref{Proofs}) states that it cannot be expressed as a polynomial superposition of this form.   
\begin{proposition}\label{rational functions}
The polynomial $P(x,y,z)=(x+y)z+y^3z^3$ cannot be written as a composition $g(f(x,y),h(x,z))$ of polynomial functions.
\end{proposition}
\noindent
We finally arrive at Conjecture \ref{conjecture'}. It is not an immediate corollary of Conjecture \ref{conjecture} due to the emergence of non-polynomial functions in the superpositions which we just encountered.

\begin{conjecture}\label{conjecture'}
Let $\mathcal{N}$ be a feedforward neural network whose inputs are labeled by the coordinate functions $x_1,\dots,x_n$. Then there exist 
\begin{itemize}
\item finitely many algebraic PDEs $\left\{\Phi_j\left(\left(F_{\mathbf{x}^{\mathbf{\alpha}}}\right)_{|\alpha|\leq r}\right)=0\right\}_{j}$, 
\item finitely many algebraic PDIs $\left\{\Theta_k\left(\left(F_{\mathbf{x}^{\mathbf{\alpha}}}\right)_{|\alpha|\leq r}\right)>0\right\}_{k}$;
\end{itemize}
with the property that for $d$ large enough the polynomial functional space $\mathbf{F}^{\rm{act}}_d(\mathcal{N})$ and the semi-algebraic set defined by $\{\Phi_j=0, \Theta_k> 0\}_{j,k}$ in $\mathbf{Poly}_{d,n}$ share a dense Zariski open  subset. The same holds for the other polynomial functional spaces $\mathbf{F}_{d}(\mathcal{N})$ provided that the number of inputs to each neuron of $\mathcal{N}$ is less than $n$.
\end{conjecture}

\begin{remark}
By working with polynomials with complex coefficients, one can define the complex versions $\mathbf{F}_{d,\Bbb{C}}(\mathcal{N})$ and $\mathbf{F}^{\text{act}}_{d,\Bbb{C}}(\mathcal{N})$ of the spaces appeared in Definition \ref{variety}. These are complex polynomial functional spaces that lie in the ambient space $\mathbf{Poly}_{d,n,\Bbb{C}}$ of complex polynomials of $n$ variables and of total degree not greater than $d$ where $n$ is the number of distinct inputs that $\mathcal{N}$ receives. An important difference with the real case is that inequalities (and hence PDIs) are not necessary  for characterizing functional spaces anymore. Over complex numbers, the Tarski–Seidenberg Theorem is replaced with \textit{Chevalley's Theorem} (see \cite{MR1416564,milneAG}) that implies the image of a \textit{morphism} of complex varieties is \textit{constructible}; that is, could be described with finitely many polynomial equations combined using ``and'', ``or'' and ``not''. The fact that there is no need for PDIs in characterizing complex-valued functions is tacitly mentioned in Example \ref{network1} where we argued that a solution to a $2^{\rm{nd}}$ order homogeneous elliptic PDE with constant coefficients (e.g. Laplace's equation) cannot be represented as $f(ax+bt)+g(a'x+b't)$ with $a,a',b,b'$ real because it fails to satisfy a certain PDI. But of course, there exists such a representation over complex numbers.  
\end{remark}



\section{Conclusion}\label{conclusion}
In this article, we proposed a systematic method for studying smooth real-valued functions constructed as compositions of other smooth functions which are either 1) of lower arity $\,$ or $\,$  2) in the form of a univariate activation function applied to a linear combination of inputs.  We established that any such smooth superposition must satisfy non-trivial constraints in the form of algebraic PDEs which are dependent only on the hierarchy of composition or equivalently, only on the topology of the neural network that produces superpositions of this type. We conjectured that there always exist characteristic PDEs that also provide sufficient conditions for a generic smooth function to be expressible by the feedforward neural network in question. The genericity is to avoid singular cases and is captured by non-vanishing conditions which require certain polynomial functions of partial derivatives to be non-zero. We observed that there are also situations where non-trivial algebraic inequalities involving partial derivatives (PDIs) are imposed on the hierarchical functions. In summary, the conjecture aims to describe generic smooth functions computable by a neural network with finitely many universal conditions of the form $\Phi\neq 0$, $\Psi=0$ and $\Theta>0$ where $\Phi$, $\Psi$ and $\Theta$ are polynomial expressions of the partial derivatives and are dependent only on the architecture of the network, not on any tunable parameter or on any activation function used in the network. This is reminiscent of the notion of a semi-algebraic set from real algebraic geometry. Indeed, in the case of compositions of polynomial functions or functions computed by polynomial neural networks, the PDE constraints yield equations for the corresponding functional variety in an ambient space of polynomials of a prescribed degree. 

The conjecture was verified in several cases, most importantly, for tree architectures with distinct inputs where, in each regime, we explicitly exhibited a PDE characterization of functions computable by a tree network. Two families of tree architectures with repeated inputs were addressed as well. The proofs were mathematical in nature and relied on classical results of multi-variable analysis. 

The article moreover highlights the differences between the two regimes mentioned at the beginning; namely, the hierarchical functions constructed out of composing functions of lower dimensionality and the hierarchical functions which are compositions of functions of the form  $\mathbf{y}\mapsto\sigma\left(\langle\mathbf{w},\mathbf{y}\rangle\right)$. The former functions appear more often in the mathematical literature on the Kolmogorov-Arnold Representation Theorem while the latter are ubiquitous in deep learning. The special form of functions $\mathbf{y}\mapsto\sigma\left(\langle\mathbf{w},\mathbf{y}\rangle\right)$ requires more PDE constraints to be imposed on their compositions whereas their mild non-linearity is beneficial in terms of ascertaining the domain on which a claimed compositional representation exists.

Our approach for describing the functional spaces associated with feedforward neural networks is of natural interest in the study of expressivity of neural networks, and could lead to new complexity measures. We believe that the point of view adapted here is novel and might shed light on a number of practical problems such as comparison of architectures and reverse-engineering deep networks.

\section{Future directions}\label{future}
We finish with several possible related research directions. They are categorized and the statements of questions appear in italic. 

\begin{itemize}
\item \textbf{Conjectures} \ref{conjecture} and \ref{conjecture'} 
\vspace{1mm} 
\begin{enumerate}
\item  \textit{How Conjecture \ref{conjecture} could be established?} The conjecture is central to this article. As the cases of it verified in this paper suggest, a proof of Conjecture \ref{conjecture} should involve constructing functions whose gradients are in a prescribed direction. This should be based on an integrability argument exploiting the PDE constraints. 
\item The role of PDIs in Conjecture \ref{conjecture} remains mysterious. As discussed in \S\ref{Polynomial}, it is reasonable to include them in the characterization in analogy with the inequalities appearing in descriptions of semi-algebraic sets. On the other hand, we have witnessed inequalities only in Examples \ref{basic-2} and \ref{network1} which were concerned with bivariate functions obtained from composing bivariate functions of the form \eqref{activation}. \textit{Could the usage of PDIs be avoided in the characterization of those superpositions whose constituent functions are of lower arity?}
\item As discussed in \S\ref{Polynomial}, Conjecture \ref{conjecture'} is more subtle than Conjecture \ref{conjecture} since it asks for superpositions in the realm of polynomial functions; and there are polynomials satisfying the relevant PDE constraints that nevertheless cannot be written as superpositions of polynomial functions in the desired form; see Proposition \ref{rational functions}. \textit{How non-generic/pathological are such examples of polynomial functions? How large the subspace they form in the space of polynomial solutions to the characteristic PDEs could be?} 
\end{enumerate}
\vspace{1mm}
\item \textbf{Mathematical Aspects}
\vspace{1mm}
\begin{enumerate}
\setcounter{enumi}{3}
\item \textit{Is there a conceptual way of deriving PDE constraints on  functions computable by an architecture?} 
Our techniques, although may seem ad hoc,  have all been based on algebraic or linear dependence of expressions obtained from differentiating enough times. Another approach suggested in \cite{buck1976approximate,MR541075,MR606252} is the use of \textit{characteristic automorphisms} which alter the constituent functions of superpositions while preserving their general form. For instance, the family of functions $F(x,y,z)=g(f(x,y),z)$ studied in Example \ref{basic-1} is the smallest family that contains the function $F(x,y,z)=x$ and  is invariant under the operations
$$
F(x,y,z)\mapsto F(h(x,y),y,z),\quad F(x,y,z)\mapsto h(F(x,y,z),z);
$$
for any bivariate function $h=h(u,v)$. 
It is not hard to verify that the PDE $F_{xz}F_y=F_{yz}F_x$, which characterizes the family, is preserved under the operations above: If it holds for a function $F$, it remains valid after applying the operations. \textit{Could the characteristic PDEs be obtained as the invariants of certain characteristic automorphisms?}
\item The following question is raised in \cite{buck1976approximate,MR541075}: \textit{Does a smooth function representable as a superposition of continuous functions satisfy 
the PDE constraints that the same class of superpositions obey once their constituent functions are smooth?} When the characteristic PDEs exist (i.e. Conjecture \ref{conjecture} holds), the positive answer to the preceding question implies that the function in question can indeed be written as a superposition of smooth functions. Answering this question requires a notion of \textit{weak solution} for the non-linear PDEs we have encountered in this paper.  
\item  \textit{How the dimensions of the functional varieties introduced in Definition \ref{variety} could be computed/estimated?} Aside from being interesting mathematically, the question is also important since the dimension of the functional varieties associated with an architecture could be interpreted as a measure of its expressive power \cite{Farhoodi2019OnFC,kileel2019expressive}. Unfortunately, the codimension of the functional space is not the same as the number of PDEs characterizing it due to the following two reasons:
\begin{enumerate}
\item Imposing a PDE constraint may increase the codimension by more than one since it is an equality of functions. For instance, asking for a polynomial function to be identically zero requires all coefficients to vanish.
\item Although all PDEs characterizing a class of superpositions are needed for describing the whole functional space, some of them could be deduced from the others away from a subspace of positive codimension. For instance, in Theorem \ref{main-tree} if the first order partial derivatives of $F$ are not identically zero, the PDEs $F_{x_ix_k}F_{x_j}=F_{x_jx_k}F_{x_i}$ and $F_{x_lx_k}F_{x_j}=F_{x_jx_k}F_{x_l}$ imply $F_{x_ix_k}F_{x_l}=F_{x_lx_k}F_{x_i}$; cf. \cite[\S5.4]{Farhoodi2019OnFC}.
\end{enumerate}
\item After discussing their dimensions, we point out that more generally, the functional varieties  are worth studying from algebro-geometric perspective. Here are two questions of this sort: \textit{Is there an easy description of their tangent spaces? Which polynomial compositions correspond to singular points of these varieties?} 
\item \textit{How the PDE characterization changes once some tunable parameters of the network are specified?} 
To elaborate, recall that the PDE constraints we have derived throughout the paper have been dependent only on the architecture. But upon restricting to a particular class of activation functions or setting some of the weights to be zero, stronger PDE constraints are anticipated. For instance, an activation function such as softmax is a solution to a polynomial ODE and this facilitates the derivation of PDE constraints (cf. Remark \ref{ODE}); or a nomographic function $\sigma\left(\sum_{i=1}^nf_i(x_i)\right)$ satisfies PDE constraints that do not hold for an arbitrary function of the form 
$\sigma\left(\sum_{i=1}^nf_i\left(\sum_{j=1}^nw_{ij}x_j\right)\right)$. 
\end{enumerate}
\vspace{1mm}
\item \textbf{Learning Theory Questions}
\vspace{1mm}
\begin{enumerate}
\setcounter{enumi}{8}
\item \textit{How our characterizations of functional spaces associated with neural networks could be used to practically compare two different architectures in terms of their expressive power?}
\item   There are papers discussing how a neural network may be  ``reverse-engineered''  in the sense that the architecture of the network is determined from the knowledge of its outputs,  or the weights and biases are recovered without the ordinary training process involving gradient descent algorithms \cite{fefferman1994recovering,dehmamy2019direct,2019arXiv191000744R}. In our approach, the weights appearing in a composition of functions of the form $\mathbf{y}\mapsto\sigma\left(\langle\mathbf{w},\mathbf{y}\rangle\right)$ could be described (up to scaling) in terms of partial derivatives of the resulting superposition; see Remark \ref{reverse-engineer}. \textit{In approximating a given function with functions expressible by an architecture, how the weights resulted from the gradient descent are compared with the formulas available for the weights of expressible functions? Could this approach be used to augment/forgo the use of gradient descent in training?}
\item In view of our proposed description of functional spaces associated with a feedforward neural network, the following question is worth asking: \textit{In the training process, how the distance of the cost function from the space of functions expressible by the network is related to the training error?} The distance of a function $C$ from a functional space defined by PDEs $\left\{\Phi_j\left(\left(F_{\mathbf{x}^{\mathbf{\alpha}}}\right)_{|\alpha|\leq r}\right)=0\right\}_{j}$ (as in Conjecture \ref{conjecture}) is expected to  be correlated with the magnitudes of numbers  $\Phi_j\left(\left(C_{\mathbf{x}^{\mathbf{\alpha}}}\right)_{|\alpha|\leq r}\right)$ \cite{buck1976approximate,MR541075}.  
\end{enumerate}
\end{itemize}

\appendix     

\section{Technical proofs}\label{Proofs}

\begin{proof}[Proof of Lemma \ref{technical}] 
We first prove that $F$ can be extended to a coordinate system on the entirety of the box-like region $B$ which we shall write as $I_1\times\dots\times I_n$. As in the proof of Theorem \ref{main-tree}, we group the variables $x_1,\dots,x_n$ according to the maximal sub-trees  of $\mathcal{T}$ in which they appear: 
$$x_1,\dots,x_{m_1};\, x_{m_1+1},\dots,x_{m_1+m_2};\,\dots;\,x_{m_1+\dots+m_{l-1}+1},\dots,x_{m_1+\dots+m_l};\,x_{m_1+\dots+m_l+1};\,\dots;\,x_n$$
where, denoting the sub-trees emanating from the root of $\mathcal{T}$ by $\mathcal{T}_1,\dots,\mathcal{T}_l$, for any $1\leq s\leq l$  the leaves of $\mathcal{T}_s$ are labeled by  $x_{m_1+\dots+m_{s-1}+1},\dots,x_{m_1+\dots+m_{s-1}+m_s}$; and $x_{m_1+\dots+m_l+1},\dots,x_n$ represent the leaves   which are directly connected to the root (if any); see  Figure \ref{fig:general_tree_many_subtrees}. 
Among the variables labeling the leaves of $\mathcal{T}_1$, there should exist one with respect to which the first order partial derivative of $F$ is nowhere zero. Without any loss of generality, we may assume that $F_{x_1}\neq 0$ at any point of $B$. Hence, the Jacobian of the map $(F,x_2,\dots,x_n):B\rightarrow\Bbb{R}^n$ is always invertible. To prove that the map provides a coordinate system, we just need to show that it is injective. Keeping $x_2,\dots,x_n$ constant and varying $x_1$, we obtain a univariate function of $x_1$ on the interval $I_1$  whose derivative is always non-zero and is hence injective. \\
\indent 
Next, to prove that the level sets of $F:B\rightarrow\Bbb{R}$ are connected, notice that $F$ admits a representation 
\begin{equation}\label{auxiliary18}
F(x_1,\dots,x_n)=\sigma\left(w_1G_1+\dots+w_lG_l+w'_{m_1+\dots+m_l+1}x_{m_1+\dots+m_l+1}+\dots+w'_nx_n\right)   
\end{equation}
where $G_s=G_s(x_{m_1+\dots+m_{s-1}+1},\dots,x_{m_1+\dots+m_{s-1}+m_s})$ is the tree function that $\mathcal{T}_s$ computes by receiving $x_{m_1+\dots+m_{s-1}+1},\dots,x_{m_1+\dots+m_{s-1}+m_s}$ from its leaves; $\sigma$ is the activation function assigned to the root of $\mathcal{T}$; and $w_1,\dots,w_l, w'_{m_1+\dots+m_l+1},\dots,w'_n$ are the weights appearing at the root. 
A simple application of the chain rule implies that  $G_1$ -- which is a function implemented on the tree $\mathcal{T}_1$ -- satisfies the non-vanishing hypotheses of Theorem \ref{main-tree-activation} on the box-like region $I_1\times\dots\times I_{m_1}$; and moreover, the derivative of $\sigma$ is non-zero at any point of its domain\footnote{As the vector $(x_1,\dots,x_n)$ of inputs varies in the box-like region $B$, the inputs to each node form an interval on which the corresponding activation function is defined.} because otherwise there exists a point of $\mathbf{p}\in B$ at which $F_{x_i}(\mathbf{p})=0$ for any leaf $x_i$. By the same logic, the weight $w_1$ must be non-zero because otherwise all first order partial derivatives with respect to the variables appearing in  $\mathcal{T}_1$ are identically zero. We now show that an arbitrary level set $L_c:=\{\mathbf{x}\in B\,|\,F(\mathbf{x})=c\}$ is connected. Given the representation \eqref{auxiliary18} of $F$, the level set is empty if $\sigma$ does not attain the value $c$. Otherwise, $\sigma$ attains $c$ at a unique point $\sigma^{-1}(c)$ of its domain. So one may rewrite the equation $F(x_1,\dots,x_n)=c$ as
\begin{equation}\label{auxiliary19}
G_1=-\frac{w_2}{w_1}\,G_2-\dots-\frac{w_l}{w_1}\,G_l-\frac{w'_{m_1+\dots+m_l+1}}{w_1}\,x_{m_1+\dots+m_l+1}-\dots-\frac{w'_n}{w_1}\,x_n+\frac{1}{w_1}\,\sigma^{-1}(c).
\end{equation}
The left-hand side of the last equation is a function of $x_1,\dots,x_{m_1}$ while its right-hand side, which we denote by $\tilde{G}$, is a function of $x_{m_1+1},\dots,x_n$. Therefore, the level set $L_c$ is the preimage of 
\begin{equation}\label{auxiliary20}
\left\{(y,\mathbf{\tilde{x}})\in\Bbb{R}\times\left(I_{m_1+1}\times\dots\times I_n\right)\,|\,y=\tilde{G}(\mathbf{\tilde{x}})\right\}
\end{equation}
under the map 
\begin{equation}\label{auxiliary21}
\begin{cases}
\pi:B=(I_1\times\dots\times I_{m_s})\times(I_{m_1+1}\times\dots\times I_n)\rightarrow\Bbb{R}\times\left(I_{m_1+1}\times\dots\times I_n\right)\\
(x_1,\dots,x_{m_1};\mathbf{\tilde{x}})\mapsto\left(G_1(x_1,\dots,x_{m_1}),\mathbf{\tilde{x}}\right).
\end{cases}    
\end{equation}
The following simple fact can now be invoked: \textit{Let $\pi:X\rightarrow Y$ be a continuous map of topological spaces that takes open sets to open sets and has connected level sets. Then the preimage of any connected subset of $Y$ under $\pi$ is connected.}
Here, $L_c$ is the preimage of \eqref{auxiliary20} --  which is connected since it is the graph of a continuous function -- under the map $\pi$ defined in \eqref{auxiliary21} which is open because the scalar-valued function $G_1$ is: its gradient never vanishes. Therefore, the connectedness of the level sets of $F$ is implied by the connectedness of the level sets of $\pi$. A level set of the map \eqref{auxiliary21} could be identified with a level set of its first component $G_1$. Consequently, we have reduced to the similar problem for the function $G_1$ which is implemented on the smaller tree $\mathcal{T}_1$.  Therefore, an inductive argument yields the connectedness of the level sets of $F$. It only remains to check the basic case of a tree whose leaves are directly connected to the root. In that setting, $F(x_1,\dots,x_n)$ is in the form of $\sigma(a_1x_1+\dots+a_nx_n)$ (the family of functions that Lemma \ref{integrability'} is concerned with). By repeating the argument used before, the activation function $\sigma$ is injective. Hence, a level set $F(\mathbf{x})=c$ is the intersection of the hyperplane $a_1x_1+\dots+a_nx_n=\sigma^{-1}(c)$ with the box-like region $B$. Such an intersection is convex and thus connected. 
\end{proof}

\begin{proof}[Proof of Lemma \ref{integrability'}]
The necessity of conditions $F_{x_ix_k}F_{x_j}=F_{x_jx_k}F_{x_i}$ follows from a simple computation. For the other direction, suppose $F_{x_j}\neq 0$ throughout an open box-like region $B\subseteq\Bbb{R}^n$, and any ratio $\frac{F_{x_i}}{F_{x_j}}$ is constant on $B$. Denoting it by $a_i$, we obtain numbers $a_1,\dots,a_n$ with $a_j=1$. They form a vector 
$[a_1\,\,\cdots\,\,a_n]^{\rm{T}}$ parallel to $\nabla F$. Thus, $F$ could have non-zero first order partial derivative only with respect to the first member of the coordinate system 
$$(a_1x_1+\dots+a_nx_n,x_1,\dots,x_{j-1},x_{j+1},\dots,x_n)$$
for $B$. The coordinate hypersurfaces are connected since they are intersections of  hyperplanes in $\Bbb{R}^n$ with the convex region $B$. This fact enables us to deduce that $F$ can be  written as a function of $a_1x_1+\dots+a_nx_n$ globally.   
\end{proof}

\begin{proof}[Proof of Lemma \ref{split}]
For a function $q=q_1\,q_2$ such as \eqref{product} equalities of the form $q\,q_{y^{(1)}_{a}y^{(2)}_{b}}=q_{y^{(1)}_{a}}\,q_{y^{(2)}_{b}}$ hold since both sides coincide with 
$q_1\,q_2\,(q_1)_{y^{(1)}_{a}}(q_2)_{y^{(2)}_{b}}$. For the other direction, let $q=q\left(y^{(1)}_1,\dots,y^{(1)}_{n_1};y^{(2)}_1,\dots,y^{(2)}_{n_2}\right)$ be a smooth function on an open box-like region $B_1\times B_2\subseteq\Bbb{R}^{n_1}\times\Bbb{R}^{n_2}$ that satisfies $q\,q_{y^{(1)}_{a}y^{(2)}_{b}}=q_{y^{(1)}_{a}}\,q_{y^{(2)}_{b}}$ for any $1\leq a\leq n_1$ and $1\leq b\leq n_2$, and never vanishes. So $q$ is either always positive or always negative. One may assume the former by replacing $q$ with $-q$ if necessary. Hence, we can define a new function $p:={\rm{Ln}}(q)$ by taking the logarithm. We have:
\small
$$
p_{y^{(1)}_{a}y^{(2)}_{b}}=\left(\frac{q_{y^{(1)}_{a}}}{q}\right)_{y^{(2)}_{b}}=\frac{q\,q_{y^{(1)}_{a}y^{(2)}_{b}}-q_{y^{(1)}_{a}}\,q_{y^{(2)}_{b}}}{q^2}=0.
$$
\normalsize
It suffices to show that this vanishing of mixed partial derivatives allows us to write  $p\left(y^{(1)}_1,\dots,y^{(1)}_{n_1};y^{(2)}_1,\dots,y^{(2)}_{n_2}\right)$ as  
$p_1\left(y^{(1)}_1,\dots,y^{(1)}_{n_1}\right)+p_2\left(y^{(2)}_1,\dots,y^{(2)}_{n_2}\right)$
since then exponentiating yields $q_1$ and $q_2$ as ${\rm{e}}^{p_1}$ and ${\rm{e}}^{p_2}$ respectively. The domain of $p$ is a box-like region of the form 
\small
$$
B_1\times B_2=\left(\prod_{a=1}^{n_1}I^{(1)}_{a}\right)\times\left(\prod_{b=1}^{n_2}I^{(2)}_{b}\right).
$$
\normalsize
Picking an arbitrary point $z^{(1)}_{1}\in I^{(1)}_{1}$, the Fundamental Theorem of Calculus implies: 
\small
$$
p\left(y^{(1)}_1,\dots,y^{(1)}_{n_1};y^{(2)}_1,\dots,y^{(2)}_{n_2}\right)=
\mathlarger{\int}_{z^{(1)}_1}^{y^{(1)}_1}p_{y^{(1)}_{1}}
\left(s^{(1)}_{1},y^{(1)}_{2},\dots,y^{(1)}_{n_1};y^{(2)}_1,\dots,y^{(2)}_{n_2}\right){\rm{d}}s^{(1)}_{1}+
p\left(z^{(1)}_1,y^{(1)}_{2},\dots,y^{(1)}_{n_1};y^{(2)}_1,\dots,y^{(2)}_{n_2}\right).
$$
\normalsize
On the right-hand side, the integral is dependent  only on $y^{(1)}_1,\dots,y^{(1)}_{n_1}$ because the partial derivatives of the integrand with respect to $y^{(2)}_1,\dots,y^{(2)}_{n_2}$  are all identically zero. The second term $p\left(z^{(1)}_1,y^{(1)}_{2},\dots,y^{(1)}_{n_1};y^{(2)}_1,\dots,y^{(2)}_{n_2}\right)$ 
is a function on the smaller box-like region 
$$
\left(\prod_{a=2}^{n_1}I^{(1)}_{a}\right)\times\left(\prod_{b=1}^{n_2}I^{(2)}_{b}\right)
$$
in $\Bbb{R}^{n_1-1}\times\Bbb{R}^{n_2}$ and thus, proceeding inductively, can be brought into the appropriate summation form.
\end{proof}

\begin{proof}[Proof of Proposition \ref{form1}]
Equation $\eqref{auxiliary16}$ along with the $2n-3$ equations obtained from applying operators $\partial_{x_2},\partial_{x_3},\dots,\partial_{x_n}$ and 
$\partial_{x_2x_3},\partial_{x_3x_4},\dots,\partial_{x_{n-1}x_n}$ to it constitute the system below: 
\scriptsize
\begin{equation}\label{system'}
\left[\begin{array}{@{}l|l@{}}
  \begin{matrix}
  F_{x_2}\quad & \hspace{1cm} F_{x_3}\quad & \hspace{2mm}\cdots & \quad\, F_{x_n}
  \end{matrix} &
  \hspace{2.5cm}\bigzero \\ 
\hline
  \begin{matrix}
  F_{x_2x_2} & \hspace{0.9cm} F_{x_2x_3} & \quad\cdots & \quad F_{x_2x_n}\\
  \vdots     & \hspace{0.9cm} \vdots     & \quad\ddots & \quad \vdots    \\
  F_{x_nx_2} & \hspace{0.9cm} F_{x_nx_3} & \quad \cdots & \quad F_{x_nx_n}
  \end{matrix} &
  \begin{matrix}
 \,\,F_{x_2}   &  \quad  0     & \quad\; 0      & \quad \cdots  & \quad 0      \\
  \,\,0        &\quad  F_{x_3} & \quad\;  0     & \quad  \cdots & \quad 0      \\
  \,\,\vdots   &\quad  \vdots  &\quad\; \vdots  & \quad \vdots  & \quad \vdots \\
  \,\,0        &\quad 0        &  \quad\;  0    &  \quad\cdots  & \quad F_{x_n}
  \end{matrix}\\
\hline
  \begin{matrix}
  F_{x_2x_3x_2}\quad & F_{x_2x_3x_3}     & \cdots & F_{x_2x_3x_n}\\
  \vdots  \quad\quad & \vdots            & \ddots & \vdots        \\
  F_{x_{n-1}x_nx_2}  & F_{x_{n-1}x_nx_3} & \cdots & F_{x_{n-1}x_nx_n}
  \end{matrix} &
  \begin{matrix}
  F_{x_2x_3}  &    F_{x_2x_3}  & 0                &  \cdots & 0          \\
  0           & F_{x_3x_4}     &    F_{x_3x_4}    &  \cdots & 0          \\
  \vdots      & \vdots         & \vdots           &  \ddots & \vdots     \\
  0           & 0              &    0             &  \cdots & F_{x_{n-1}x_n}
  \end{matrix}
\end{array}\right]
\left[\begin{array}{c}
 \begin{matrix}
 A_{2}\\
 A_{3}\\
 \vdots\\
 A_{n}
 \end{matrix}\\
\hline
 \begin{matrix}
 (A_{2})_{x_2}\\
 (A_{3})_{x_3}\\
 \vdots\\
 (A_{n})_{x_n}
 \end{matrix}
\end{array}\right]=
\left[\begin{array}{c}
 F_{x_1}\\
\hline 
 \begin{matrix}
 F_{x_1x_2}\\
 F_{x_1x_3}\\
 \vdots\\
 F_{x_1x_n}
 \end{matrix}\\
\hline
 \begin{matrix}
 F_{x_1x_2x_3}\\
 F_{x_1x_3x_4}\\
 \vdots\\
 F_{x_1x_{n-1}x_n}
 \end{matrix}
\end{array}
\right].
\end{equation}
\normalsize
If the determinant of the square matrix above -- denoted by $\Psi$ in Proposition \ref{form1} -- is non-zero, the system may be solved  to obtain $A_i$ as $\frac{\Psi_i}{\Psi}$ as in \eqref{huge4}. Algebraic PDEs \eqref{form1-1} are the cross-multiplied form of  $(\frac{\Psi_i}{\Psi})_{x_j}=0$ where $j\neq 1,i$. It follows from a simple continuity argument that \eqref{form1-1} holds for all smooth functions $F$ of the form \eqref{superposition1} even if the corresponding determinant $\Psi$ becomes zero. Conversely, if $\Psi\neq 0$\footnote{This condition is not vacuous for the family of functions \eqref{superposition1}; they include all functions of the form $F=g(x_2,\dots,x_n)$ and it is not hard to find such an $F$ for which the matrix from \eqref{system'} is non-singular at a point.} the matrix equation \eqref{huge4} yields functions $A_2:=\frac{\Psi_2}{\Psi},\dots,A_n:=\frac{\Psi_n}{\Psi}$ for which 
$F_{x_1}=A_2F_{x_2}+\dots+A_nF_{x_n}$. Each coefficient $A_i$ is dependent only on $x_1$ and $x_i$ due to \eqref{form1-1}. Invoking Theorem \ref{integrability}, there exist locally defined smooth functions $f_2(x_1,x_2),\dots,f_n(x_1,x_n)$ with the property that each   
$A_i(x_1,x_i)\frac{\partial}{\partial x_1}+\frac{\partial}{\partial x_i}$ is parallel with $\nabla f_i(x_1,x_i)\neq\mathbf{0}$. Hence $A_i=\frac{(f_i)_{x_1}}{(f_i)_{x_i}}$ for each $2\leq i\leq n$, and $F_{x_1}=A_2F_{x_2}+\dots+A_nF_{x_n}$ then implies: 
$$
\nabla F=\frac{F_{x_2}}{(f_2)_{x_2}}\nabla f_2+\dots+\frac{F_{x_n}}{(f_n)_{x_n}}\nabla f_n.
$$
The functions $f_2(x_1,x_2),\dots,f_n(x_1,x_n)$ can be extended to a local coordinate system in $\Bbb{R}^n$ near the point under consideration (because $(f_i)_{x_i}\neq 0$).
The previous identity indicates that $F$ is independent of the other coordinate function and is thus in the form of $g(f_2,\dots,f_n)$. 
\end{proof}

\begin{proof}[Proof of Proposition \ref{form2}]
We first establish the claim that for the form  $\omega$ defined in \eqref{auxiliary2} the integrability condition $\omega\wedge{\rm{d}}\omega=0$ amounts to identities \eqref{constraint}. Differentiating \eqref{auxiliary2} yields:
\small
\begin{equation}\label{auxiliary3}
{\rm{d}}\omega={\rm{d}}\left(\frac{\beta(x_1,x_2)}{F_{x_2}}\right)\wedge\left({\rm{d}}F-F_{x_1}\,{\rm{d}}x_1\right)-\frac{\beta(x_1,x_2)}{F_{x_2}}\,{\rm{d}}F_{x_1}\wedge{\rm{d}}x_1.  
\end{equation}
\normalsize
Wedging \eqref{auxiliary2} with \eqref{auxiliary3}:
\small
\begin{equation}\label{auxiliary4}
\begin{split}
\omega\wedge{\rm{d}}\omega
&=\left({\rm{d}}x_1+\frac{\beta(x_1,x_2)}{F_{x_2}}\left({\rm{d}}F-F_{x_1}\,{\rm{d}}x_1\right)\right)\wedge
\left({\rm{d}}\left(\frac{\beta(x_1,x_2)}{F_{x_2}}\right)\wedge\left({\rm{d}}F-F_{x_1}\,{\rm{d}}x_1\right)-\frac{\beta(x_1,x_2)}{F_{x_2}}\,{\rm{d}}F_{x_1}\wedge{\rm{d}}x_1\right)\\
&={\rm{d}}x_1\wedge{\rm{d}}\left(\frac{\beta(x_1,x_2)}{F_{x_2}}\right)\wedge{\rm{d}}F-\left(\frac{\beta(x_1,x_2)}{F_{x_2}}\right)^2{\rm{d}}F\wedge{\rm{d}}F_{x_1}\wedge{\rm{d}}x_1
\\
&=-\frac{F_{x_2}\,{\rm{d}}\beta(x_1,x_2)-\beta(x_1,x_2)\,{\rm{d}}F_{x_2}}{(F_{x_2})^2}\wedge{\rm{d}}x_1\wedge{\rm{d}}F-\frac{\beta(x_1,x_2)^2}{(F_{x_2})^2}\,{\rm{d}}F_{x_1}\wedge
{\rm{d}}x_1\wedge{\rm{d}}F\\
&=\frac{1}{(F_{x_2})^2}\left[-F_{x_2}\beta_{x_2}\,{\rm{d}}x_2\wedge{\rm{d}}x_1+\beta\,{\rm{d}}F_{x_2}\wedge{\rm{d}}x_1-\beta^2\,{\rm{d}}F_{x_1}\wedge{\rm{d}}x_1\right]\wedge{\rm{d}}F\\
&=\frac{1}{(F_{x_2})^2}\,{\rm{d}}x_1\wedge\left[\beta_{x_2}F_{x_2}\,{\rm{d}}x_2\wedge{\rm{d}}F-\beta\,{\rm{d}}F_{x_2}\wedge{\rm{d}}F+\beta^2\,{\rm{d}}F_{x_1}\wedge{\rm{d}}F\right].
\end{split}    
\end{equation}
\normalsize
To see what $\omega\wedge{\rm{d}}\omega=0$ entails to, the coefficients of ${\rm{d}}x_1\wedge{\rm{d}}x_2\wedge{\rm{d}}x_i$ and  
${\rm{d}}x_1\wedge{\rm{d}}x_i\wedge{\rm{d}}x_j$ ($i,j\geq 3$ distinct) in the form appeared  in the last line of \eqref{auxiliary4} must be equated to zero. The former coefficient is given by 
\small
\begin{equation*}
\begin{split}
&\beta_{x_2}\frac{F_{x_i}}{F_{x_2}}-\beta\,\frac{F_{x_2x_2}F_{x_i}-F_{x_2x_i}F_{x_2}}{(F_{x_2})^2}+\beta^2\,\frac{F_{x_1x_2}F_{x_i}-F_{x_1x_i}F_{x_2}}{(F_{x_2})^2}\\
&=\beta_{x_2}\frac{F_{x_i}}{F_{x_2}}+\beta\left(\frac{F_{x_i}}{F_{x_2}}\right)_{x_2}-\beta^2\left(\frac{F_{x_i}}{F_{x_2}}\right)_{x_1}.
\end{split}
\end{equation*}
\normalsize
Dividing by $\beta^2$, we recover \eqref{constraint}. Going back to \eqref{auxiliary4}, for any $3\leq i<j\leq n$ the coefficient of ${\rm{d}}x_1\wedge{\rm{d}}x_i\wedge{\rm{d}}x_j$ in the last line is 
\small
$$
\frac{1}{(F_{x_2})^2}\left[-\beta\left(F_{x_2x_i}F_{x_j}-F_{x_2x_j}F_{x_i}\right)+\beta^2\left(F_{x_1x_i}F_{x_j}-F_{x_1x_j}F_{x_i}\right)\right].
$$
\normalsize
Hence, we want $\frac{1}{\beta}$ to coincide with 
\small
$$
\frac{F_{x_1x_i}F_{x_j}-F_{x_1x_j}F_{x_i}}{F_{x_2x_i}F_{x_j}-F_{x_2x_j}F_{x_i}}=\frac{(\frac{F_{x_i}}{F_{x_j}})_{x_1}}{(\frac{F_{x_i}}{F_{x_j}})_{x_2}}
$$
\normalsize
for any $3\leq i<j\leq n$. This is constraint \eqref{extraconstraint2} which Proposition \ref{form2} necessitates. 
Indeed,  showing the numerator and the denominator of the last fraction by $\Psi_{ij}^{(1)}$ and $\Psi_{ij}^{(2)}$ as in \eqref{form2-1}, the first two constraints in \eqref{form2-3} simply state the independence of $\frac{1}{\beta}=\frac{\Psi_{ij}^{(1)}}{\Psi_{ij}^{(2)}}$ from $i<j$ and  $\left(\frac{1}{\beta}\right)_{x_l}=0$ (where $l\in\{3,\dots,n\}$) while the last line of \eqref{form2-3} is the cross-multiplied form of 
\small
$$
\left(\frac{F_{x_i}}{F_{x_2}}\right)_{x_1}=
\frac{\Psi_{i'j'}^{(1)}}{\Psi_{i'j'}^{(2)}}\left(\frac{F_{x_i}}{F_{x_2}}\right)_{x_2}-\left(\frac{\Psi_{i'j'}^{(1)}}{\Psi_{i'j'}^{(2)}}\right)_{x_2}\frac{F_{x_i}}{F_{x_2}}
$$
\normalsize
which is obtained from substituting for $\frac{1}{\beta}$ in \eqref{constraint}. These necessary constraints are derived assuming that divisions taking place in the process are not problematic. Nonetheless, as in Examples of \S\ref{toy examples-trees}, a continuity argument indicates that they hold for any smooth function of the form \eqref{superposition2} regardless.\\
\indent
Finally, observe that, in the presence of appropriate non-vanishing conditions, our computations could be reversed to prove the sufficiency of constraints \eqref{form2-2} and \eqref{form2-3} for the existence of a local representation such as \eqref{superposition2}: If there exists a pair $i_0<j_0$ of indices in $\{3,\dots,n\}$ with 
$\Psi_{i_0j_0}^{(1)},\Psi_{i_0j_0}^{(2)}\neq 0$, the function $\beta:=\frac{\Psi_{i_0j_0}^{(2)}}{\Psi_{i_0j_0}^{(1)}}\neq 0$ satisfies \eqref{constraint} and therefore, 
for the form $\omega$ defined as \eqref{auxiliary2}, one has $\omega\wedge{\rm{d}}\omega=0$. Invoking Theorem \ref{integrability}, there exists a locally defined smooth function $\zeta$ with $\zeta_{x_1},\zeta_{x_2}\neq 0$ for which \eqref{auxiliary1} and  \eqref{auxiliary1'} hold. This means that $F$ is a function of the first two members of the local coordinate system   
$$
(\xi:=x_1,\zeta,\eta_3:=x_3,\dots,\eta_{n}:=x_n),
$$
i.e. $F$ is in the form of $g(x_1,\zeta(x_1,\dots,x_n))$, cf. \eqref{superposition2'}. It remains to show that $\zeta$ is a tree function without repetition as described before. One should verify that $\left(\frac{\zeta_{x_j}}{\zeta_{x_i}}\right)_{x_k}=0$ for any three indices $i<j<k$ form $\{2,3,\dots,n\}$, cf. Theorem \ref{main-tree}. This is due to the fact that by \eqref{auxiliary1'} $\frac{\zeta_{x_2}}{\zeta_{x_1}}$ is the same as $\beta(x_1,x_2)$; and for any $j>2$ the ratio $\frac{\zeta_{x_j}}{\zeta_{x_2}}$  coincides with 
$\frac{F_{x_j}}{F_{x_2}}$ which, because of \eqref{form2-2}, is independent of $x_k$ for any $k>j$.
\end{proof}

\begin{proof}[Proof of Proposition \ref{rational functions}]
Aiming for a contradiction, suppose there exist bivariate polynomials $f=f(x,y)$, $h=h(x,z)$ and $g=g(u,v)$ for which 
\begin{equation}\label{auxiliary15}
(x+y)z+y^3z^3=g(f(x,y),h(x,z)).
\end{equation}
Plugging in $z=0$, observe that $g(f(x,y),h(x,0))\equiv 0$. If $h(x,0)\not\equiv 0$, then the degree of $g(f(x,y),h(x,0))$ with respect to $y$ is given by $\deg_ug.\deg_yf$. But this is a positive integer because if $g(u,v)$ is independent of $u$ or $f(x,y)$ is independent of $y$, the left-hand side of \eqref{auxiliary15} must be independent of $y$ which is absurd. Therefore, $h(x,z)$ is divisible by $z$, say in the form of $h(x,z)=zh_1(x,z)$. Next, we argue that $g(u,v)$ is divisible by $v$. Writing it as $g_1(u)+vg_2(u,v)$ for appropriate polynomials $g_1$ and $g_2$, \eqref{auxiliary15} amounts to 
$$
(x+y)z+y^3z^3=g_1(f(x,y))+zh_1(x,z)\,g_2(f(x,y),zh_1(x,z)).
$$
Now $g_1(f(x,y))$ must be divisible by $z$ as all other terms are. But as $f$ is non-constant, this cannot occur unless $g_1\equiv 0$. The previous equation then may be written as 
$$
x+y+y^3z^2=h_1(x,z)\,g_2(f(x,y),zh_1(x,z)).
$$
Thus, the polynomial $h_1(x,z)$ divides $x+y+y^3z^2$. The only possibility is $h_1(x,z)\equiv\pm 1$. Substituting in the equation above, we arrive at an identity of the form
$$
x+y+y^3z^2=\pm g_2(f(x,y),\pm z).
$$
This cannot be the case since the trivariate function $x+y+y^3z^2$ does not satisfy  constraint \eqref{auxiliary-basic-constraint'}.
\end{proof}

\section{Differential forms}\label{Frobenius}
Differential forms are ubiquitous objects in differential geometry and tensor calculus. We only need the theory of differential forms on open domains in Euclidean spaces. 
Theorem \ref{integrability} (which has been used several times throughout the paper e.g. in the proof of Theorem \ref{main-tree}) is formulated in terms of differential forms.  The goal of this appendix is to provide the necessary background for understanding the theorem and its proof along with the related computations appeared in the proof of Proposition \ref{form2}. 

We begin with a very brief account of the local theory of differential forms. The reader can find a detailed treatment in \cite[chap. 5]{MR1886084}. Let $U$ be an open subset of $\Bbb{R}^n$. A \textit{differential $k$-form} $\omega$ on $U$ assigns a scalar to any $k$-tuple of tangent vectors at a point $\mathbf{p}$ of $U$. 
This assignment, denoted by $\omega_{\mathbf{p}}$, must be \textit{multi-linear} and \textit{alternating}.  We say $\omega$ is smooth (respectively analytic) if $\omega_{\mathbf{p}}$ varies smoothly (resp. analytically) with $\mathbf{p}$. In other words, feeding $\omega$ with $k$ smooth (resp. analytic) vector fields $\mathbf{V}_1,\dots,\mathbf{V}_k$ on $U$ results in a function $\omega(\mathbf{V}_1,\dots,\mathbf{V}_k):U\rightarrow\Bbb{R}$ which is smooth (resp. analytic). We next exhibit an expression for $\omega$. Consider the standard basis  
$\left(\frac{\partial}{\partial x_1},\dots,\frac{\partial}{\partial x_n}\right)$ of vector fields on $U$ where $\frac{\partial}{\partial x_i}$ assigns $\mathbf{e}_i$ to each point. The dual basis is denoted by $({\rm{d}}x_1,\dots,{\rm{d}}x_n)$ which at each point yields the dual of the standard basis $(\mathbf{e}_1,\dots,\mathbf{e}_n)$ for $\Bbb{R}^n$, i.e. 
${\rm{d}}x_i\left(\frac{\partial}{\partial x_j}\right)\equiv\delta_{ij}$.
Each of ${\rm{d}}x_1,\dots,{\rm{d}}x_n$ is a $1$-form on $U$, and any $k$-form $\omega$ can be written in terms of them:
\begin{equation}\label{differential form}
\omega=\sum_{1\leq i_1<\dots<i_k\leq n}f_{i_1\dots i_k}\,{\rm{d}}x_{i_1}\wedge\dots\wedge{\rm{d}}x_{i_k}.
\end{equation}
Here, each coefficient $f_{i_1\dots i_k}$ is a (smooth or analytic according to the context) function $U\rightarrow\Bbb{R}$. In front of it, ${\rm{d}}x_{i_1}\wedge\dots\wedge{\rm{d}}x_{i_k}$ appears which is a $k$-form satisfying 
${\rm{d}}x_{i_1}\wedge\dots\wedge{\rm{d}}x_{i_k}\left(\frac{\partial}{\partial x_{i_1}},\dots,\frac{\partial}{\partial x_{i_k}}\right)=1.$
This is constructed by the operation of \textit{exterior product} (also called \textit{wedge product}) from multi-linear algebra. The exterior product is an associative and distributive linear operation that out of $k_i$-tensors $\tau_i$ $(1\leq i\leq l)$ constructs an alternating $(k_1+\dots+k_l)$-tensor $\tau_1\wedge\dots\wedge\tau_l$. This product is \textit{anti-commutative}, e.g. 
${\rm{d}}x_i\wedge{\rm{d}}x_j=-{\rm{d}}x_j\wedge{\rm{d}}x_i$; this is the reason that in \eqref{differential form} the indices are taken to be strictly ascending.  \\
\indent
Another operation in the realm of differential forms is \textit{exterior differentiation}. For the $k$-form $\omega$ from \eqref{differential form}, its exterior derivative ${\rm{d}}\omega$ is a $(k+1)$-form defined as
$$
{\rm{d}}\omega:=\sum_{1\leq i_1<\dots<i_k\leq n}{\rm{d}}f_{i_1\dots i_k}\wedge{\rm{d}}x_{i_1}\wedge\dots\wedge{\rm{d}}x_{i_k};
$$
where the exterior derivative of a function $f$ is defined as
\begin{equation}\label{exact1}
{\rm{d}}f:=\sum_{i=1}^nf_{x_i}\,{\rm{d}}x_i.
\end{equation}
Notice that the $1$-form just appeared is the dual of the gradient vector field 
\begin{equation}\label{exact2}
\nabla f=\sum_{i=1}^n f_{x_i}\,\frac{\partial}{\partial x_i}.
\end{equation}

\begin{example}
In dimension three, the exterior differentiation encapsulates the familiar vector calculus operators curl and divergence. Consider the  vector field 
$$
\mathbf{V}(x,y,z)=[V_1(x,y,z),V_2(x,y,z),V_3(x,y,z)]^{\rm{T}}.
$$
The exterior derivatives 
$$
{\rm{d}}\big(V_1\,{\rm{d}}x+V_2\,{\rm{d}}y+V_3\,{\rm{d}}z\big)=\big((V_3)_y-(V_2)_z\big){\rm{d}}y\wedge{\rm{d}}z+\big((V_1)_z-(V_3)_x\big){\rm{d}}z\wedge{\rm{d}}x+\big((V_2)_x-(V_1)_y\big){\rm{d}}x\wedge{\rm{d}}y
$$
and 
$$
{\rm{d}}\big(V_1\,{\rm{d}}y\wedge{\rm{d}}z+V_2\,{\rm{d}}z\wedge{\rm{d}}x+V_3\,{\rm{d}}x\wedge{\rm{d}}y\big)=\big((V_1)_x+(V_2)_y+(V_3)_z\big){\rm{d}}x\wedge{\rm{d}}y\wedge{\rm{d}}z
$$
respectively have ${\rm{curl}}\mathbf{V}$ and ${\rm{div}}\mathbf{V}$ as their coefficients. As a matter of fact, there is a general Stokes' formula for differential forms that recovers the Kelvin–Stokes Theorem  and the Divergence Theorem as special cases. Finally, we point out that the familiar identities ${\rm{curl}}\circ\nabla=\mathbf{0}$ and ${\rm{div}}\circ{\rm{curl}}=0$ are instances of the general property  
${\rm{d}}\circ {\rm{d}}=0$ of the exterior differentiation.
\end{example} 

\begin{example}
As mentioned in the previous example, the outcome of twice applying the exterior differentiation operator to a form is always zero. This is an extremely important property that leads to the definitions of \textit{closed} and \textit{exact} differential forms. A $k$-form $\omega$ on an open subset $U$ of $\Bbb{R}^n$ is called closed if ${\rm{d}}\omega=0$. This holds if $\omega$ is in the form of $\omega={\rm{d}}\alpha$ for a $(k-1)$-form $\alpha$ on $U$. Such forms are called exact. The space of closed forms may be strictly larger than the space of exact forms; the difference of these spaces can be used to measure the topological complexity of $U$. If $U$ is an open box-like region, every closed form on it is exact. But for instance, the $1$-form $\omega=-\frac{y}{x^2+y^2}\,{\rm{d}}x+\frac{x}{x^2+y^2}\,{\rm{d}}y$ on $\Bbb{R}^2-\{(0,0)\}$ is closed while may not be written as ${\rm{d}}\alpha$  for any smooth function $\alpha:\Bbb{R}^2-\{(0,0)\}\rightarrow\Bbb{R}$. This brings us to a famous fact from multi-variable calculus that we have used several times (e.g. in the proof of Theorem \ref{main-tree-activation}): A necessary condition for a vector field  $\mathbf{V}=\sum_{i=1}^nV_i\,\frac{\partial}{\partial x_i}$ on an open subset $U$ of $\Bbb{R}^n$ to be a gradient vector field is $(V_i)_{x_j}=(V_j)_{x_i}$ for any $1\leq i,j\leq n$. Near each point of $U$ the vector field $\mathbf{V}$ may be written as $\nabla f$; it is globally in the form of $\nabla f$ for a function 
$f:U\rightarrow\Bbb{R}$ when $U$ is simply connected. In view of \eqref{exact1} and \eqref{exact2}, one may rephrase this fact as: \textit{Closed $1$-forms on $U$ are exact if and only if $U$ is simply connected.}
\end{example}

\begin{proof}[Proof of Theorem \ref{integrability}]
Near a point $\mathbf{p}\in\Bbb{R}^n$ at which $\mathbf{V}(\mathbf{p})\neq\mathbf{0}$ we seek a locally defined  function  $\xi$ with $\mathbf{V}\parallel\nabla\xi\neq\mathbf{0}$. 
Recall that if $\mathbf{q}\in\Bbb{R}^n$ is a regular point of $\xi$, then near $\mathbf{q}$ the level set of $\xi$ passing through $\mathbf{q}$ is an $(n-1)$-dimensional submanifold of $\Bbb{R}^n$ to which the gradient vector field $\nabla\xi\neq\mathbf{0}$ is perpendicular. As we want the gradient to be parallel to the vector field $\mathbf{V}$, the equivalent characterization in terms of the $1$-form $\omega$ which is the dual of $\mathbf{V}$ (cf. \eqref{vector field},\eqref{1-form}) asserts that $\omega$ is zero at any vector tangent to the level set. So the tangent space to the level set at the point $\mathbf{q}$ could be described as $\{\mathbf{v}\in\Bbb{R}^n\,|\,\omega_{\mathbf{q}}(\mathbf{v})=0\}$. 
As $\mathbf{q}$ varies near $\mathbf{p}$, these $(n-1)$-dimensional subspaces of $\Bbb{R}^n$ vary smoothly. In differential geometry, such a higher-dimensional version of a vector field is called a \textit{distribution}; and the property that these subspaces are locally given by tangent spaces to a family of submanifolds (the level sets here) is called \textit{integrability}. The seminal Frobenius Theorem (\cite[Theorem 2.11.11]{MR0251745}) states that the distribution defined by a nowhere vanishing $1$-from $\omega$ is integrable if and only if $\omega\wedge{\rm{d}}\omega=0$. 
\end{proof}

\section{Basic notions of algebraic geometry}\label{Background}
The goal of this appendix is to provide a very brief introduction to the theory of algebraic varieties in order to illuminate the concepts and results cited in \S\ref{Polynomial}. There are numerous introductory texts on the subject; see \cite{MR1416564,milneAG} for the general theory and \cite{coste2000introduction} for real varieties.

An  affine algebraic variety (complex or real) is the set of common zeros of a family of polynomials.\footnote{A variety is not required to be \textit{irreducible} in our convention. As a matter of fact, it is interesting to determine if  functional varieties from Definition \ref{variety} are irreducible and if not, what are their irreducible components.} A complex (respectively real) variety considered in this paper is a subset of $\Bbb{C}^m$ (resp. $\Bbb{R}^m$) defined as the set of common zeros of a family polynomials belonging to $\Bbb{C}[x_1,\dots,x_m]$ (resp. $\Bbb{R}[x_1,\dots,x_m]$). The subsets just described constitute the closed subsets of the affine spaces $\Bbb{C}^m$ or $\Bbb{R}^m$ in the Zariski topology, a topology coarser than the usual topology.  \\
\indent
There is a notion of dimension for varieties that agrees with our usual conception of dimension when the variety under consideration is a manifold. Also in the algebraic category, the maps between varieties (the morphisms) are studied. In our setting, one only needs to consider polynomial maps such as  $\Bbb{C}^m\rightarrow\Bbb{C}^l$ or 
$\Bbb{R}^m\rightarrow\Bbb{R}^l$. For polynomial maps the dimension of the Zariski closure of the image is equal to the dimension of the domain minus the dimension of a \textit{generic fiber}. The image itself may not be a locally closed subvariety, but belongs to certain Boolean algebras of subsets that we shall describe. This brings us to Chevalley's Theorem and the Tarski–Seidenberg Theorem which were both mentioned in \S\ref{Polynomial}. A subset of an ambient complex affine space is called constructible if it can be expressed as a union of finitely subsets of the form 
$$
\{P_1=0,\dots, P_k=0, Q\neq 0\}
$$
where $P_1,\dots,P_k$ and $Q$ are complex polynomials.
Chevalley's Theorem asserts that a polynomial map $\Bbb{C}^m\rightarrow\Bbb{C}^l$ takes constructible sets to constructible sets. In particular, the image of $\Bbb{C}^m\rightarrow\Bbb{C}^l$ is a constructible set.  Over the real numbers, one must include polynomial inequalities as well: A semi-algebraic subset of a Euclidean space is defined as a finite union of subsets of the form 
$$
\{P=0, Q_1>0,\dots,Q_j>0\}
$$
where $P$ and $Q_1,\dots,Q_j$ are real polynomials. The Tarski–Seidenberg Theorem asserts that the polynomial image of a semi-algebraic set is semi-algebraic. In particular, the image of a polynomial map $\Bbb{R}^m\rightarrow\Bbb{R}^l$ is semi-algebraic.

\begin{example}
The image of the map $\Bbb{C}^2\rightarrow\Bbb{C}^2$ defined by $(x,y)\mapsto (x^2,xy)$ is not a closed subvariety of $\Bbb{C}^2$; it is the constructible set $\{x\neq 0\}\cup\{x=y=0\}$. The image of the map $\Bbb{R}^2\rightarrow\Bbb{R}^2$ defined by the same formula is the semi-algebraic set $\{x>0\}\cup\{x^2+y^2=0\}$ in $\Bbb{R}^2$. 
\end{example}

\begin{proof}[Proof of Lemma \ref{dependence}]
The $m$-tuple $(p_1,\dots,p_l)$ of polynomials defines a map from the $m$-dimensional affine space to the $l$-dimensional affine space. The dimension of the image could not be greater than $m$. So when $l>m$, the image is nowhere dense in the Zariski topology; it must be included in a closed subset of the form $\{\Phi=0\}$ where $\Phi\not\equiv 0$. Indeed, the closure of the image is of codimension at least $l-m$, so there exist at least $l-m$ independent algebraic equation that are satisfied at any point $\left(p_1(t_1,\dots,t_m),\dots,p_l(t_1,\dots,t_m)\right)$ (compare with Remark \ref{dimension}).  \\
\indent
The existence of $\Phi$ could also be deduced in a more elementary manner which we include here for the sake of thoroughness. For a positive integer $a$, there are precisely $\binom{a+l}{l}$ monomials such as 
$p_1^{a_1}\dots p_l^{a_l}$ with their total degree $a_1+\dots+a_l$ not greater than $a$. But each of them is a polynomial of $t_1,\dots,t_m$ of total degree at most $ad$ where $d:=\max\{\deg p_1,\dots,\deg p_l\}$. For $a$ large enough, $\binom{a+l}{l}$ is greater than $\binom{ad+m}{m}$ because the degree of the former as a polynomial of $a$ is $l$ while the degree of the latter is $m$. For such an $a$ the number of monomials  $p_1^{a_1}\dots p_l^{a_l}$ is larger than the dimension of the space of polynomials of $t_1,\dots,t_m$ of total degree at most $ad$. Therefore, there  exists a linear dependency among these monomials which amounts to a  non-trivial polynomial relation among $p_1,\dots,p_l$.
\end{proof}

\begin{proof}[Proof of Lemma \ref{semi-algebraic}]
It suffices to show that the complement 
$$
\{P\in {\mathbf{Poly}}_{d,n}\,|\, \exists\, \mathbf{p}\in\Bbb{R}^n \text{ with } P(\mathbf{p})\leq 0 \}
$$
is a semi-algebraic subset of the affine space ${\mathbf{Poly}}_{d,n}$. It is the image of  
$$
\left\{(P,\mathbf{p})\in{\mathbf{Poly}}_{d,n}\times\Bbb{R}^n\,|\, P(\mathbf{p})\leq 0\right\}
$$
under the projection map ${\mathbf{Poly}}_{d,n}\times\Bbb{R}^n\rightarrow{\mathbf{Poly}}_{d,n}$. Showing that the latter subset of ${\mathbf{Poly}}_{d,n}\times\Bbb{R}^n$ is semi-algebraic concludes the proof: It is the preimage of the semi-algebraic set $(-\infty,0]\subset\Bbb{R}$ under the evaluation map 
$$
\begin{cases}
{\mathbf{Poly}}_{d,n}\times\Bbb{R}^n\rightarrow\Bbb{R}\\
(P,\mathbf{p})\mapsto P(\mathbf{p}).
\end{cases}
$$
\end{proof}

\vspace*{0.5cm}
\bibliography{biblography}
\bibliographystyle{alpha}

\end{document}